\PassOptionsToPackage{cleveref}{jmlrutils}
\documentclass[12pt]{colt2026}
\usepackage{amsmath,amsfonts,amssymb,commath,dsfont,nicefrac,xfrac,mathtools,mathrsfs}
\usepackage{float}
\usepackage{bm,bbm}
\makeatletter
\expandafter\let\csname enumerate*\endcsname\relax
\expandafter\let\csname endenumerate*\endcsname\relax
\expandafter\let\csname itemize*\endcsname\relax
\expandafter\let\csname enditemize*\endcsname\relax
\expandafter\let\csname description*\endcsname\relax
\expandafter\let\csname enddescription*\endcsname\relax
\makeatother
\usepackage[inline]{enumitem}
\usepackage{framed}
\usepackage{xspace}
\usepackage{microtype}

\usepackage{hyperref}
\usepackage{cleveref}

\Crefname{equation}{Eq.\!}{Eqs.\!}
\usepackage[dvipsnames]{xcolor}
\usepackage[toc,page]{appendix}
\usepackage{svg}
\usepackage{booktabs}       

\usepackage{graphicx}
\usepackage{wrapfig}

\renewcommand{\parallel}{\mathrel{/\mkern-5mu/}}
\makeatletter
\newcommand{\notparallel}{%
  \mathrel{\mathpalette\not@parallel\relax}%
}
\newcommand{\not@parallel}[2]{%
  \ooalign{\reflectbox{\(\m@th#1\smallsetminus\)}\cr\hfil\(\m@th#1\parallel\)\cr}%
}
\makeatother

\DeclareMathOperator{\diag}{diag}

\DeclareMathOperator{\range}{range}

\DeclareMathOperator{\Var}{Var}
\DeclareMathOperator{\Span}{span}
\DeclareMathOperator{\rad}{rad}
\DeclareMathOperator{\rank}{rank}
\DeclareMathOperator{\tr}{tr}

\usepackage{tcolorbox}
\tcbuselibrary{theorems,skins}

\newtcolorbox[auto counter]{mainresult}[1]{%
  enhanced,
  title={\textbf{Main Result~\thetcbcounter\if\relax\detokenize{#1}\relax\else~(#1):\fi}},
  attach title to upper={\ },
  coltitle=black,
  boxrule=0.5pt,
  arc=4pt,
  left=1pt,
  right=1pt,
  bottom=2pt,
  top=2pt,
  grow to left by=-0.01cm,
  grow to right by=-0.01cm,
  rounded corners,
}

\makeatother

\newenvironment{proofsketch}{%
  \begingroup
  \renewcommand{\proofname}{Proof sketch}%
  \begin{proof}%
}{%
  \end{proof}%
  \endgroup
}

\newtheorem*{theorem*}{Theorem}

\everypar{\looseness=-1}

\title[A Unified Theory of Random Projection for Influence Functions]{A Unified Theory of Random Projection for Influence Functions\(^{\ast}\)}
\usepackage{times}

\coltauthor{%
 \Name{Pingbang Hu} \Email{pbb@illinois.edu}\\
 \Name{Yuzheng Hu}\(^{\dagger}\) \Email{yh46@illinois.edu}\\
 \Name{Jiaqi W.\ Ma} \Email{jiaqima@illinois.edu}\\
 \Name{Han Zhao}\(^{\dagger}\) \Email{hanzhao@illinois.edu}\\
 \addr University of Illinois Urbana-Champaign, Urbana, IL, USA%
}

\begin{document}
\maketitle

\begingroup
\makeatletter
\insert\footins{%
  \footnotesize
  \interlinepenalty\interfootnotelinepenalty
  \hsize\columnwidth
  \@parboxrestore
  \noindent {\(^{\ast}\)} Authors ordered alphabetically.\par
  \noindent {\(^\dagger\)} Corresponding to: Yuzheng Hu and Han Zhao.\par
}%
\makeatother
\endgroup

\begin{abstract}
    Influence functions and related data attribution scores take the form of inverse-sensitive bilinear functionals \(g^{\top}F^{-1}g^{\prime}\), where \(F\succeq 0\) is a curvature operator and \(g,g^{\prime}\) are training and test gradients.
    In modern overparametrized models, forming or inverting \(F\in\mathbb{R}^{d\times d}\) is prohibitive, motivating scalable influence computation via \textit{random projection} with a sketch \(P \in \mathbb{R}^{m\times d}\).
    This practice is commonly justified via the Johnson--Lindenstrauss (JL) lemma, which ensures approximate preservation of Euclidean geometry for a fixed dataset.
    However, preserving pairwise distances does not address how sketching behaves under inversion. Furthermore, there is no existing theory that explains how sketching interacts with other widely-used techniques, such as ridge regularization (replacing \(F^{-1}\) with \((F+\lambda I)^{-1}\)) and structured curvature approximations.

    We develop a unified theory characterizing when projection provably preserves influence functions, with a focus on the required sketch size \(m\). When \(g,g^{\prime}\in\range(F)\), we show that:
    \begin{enumerate*}[label=(\roman*), itemjoin={{; }}, itemjoin*={{; and }}]
        \item \textbf{Unregularized projection}: exact preservation holds if and only if \(P\) is injective on \(\range(F)\), which necessitates \(m\geq \rank(F)\)
        \item \textbf{Regularized projection}: ridge regularization fundamentally alters the sketching barrier, with approximation guarantees governed by the \emph{effective dimension} of \(F\) at the regularization scale \(\lambda\). This dependence is both sufficient and worst-case necessary, and can be substantially smaller than \(\rank(F)\)
        \item \textbf{Factorized influence}: for Kronecker-factored curvatures \(F=A\otimes E\), the guarantees continue to hold for decoupled sketches \(P=P_A\otimes P_E\), even though such sketches exhibit structured row correlations that violate canonical i.i.d.\ assumptions; the analysis further reveals an explicit computational–statistical trade-off inherent to factorized sketches.
    \end{enumerate*}
    Beyond this range-restricted setting, we analyze \textbf{out-of-range test gradients} and quantify a sketch-induced \emph{leakage} term that arises when test gradients have components in \(\ker(F)\). This yields guarantees for influence queries on general, unseen test points.

    Overall, this work develops a novel theory that characterizes when projection provably preserves influence and provides principled, instance-adaptive guidance for choosing the sketch size in practice.
\end{abstract}


\section{Introduction}\label{sec:introduction}
Data attribution aims to explain a trained model's behavior by tracing its predictions back to the training examples~\citep{hammoudeh2024training,deng2025survey}. A classical tool is the \emph{influence function}~\citep{hampel1974influence,koh2017understanding}, which measures how reweighting a training example changes the loss at a test point. In modern neural networks, computing influence involves extremely high-dimensional per-example gradients and ill-conditioned (often singular) curvature operators \(F\). Consequently, scalable influence methods rely on \emph{random projection}, which compresses gradients and curvature to a much smaller dimension before carrying out influence computations~\citep{wojnowicz2016influence,park2023trak,choe2024your,hu2025grass}.  

In these works, projection is often heuristically justified via the Johnson--Lindenstrauss (JL) lemma~\citep{lindenstrauss1984extensions}, since common sketches (Gaussian, Rademacher, and sparse JL) approximately preserve Euclidean geometry~\citep{ailon2009fast,kane2014sparser,nelson2013osnap,cohen2016nearly}. However, influence depends on an \emph{inverse-sensitive} bilinear form induced by \(F^{-1}\), so JL-style arguments do not, on their own, guarantee that projection preserves influence. Furthermore, while recent empirical evidence suggests that the quality of projected influence is sensitive to the sketch size and other hyperparameters, such as the regularization strength~\citep{wang2025taming}, a formal theory explaining these effects remains lacking.

We develop a unified theoretical analysis of projection across three widely used influence-function variants in large-scale neural networks:
\begin{enumerate*}[label=(\arabic*), itemjoin={{; }}, itemjoin*={{; and }}]
    \item \textbf{Unregularized projection}~\citep{wojnowicz2016influence,park2023trak}, which applies sketching directly to influence computations without explicit regularization
    \item \textbf{Regularized projection}~\citep{zheng2024intriguing,mlodozeniec2025influence}, which combines sketching with ridge regularization to stabilize inverse curvature computations~\citep{koh2017understanding}
    \item \textbf{Kronecker-factored influence}~\citep{choe2024your,hu2025grass}, which applies factorized projection on top of structured curvature approximations such as K-FAC~\citep{martens2015kfac,george2018fast}.
\end{enumerate*}

\paragraph{Setup and Notation.}
Let \(g\) and \(g^{\prime}\) denote training and test gradients with respect to the trained model parameters \(\theta\in\mathbb{R}^d\), and let \(F\succeq 0\) be a curvature matrix evaluated at \(\theta \), with \(r \coloneqq \rank(F)\). Typical choices of \(F\) include the generalized Gauss--Newton matrix~\citep{bae2022if,mlodozeniec2025influence} and the empirical Fisher \(\frac{1}{n}\sum_{i=1}^n g_i g_i^{\top}\)~\citep{grosse2023studying,kwon2024datainf}, both standard approximations to the Hessian. We study the inverse-sensitive bilinear form with a ridge parameter \(\lambda \geq 0\), denoted as \(\tau_{\lambda}(g,g^{\prime}) \coloneqq g^{\top}(F+ \lambda I_d)^{-1}g^{\prime}\), where \(F^{-1}\) denotes either the matrix inverse or the Moore--Penrose pseudoinverse when \(F\) is singular. Unless otherwise stated, we let \(P \in \mathbb{R}^{m \times d}\) denote a sketch whose rows are i.i.d. \(1/\sqrt{m}\)-scaled isotropic sub-Gaussian vectors~\citep[Chapter 2]{vershynin2018high}.\footnote{A mean-zero random variable \(X\) is \emph{sub-Gaussian} with parameter \(\sigma^2\) if \(\mathbb{E}[\exp(tX)] \leq \exp(\sigma^2 t^2/2)\) for all \(t\in\mathbb{R}\); a random vector is sub-Gaussian if all one-dimensional marginals are sub-Gaussian.} Such matrices are commonly referred to as \emph{oblivious sketching matrices} and include Gaussian, Rademacher, and sparse JL transforms widely used in practice. The resulting projected (possibly regularized) influence is defined \(\widetilde{\tau}_{\lambda}(g,g^{\prime}) \coloneqq (Pg)^{\top}(PFP^{\top}+\lambda I_m)^{-1}(Pg^{\prime})\).

\paragraph{Our Contributions.}
We present a sequence of results characterizing when projection provably preserves influence functions across a range of settings. Under the assumption \(g, g^{\prime} \in \range(F)\), we precisely delineate when projection \emph{can} and \emph{cannot} succeed without regularization, show how ridge regularization alters the required sketch size, and extend the analysis to Kronecker-factored curvature approximations. We then relax the assumption on \(g^{\prime}\) and quantify an additional sketch-induced \emph{leakage} term arising from components of the test gradient in \(\ker(F)\), yielding guarantees for influence queries at general, unseen test points.

First, we ask whether sketching can preserve the unregularized influence \(\tau_0(g,g^{\prime})\). We show a dichotomy: unless the sketch is injective on \(\range(F)\), uniform multiplicative approximation is impossible, in the sense that no bound of the form \(\lvert \tau_0 (g, g^{\prime}) - \widetilde{\tau}_0 (g, g^{\prime}) \rvert \leq \varepsilon \tau_0(g, g^{\prime}) \) can hold for all \(g, g^{\prime}\) and any \(\varepsilon > 0\). Conversely, injectivity on \(\range(F)\) guarantees \emph{exact} preservation.

\begin{mainresult}{Unregularized projection, \Cref{thm:unregularized-projection}}
    Let \(F\succeq 0\) with \(r\coloneqq \rank(F)\). For \(\lambda=0\), for all \(g, g^{\prime} \in \range(F)\), \(\widetilde{\tau}_0(g,g^{\prime})=\tau_0(g,g^{\prime})\) \emph{if and only if} \(P\) is injective on \(\range(F)\). If \(P\) is not injective on \(\range(F)\) (in particular if \(m<r\)), then for any constant factor, no uniform multiplicative approximation guarantee is possible over \(g, g^{\prime} \in\range(F)\setminus\{0\}\).
\end{mainresult}

\Cref{thm:unregularized-projection} shows that, without regularization, influence preservation requires \(m\) to scale on the order of \(r\). In contrast, when ridge regularization is employed, we show that the required sketch size is no longer governed by \(r\) but instead by the \emph{effective dimension} \(d_{\lambda}(F) \coloneqq \tr(F(F+\lambda I)^{-1})\), a classical notion in Bayesian model selection~\citep{gull1989developments,mackay1991bayesian}.
This quantity is always bounded above by \(r\) and can be substantially smaller when the spectrum of \(F\) decays quickly.

\begin{mainresult}{Regularized projection: \Cref{thm:regularized-projection-upper-bound,thm:regularized-projection-lower-bound}}
    Fix \(\lambda>0\) and define \(d_\lambda(F)=\tr(F(F+\lambda I)^{-1})\). If \(m = \Omega\big((d_\lambda(F)+\log(1/\delta))/\varepsilon^2\big)\), then with probability at least \(1-\delta\), for all \(g,g^{\prime}\in\range(F)\),
    \[
        \big\lvert \widetilde{\tau}_{\lambda}(g,g^{\prime})-\tau_{\lambda}(g,g^{\prime})\big\rvert
        \leq \varepsilon \sqrt{\tau_0(g,g)}\sqrt{\tau_0(g^{\prime},g^{\prime})}
    \]
    Conversely, for Gaussian oblivious sketches, there exist \(F \succeq 0\) such that if \(m=o(d_\lambda(F)/\varepsilon^2)\), there exists some \(g, g^{\prime}\in\range(F)\) admits an \(\Omega(\varepsilon)\) error with constant probability.
\end{mainresult}

In large neural networks, influence computation hinges on curvature inversion, yet forming or inverting the full empirical Hessian or Fisher is infeasible. Consequently, practical pipelines adopt structured curvature approximations, most notably Kronecker-factored approximate curvature (K-FAC). This motivates us to develop a projection theory tailored to this setting.

As reviewed in \Cref{subsec:factorized-influence}, K-FAC models the curvature as \(F = A \otimes E\) (in a layerwise manner), where \(A\) and \(E\) capture the empirical covariances of forward activations and backpropagated gradients, respectively. To exploit this structure, one natural idea is to enforce the sketch to share the same factorization \(P = P_A \otimes P_E\), where \(P_A\) and \(P_E\) are oblivious sketching matrices~\citep{choe2024your}. While this yields substantial computational savings, the Kronecker structure breaks the i.i.d.\ row assumption on \(P\), rendering a direct adaptation of \Cref{thm:unregularized-projection,thm:regularized-projection-upper-bound} inapplicable. We overcome this technical challenge through a fine-grained analysis and establish rigorous approximation guarantees.

\begin{mainresult}{Factorized influence, \Cref{thm:unregularized-factorized-projection,thm:regularized-factorized-projection-upper-bound}}
    Assume \(F=A\otimes E\succeq 0\) and a Kronecker sketch \(P=P_A\otimes P_E\) with factor sketch sizes \(m_A\) and \(m_E\).
    \begin{enumerate}[label=(\roman*), leftmargin=*]
        \item \textbf{Unregularized barrier.} For \(\lambda=0\), exact invariance on \(\range(F)\) holds if and only if \(P_A\) is injective on \(\range(A)\) and \(P_E\) is injective on \(\range(E)\), which in particular necessitates \(m_A\geq \rank(A)\) and \(m_E\geq \rank(E)\).
        \item \textbf{Regularized approximation.} Let \(P_A\) and \(P_E\) each to be oblivious sketch. For \(\lambda>0\), letting \(\lambda_E\coloneqq \lambda/\lVert E\rVert_2\) and \(\lambda_A\coloneqq \lambda/\lVert A\rVert_2\), if \(m_A = \Omega\big((d_{\lambda_E}(A)+\log(1/\delta))/\varepsilon^2\big)\) and \(m_E = \Omega\big((d_{\lambda_A}(E)+\log(1/\delta))/\varepsilon^2\big)\), then with probability at least \(1-\delta\), for all \(g,g^{\prime}\in\range(F)\),
              \[
                  \lvert \widetilde{\tau}_{\lambda}(g,g^{\prime})-\tau_{\lambda}(g,g^{\prime}) \rvert
                  \leq \varepsilon \sqrt{\tau_0(g,g)}\sqrt{\tau_0(g^{\prime},g^{\prime})}.
              \]
    \end{enumerate}
\end{mainresult}

Finally, we note that all of the above guarantees are stated for gradients lying in \(\range(F)\), which, when \(F\) is the empirical Fisher, includes all training gradients used for attribution. In practice, however, a test gradient \(g^{\prime}\) may have a component in \(\ker(F)\). We show that, in both the unregularized and regularized settings, these components do not affect the true (unsketched) influence, while sketching introduces an additional \emph{leakage} term. We quantify this ``out-of-range leakage'' and show it decays at the usual \(O(m^{-1/2})\) rate with explicit dependence on \(\lambda\) and the spectrum of \(F\).

\begin{mainresult}{Projection leakage, \Cref{thm:projection-leakage,thm:factorized-influence-leakage}}
    For a general \(g^{\prime} \in \mathbb{R}^{d}\), write \(g^{\prime}=g^{\prime}_{\parallel}+g^{\prime}_{\perp}\) with \(g^{\prime}_{\parallel}\in\range(F)\) and \(g^{\prime}_{\perp}\in\ker(F)\). We show that in this case, for either \(\lambda = 0\) or \(\lambda > 0\),
    \[
        \lvert \widetilde{\tau}_\lambda(g,g^{\prime}) - \tau_\lambda(g,g^{\prime}) \rvert
        \leq \lvert \widetilde{\tau}_\lambda(g,g^{\prime}_{\parallel}) - \tau_\lambda(g,g^{\prime}_{\parallel}) \rvert + \lvert \widetilde{\tau}_\lambda(g,g^{\prime}_{\perp}) \rvert ,
    \]
    with an additional leakage error \(\lvert \widetilde{\tau}_\lambda(g,g^{\prime}_{\perp}) \rvert\) beyond \Cref{thm:regularized-projection-upper-bound}. We then prove  in \Cref{thm:projection-leakage} that  for a collection of \(k\) test gradients \(\{g^{\prime}_{j}\}_{j=1}^{k}\), with sketch size \(m = \Omega\big((r+\log(k/\delta))/\varepsilon^2\big)\),\footnote{Or alternatively linear in \(k^{\prime} = \dim\bigl(\Span(\{g_{j,\perp}^{\prime}\}_{j=1}^{k})\bigr)\), which in practice is usually worse than \(\log(k)\).}
    \begin{enumerate}[label=(\roman*)]
        \item \textbf{Unregularized}: \(\lvert \widetilde{\tau}_{0}(g,g^{\prime}_{\perp}) \rvert \leq \varepsilon \lVert g\rVert_2 \lVert g^{\prime}_{\perp}\rVert_2/\lambda_{\min}^{+}(F)\).
        \item \textbf{Regularized}: \(\lvert \widetilde{\tau}_\lambda(g,g^{\prime}_{\perp}) \rvert \leq \varepsilon \lVert g\rVert_2 \lVert g^{\prime}_{\perp}\rVert_2 (1 / \lambda + 2\lVert F\rVert_2 / \lambda^2 )\).
    \end{enumerate}
    Moreover, in \Cref{thm:factorized-influence-leakage}, we show that a similar leakage guarantees extend to the factorized influence setting.
\end{mainresult}

Taken together, we develop a unified theory for when projection can provably approximate influence-style data attribution scores of the form \(g^{\top}(F+\lambda I)^{-1}g^{\prime}\). Specifically, without regularization, projection preserves influence for all \(g,g^{\prime}\in\range(F)\) only when the sketch is injective on \(\range(F)\), which essentially forces \(m \geq \rank(F)\); otherwise, uniform multiplicative approximation is impossible. With regularization, the required sketch size is instead governed by the effective dimension \(d_{\lambda}(F)\). We further extend these guarantees to Kronecker-factored (K-FAC-style) curvature and sketches. Finally, we quantify an additional sketch-induced leakage term that can appear when test gradients have components in \(\ker(F)\). Overall, our results provide principled, instance-adaptive guidance for choosing \(m\) and clarify how projection interacts with regularization and structured curvature approximations.

\subsection{Related Works}\label{subsec:related-works}
Influence functions were originally introduced as a classical tool in robust statistics~\citep{hampel1974influence} and later adapted to machine learning by \citet{koh2017understanding}. Owing to their flexibility and generality, influence-based methods have since been widely applied to tasks such as data cleaning~\citep{teso2021interactive}, model debugging~\citep{guo2021fastif}, and subset selection~\citep{hu2024most}, and have been extended to large-scale models, including large language models~\citep{grosse2023studying} and diffusion models~\citep{mlodozeniec2025influence}. However, applying influence functions to modern neural networks poses significant computational challenges due to the need to invert a high-dimensional, often rank-deficient, curvature matrix \(F\)~\citep{koh2017understanding,schioppa2022scaling}.

Several recent works propose scalable approximations based on random projection and related sketching techniques, where they typically project per-sample gradients into a lower-dimensional space before computing influence scores~\citep{wojnowicz2016influence,schioppa2022scaling,park2023trak}, sometimes in combination with explicit regularization~\citep{choe2024your,hu2025grass}. Despite their empirical success, the theoretical guarantees underlying these methods remain limited, and their correctness is often justified heuristically.

Specifically, existing theoretical justifications for projection-based influence methods typically appeal to the Johnson–Lindenstrauss (JL) lemma~\citep{lindenstrauss1984extensions} in the data attribution literature~\citep{wojnowicz2016influence,park2023trak,deng2025survey}. Given a finite set of vectors of size \(n\) in \(\mathbb{R}^d\), the JL lemma guarantees that \(m = O(\log(n) / \varepsilon^2)\) suffices to approximately preserve the pairwise distances between the \(n\) points up to a \((1\pm \varepsilon)\) factor. While powerful, this guarantee is fundamentally misaligned with the structure of influence functions. Influence scores are not determined by Euclidean distances between gradients, but by \emph{inverse-sensitive} bilinear forms \(\tau_0(g,g^{\prime}) = g^{\top} F^{-1} g^{\prime}\) involving the inverse (or pseudoinverse) of a second-order matrix \(F\), and sketching changes the operator to be inverted. Thus, preserving \(\lVert Pg \rVert _2\) (even uniformly over a finite set) does not directly control either the stability of matrix inversion after projection, nor the resulting bilinear form.

Consistent with this mismatch, empirical studies on hyperparameter sensitivity show that the quality of projected influence does not improve monotonically with the sketch size in certain scenarios~\citep{park2023trak}. More detailed ablation analyses further attribute this behavior to a coupled interaction between sketch size and regularization strength~\citep{wang2025taming}. Taken together, these observations underscore the need for a formal theoretical understanding of how projection interacts with the curvature operator, in order to guide the principled use of influence function methods in practice.
\section{Projection-Based Influence Approximation}\label{sec:projection-based-influence-approximation}

\subsection{Unregularized Projection}\label{subsec:unregularized-projection}
In this section, we show that in the absence of regularization, projection alone encounters a fundamental barrier in the sketch size \(m\). In particular, there is a sharp phase transition: when \(m < r\), no multiplicative approximation guarantee is possible; whereas when \(m \geq r\), a continuous sketch yields exact invariance with probability one.

\begin{theorem}[Barrier of unregularized projection]\label{thm:unregularized-projection}
    The equality \(\tau_0(g, g^{\prime}) = \widetilde{\tau}_0(g, g^{\prime})\) holds for any \(g, g^{\prime} \in\range(F)\) \emph{iff} \(P\) is injective on \(\range(F)\), i.e. \(\rank(PU) = \rank(F) = r\) where \(F=U\Lambda U^{\top}\) is the compact eigendecomposition of \(F\) with \(U\in\mathbb{R}^{d\times r}\) orthonormal and \(\Lambda\in\mathbb{R}^{r\times r}\) positive definite. Subsequently, for any PSD \(F \in \mathbb{R}^{d \times d}\) and \textbf{any} matrix \(P \in \mathbb{R}^{m \times d}\), one cannot hope to obtain any multiplicative approximation of \(\tau_0(g, g^{\prime})\) via \(\widetilde{\tau}_0(g, g^{\prime})\) when \(\rank(PU) < r\).
\end{theorem}

The proof can be found in \Cref{adxsec:unregularized-projection}. The key intuition is that, without regularization, influence depends on exact inversion over \(\range(F)\). Any collapse of directions within \(\range(F)\) renders \(F^{-1}\) ill-defined after sketching, hence no multiplicative control is possible.
Consequently, exact preservation of unregularized influence requires the sketch to be injective on \(\range(F)\), which in turn forces \(m\geq r\). In overparameterized regimes where high-dimensional per-sample gradients are in general position, one typically has \(r \approx n\), and thus \(m\) must scale with the dataset size. In contrast, we will show that introducing ridge regularization (\(\lambda > 0\)) fundamentally changes this requirement, with the sketch size governed instead by the effective dimension \(d_{\lambda}(F)\), which can be substantially smaller than \(r\).


\subsection{Regularized Projection}\label{subsec:regularized-projection}
Unlike the unregularized case, in this section, we show that for the projected influence function, the extra damping term \(\lambda I_d\) helps control the effective dimension by shrinking small eigenvalues of the curvature operator \(F\), effectively reducing the Gaussian complexity governing the uniform concentration bound. In particular, we show that the sketch size requires \emph{only} to scale with the \emph{effective dimension} of \(F\) with \(\lambda > 0\):
\[
    d_\lambda(F)
    \coloneqq \tr\big(F(F+\lambda I_d)^{-1}\big)
    =\sum_{j=1}^{r}\frac{\lambda_j(F)}{\lambda_j(F)+\lambda}
    \leq r.
\]
In practice, as we shall observe in \Cref{sec:experiment-discussion}, the spectrum of \(F\) decays rapidly, and thus \(d_\lambda \ll r \ll d\) for moderate \(\lambda\). Hence, having the sketch size \(m\) to only scale with the effective dimension \(d_\lambda\) at scale \(\lambda\) rather than the ambient dimension \(d\) or the rank \(r\) of \(F\) makes the regularized projection approach feasible at scale. We now state the theorem and sketch the proof below.

\begin{theorem}[Upper bound of regularized projection]\label{thm:regularized-projection-upper-bound}
    Let \(P \in \mathbb{R}^{m \times d}\) be a oblivious sketching matrix with rows \(P_{i}^{\top} = \frac{1}{\sqrt{m}} W_{i}^{\top}\), where \(\{W_{i}\}_{i=1}^m\sim W\) are i.i.d.\ sub-Gaussian random vectors in \(\mathbb{R}^d\) satisfying \(\mathbb{E}[W]=0\) and \(\mathbb{E}[W W^{\top}]=I_d\).\footnote{Since we only assume bounded sub-Gaussian norm on the random vectors \(W_i\), the result applies to a wide range of random projection matrices, including Gaussian, Rademacher, and sparse JL transform.} For any \(\varepsilon, \delta \in (0,1)\), if the sketch size satisfies
    \[
        m
        = \Omega \left( \frac{d_{\lambda}(F)+\log(1/\delta)}{\varepsilon^2} \right) ,
    \]
    then with probability at least \(1-\delta\), the following bounds hold for all \(g, g^{\prime} \in \range(F)\):
    \[
        \lvert \widetilde{\tau}_{\lambda}(g, g^{\prime}) - \tau_{\lambda}(g, g^{\prime}) \rvert
        \leq \varepsilon \sqrt{\tau_0(g, g)} \sqrt{\tau_0(g^{\prime}, g^{\prime})}.
    \]
\end{theorem}
\begin{proof}
    Let \(g, g^{\prime} \in \range(F)\) and write \(g = F^{1/2} y\) and \(g^{\prime} = F^{1/2} y^{\prime}\). Using the push-through identity \(A(A^{\top} A + \lambda I)^{-1} = (A A^{\top} + \lambda I)^{-1} A\) with \(A = P F^{1/2}\), and defining \(G \coloneqq F^{1/2} P^{\top} P F^{1/2}\) yields
    \[
        \begin{split}
            \widetilde{\tau}_\lambda(g, g^{\prime})
             & = (Pg)^{\top} (PFP^{\top} + \lambda I)^{-1} (Pg^{\prime})                        \\
             & = y^{\top} F^{1/2} P^{\top} (P F P^{\top} + \lambda I)^{-1} P F^{1/2} y^{\prime}
            = y^{\top} G (G+\lambda I)^{-1} y^{\prime} .
        \end{split}
    \]
    On the other hand, define \(B \coloneqq F^{1/2} (F + \lambda I)^{-1/2}\) as the \textit{\(\lambda\)-whitened influence subspace}. Since \(F\) and \(F+\lambda I\) are simultaneously diagonalizable (they share the eigenbasis of \(F\)), all matrix functions of these operators commute; in particular, \(F^{1/2}\), \((F+\lambda I)^{-1/2}\), and \((F+\lambda I)^{-1}\) commute and \(BB^{\top} = F^{1/2}(F+\lambda I)^{-1}F^{1/2} = F(F+\lambda I)^{-1}\). Hence, we have
    \[
        \tau_\lambda(g, g^{\prime})
        = g^{\top} (F + \lambda I)^{-1} g^{\prime}
        = y^{\top} B B^{\top} y^{\prime}
        = y^{\top} F (F + \lambda I)^{-1} y^{\prime},
    \]
    which gives \(\lvert \widetilde{\tau}_\lambda(g,g^{\prime}) - \tau_\lambda(g,g^{\prime}) \rvert = \lvert y^{\top} G(G+\lambda I)^{-1} y^{\prime} - y^{\top} F (F + \lambda I)^{-1} y^{\prime} \rvert\). Thus, it suffices to control the spectrum of \(F(F+\lambda I)^{-1} - G(G+\lambda I)^{-1}\).

    Note \(B^{\top}B = (F+\lambda I)^{-1/2}F(F+\lambda I)^{-1/2}\) and \(\lVert B \rVert_2^2 = \lVert B^{\top}B \rVert_2 \leq 1\). Applying \Cref{lma:covariance-concentration} with \(M=B\) and \(m = \Omega\big(\varepsilon^{-2}\big(d_\lambda(F)+\log(1/\delta)\big)\big)\) yields \(\lVert B^{\top}(P^{\top}P-I)B \rVert_2 \leq \varepsilon/2\). Conjugating by \((F+\lambda I)^{1/2}\) and using \(G=F^{1/2}P^{\top}PF^{1/2}\), this implies a PSD sandwich
    \[
        \left(1-\tfrac{\varepsilon}{2}\right)(F+\lambda I)
        \preceq
        (G+\lambda I)
        \preceq
        \left(1+\tfrac{\varepsilon}{2}\right)(F+\lambda I).
    \]
    Inverting the sandwich gives \(\lVert (G+\lambda I)^{-1} - (F+\lambda I)^{-1} \rVert_2 \leq \frac{1}{\lambda}\cdot \frac{\varepsilon/2}{1-\varepsilon/2} \leq \varepsilon / \lambda\). Finally, using the identity \(A(A+\lambda I)^{-1} = I - \lambda(A+\lambda I)^{-1}\) for any PSD \(A\), we get
    \[
        \left\lVert F(F+\lambda I)^{-1} - G(G+\lambda I)^{-1} \right\rVert_2
        = \lambda\left\lVert (G+\lambda I)^{-1} - (F+\lambda I)^{-1} \right\rVert_2
        \leq \varepsilon,
    \]
    which is the desired operator control (formal details are in \Cref{lma:concentration-covariance-operator-norm}). Therefore,
    \[
        \lvert \widetilde{\tau}_\lambda(g,g^{\prime}) - \tau_\lambda(g,g^{\prime}) \rvert
        = \left\lvert y^{\top} \left[ G(G+\lambda I)^{-1} - F (F + \lambda I)^{-1} \right] y^{\prime} \right\rvert
        \leq \varepsilon \lVert y \rVert _2 \lVert y^{\prime} \rVert _2.
    \]
    As \(\lVert y \rVert _2^2 = \tau_0(g, g)\) and \(\lVert y^{\prime} \rVert _2^2 = \tau_0(g^{\prime}, g^{\prime})\), we conclude the proof.
\end{proof}

\begin{remark}
    The core technical challenge in the proof of \Cref{thm:regularized-projection-upper-bound} is to bound \(\lVert F(F + \lambda I)^{-1} - G(G + \lambda I)^{-1} \rVert _2\). A natural alternative is to invoke an oblivious subspace embedding (OSE)~\citep{woodruff2014sketching}. For a fixed matrix \(A\in\mathbb{R}^{d\times r}\), \(P\in\mathbb{R}^{m\times d}\) is an \(\varepsilon\)-OSE for \(\range(A)\) if
    \[
        -\varepsilon A^\top A
        \preceq
        A^\top(P^\top P-I)A
        \preceq
        \varepsilon A^\top A .
    \]
    Instantiating \(A=F^{1/2}\) yields a sandwich \((1-\varepsilon)F\preceq G\preceq(1+\varepsilon)F\), which implies \(\|F(F+\lambda I)^{-1}-G(G+\lambda I)^{-1}\|_2=O(\varepsilon)\) by operator monotonicity of \(t\mapsto t/(t+\lambda)\) (see \Cref{adxsubsec:OSE}). 
    However, OSE enforces uniform multiplicative accuracy over \(\range(F^{1/2})\), so even directions with \(\lambda_j(F) \ll \lambda\) must be preserved up to a \((1 \pm \varepsilon)\) factor, leading to \(m=\Omega(r/\varepsilon^{2})\)~\citep[Theorems~2.3 and~6.10]{woodruff2014sketching}.

    Our proof instead exploits the weaker, \(\lambda\)-dependent requirement: it suffices for \(P\) to be an approximate isometry on the \(\lambda\)-whitened influence subspace \(B = F^{1/2}(F+\lambda I)^{-1/2}\), i.e., \(\lVert B^{\top}(P^{\top}P-I)B \rVert_2 \leq O(\varepsilon)\). This yields the \textbf{ridge-regularized sandwich} \((1-\varepsilon)(F+\lambda I)\preceq G+\lambda I\preceq(1+\varepsilon)(F+\lambda I)\). 
    Importantly, this condition controls \(F+\lambda I\) rather than \(F\) itself: in directions where \(\lambda_j(F)\ll\lambda\), both \(F+\lambda I\) and \(G+\lambda I\) are dominated by \(\lambda\), so even large relative errors in \(F\) have negligible impact on the inverse. Consequently, such low-eigenvalue directions need not be preserved multiplicatively, and the required sketch size is governed by the effective dimension at scale \(\lambda\),
    yielding the sharper bound \(m=\Omega(d_\lambda(F)/\varepsilon^2)\).
\end{remark}

We now complement \Cref{thm:regularized-projection-upper-bound} with a worst-case matching lower bound, showing that the effective dimension \(d_\lambda(F)\) characterizes the tight dependence of \(m\) for oblivious sketching in regularized influence. Concretely, we show that for Gaussian oblivious sketches, if the sketch size is smaller than \(\Theta(d_\lambda(F)/\varepsilon^2)\), then there exist problem instances on which the sketched influence incurs \(\Omega(\varepsilon)\) error with constant probability.

\begin{theorem}[Lower bound for regularized projection]\label{thm:regularized-projection-lower-bound}
    Let \( P \in \mathbb{R}^{m \times d} \) be a Gaussian oblivious sketch with rows i.i.d.\ \(\mathcal{N}(0,I_d)\). There exists a family of \(F \in \mathbb{R}^{d \times d}\) such that if \(m = o(d_\lambda(F)/\varepsilon^2)\), then there exists \(g \in \range(F)\) with \(\lvert \widetilde{\tau}_\lambda(g,g) - \tau_\lambda(g,g) \rvert = \Omega(\varepsilon) \tau_0(g,g)\) with constant probability.
\end{theorem}


Full details are in \Cref{adxsec:regulairzed-projection}. We see that \Cref{thm:regularized-projection-lower-bound} formalizes a worst-case limitation for this class of sketches: with Gaussian oblivious projections, one cannot uniformly beat the \(d_\lambda(F)/\varepsilon^2\) scaling. Combined with the instance-adaptive upper bound in \Cref{thm:regularized-projection-upper-bound}, this identifies \(d_\lambda(F)\) as the fundamental complexity parameter governing regularized projection.


\subsection{Factorized Influence}\label{subsec:factorized-influence}
In many large-scale settings, explicitly forming or inverting the empirical Fisher/Hessian \(F\) is infeasible, and second-order methods instead rely on structured approximations. A common choice is a Kronecker factorization (e.g., K-FAC~\citep{martens2015kfac,grosse2023studying}), which models each layerwise block as \(F \approx A \otimes E\) for smaller PSD factors \(A\in\mathbb{R}^{d_A\times d_A}\) and \(E\in\mathbb{R}^{d_E\times d_E}\), which are forward activation and backprop-gradient covariances, respectively.

This structure suggests a natural computational counterpart on the sketching side: use a \emph{factorized sketch} \(P=P_A\otimes P_E\), where \(P_A\in\mathbb{R}^{m_A\times d_A}\) and \(P_E\in\mathbb{R}^{m_E\times d_E}\) are respectively the standard oblivious sketching considered in \Cref{thm:regularized-projection-upper-bound}.\footnote{Concretely, rows of \(P_A\) and \(P_E\) are i.i.d.\ isotropic sub-Gaussian random vectors with scaling \(1 / \sqrt{m_A}\) or \(1 / \sqrt{m_E}\).} The resulting sketch has ambient dimension \(d\coloneqq d_A d_E\) and sketch dimension \(m\coloneqq m_A m_E\), i.e., \(P\in\mathbb{R}^{m\times d}\). Moreover, write a per-example layer gradient as a matrix \(G\in\mathbb{R}^{d_E\times d_A}\) with \(g=\operatorname{vec}(G)\in\mathbb{R}^{d}\). Then the projection can be computed without materializing the full \(m\times d\) sketching matrix as \(Pg = (P_A\otimes P_E)\operatorname{vec}(G) = \operatorname{vec}(P_E G P_A^{\top})\). Consequently, the per-example cost reduces to two smaller multiplies \(P_E G\) and \((P_E G)P_A^{\top}\), plus solving the resulting regularized system in sketch dimension \(m\). Similarly, we can also form the sketched curvature efficiently: using the mixed-product identity of Kronecker products, \(PFP^{\top} = (P_A\otimes P_E)(A\otimes E)(P_A\otimes P_E)^{\top} = (P_A A P_A^{\top})\otimes (P_E E P_E^{\top})\).

In the unregularized case (\(\lambda=0\)), the exact invariance barrier becomes strictly more stringent under a Kronecker sketch: exact preservation on \(\range(F)\) holds if and only if \emph{both} factor sketches are injective on their respective ranges.

\begin{theorem}[Barrier of unregularized projection for factorized influence]\label{thm:unregularized-factorized-projection}
    Let \(F=A\otimes E\succeq 0\) and \(P=P_A\otimes P_E\) as above. Then \(\widetilde{\tau}_0(g,g^{\prime})=\tau_0(g,g^{\prime})\) for all \(g,g^{\prime}\in\range(F)\) if and only if \(P_A\) is injective on \(\range(A)\) and \(P_E\) is injective on \(\range(E)\). In particular, this necessitates \(m_A\geq \rank(A)\) and \(m_E\geq \rank(E)\), hence \(m=m_A m_E\geq \rank(A)\rank(E)=\rank(F)\).
\end{theorem}

See \Cref{adxsubsec:unregularized-factorized-projection} for a proof. This motivates integrating regularization and consider
\[
    \begin{split}
        \widetilde{\tau}_\lambda(g,g^{\prime})
         & = (Pg)^{\top} (PFP^{\top} + \lambda I_m)^{-1} (Pg^{\prime})                                                                                                               \\
         & = \operatorname{vec}(P_E G P_A^{\top})^{\top} \left((P_A A P_A^{\top})\otimes (P_E E P_E^{\top}) + \lambda I_m\right)^{-1} \operatorname{vec}(P_E G^{\prime} P_A^{\top}).
    \end{split}
\]
However, factorization changes the sketching analysis: when \(P=P_A\otimes P_E\), the matrix \(P^{\top}P\) is no longer a standard i.i.d.\ sample covariance, so the covariance-type deviation driving the proof of \Cref{thm:regularized-projection-upper-bound} requires a dedicated argument. We now present the corresponding approximation guarantee for regularized projection under this factorized model. The key technical step is a factorized covariance deviation bound (\Cref{thm:factorized-covariance}), proved in \Cref{adxsec:factorized-influence}.

\begin{theorem}[Upper bound of regularized projection for factorized influence]\label{thm:regularized-factorized-projection-upper-bound}
    Let \(F=A\otimes E\succeq 0\) and \(P=P_A\otimes P_E\) be as above, with the factors \(P_A, P_E\) denote the sketching matrix defined in \Cref{thm:regularized-projection-upper-bound}. Assume \(\lambda \leq \lVert A \rVert_2\lVert E \rVert_2\), and define the rescaled regularization levels \(\lambda_E \coloneqq \lambda / \lVert E \rVert _2\) and \(\lambda _A \coloneqq \lambda / \lVert A \rVert _2\). For any \(\varepsilon,\delta\in(0,1)\), if the sketch sizes for \(P_A\) and \(P_E\) satisfy
    \[
        m_A = \Omega \left(\frac{d_{\lambda_E}(A)+\log(1/\delta)}{\varepsilon^2}\right), \qquad
        m_E = \Omega \left(\frac{d_{\lambda_A}(E)+\log(1/\delta)}{\varepsilon^2}\right),
    \]
    then with probability at least \(1-\delta\), the following holds for all \(g,g^{\prime}\in\range(F)\):
    \[
        \lvert \widetilde{\tau}_\lambda(g,g^{\prime}) - \tau_\lambda(g,g^{\prime}) \rvert
        \leq \varepsilon \sqrt{\tau_0(g,g)}\sqrt{\tau_0(g^{\prime},g^{\prime})}.
    \]
\end{theorem}
\begin{proof}
    The proof follows the same template as \Cref{thm:regularized-projection-upper-bound}. Let \(B \coloneqq F^{1/2}(F+\lambda I)^{-1/2}\) and \(G \coloneqq F^{1/2}P^{\top}PF^{1/2}\), and the key step is again to control the covariance-type deviation \(\lVert B^{\top}(P^{\top}P-I)B\rVert_2\). We apply \Cref{thm:factorized-covariance} (proved in \Cref{adxsec:factorized-influence}) with parameters \(\varepsilon_0 \coloneqq \varepsilon/10\) and \(\delta_0 \coloneqq \delta/2\). Under the stated conditions on \(m_A\) and \(m_E\), this yields that with probability at least \(1-2\delta_0 = 1-\delta\),
    \[
        \lVert B^{\top}(P^{\top}P-I)B\rVert_2
        \leq 2\varepsilon_0 + 3\varepsilon_0^2
        \leq \varepsilon/2,
    \]
    where the last inequality uses \(\varepsilon\in(0,1)\). On this event, the same PSD sandwich and resolvent perturbation argument used in \Cref{lma:concentration-covariance-operator-norm} implies \(\lVert F(F+\lambda I)^{-1} - G(G+\lambda I)^{-1} \rVert_2 \leq \varepsilon\), which in turn gives the stated bilinear (and quadratic) influence error bounds.
\end{proof}

\begin{remark}
    \Cref{thm:regularized-factorized-projection-upper-bound} highlights a fundamental computational--statistical trade-off. While factorized sketches offer substantial computational and memory advantages over unfactorized ones, they incur a higher statistical cost in terms of the required sketch size. In particular, since the total sketch size is \(m=m_A m_E\), achieving an \(\varepsilon\)-approximation error requires \(m = m_A m_E = \widetilde{\Omega}\big(\varepsilon^{-4} (d_{\lambda_E}(A) d_{\lambda_A}(E))\big)\), which exhibits a worse dependence on \(\varepsilon\) (from \(\varepsilon^{-2}\) to \(\varepsilon^{-4}\)) compared to the unfactorized sketch guarantee in \Cref{thm:regularized-projection-upper-bound}. 
    Importantly, this gap is not an artifact of loose analysis, but follows from the separable nature of the factorized sketch: \(P_A\) and \(P_E\) must independently satisfy an \(\varepsilon\)-level concentration bound at its own regularization scale. Consequently, the total sketch size reflects the product of the factor-level requirements. As a result, factorized sketches are most effective in regimes where the computational and memory savings from separability outweigh the increased statistical overhead. 
\end{remark}
\section{Influence with Out-of-Range Test Gradients}\label{sec:regularized-projection-leakage}
The analysis in \Cref{sec:projection-based-influence-approximation} assumes that both arguments of the (regularized) influence bilinear form lie in \(\range(F)\). This assumption is natural for training gradients: when \(F\) is instantiated as the (empirical) Fisher information matrix, \(F = \frac{1}{n}\sum_{i=1}^n g_{i} g_{i}^{\top}\), every training gradient lies in \(\range(F)\) by construction. In practice, however, we are often interested in the influence of an unseen test point \(z'\) with respect to a training point \(z\), for which the corresponding test gradients \(g^{\prime}\) need not lie in \(\range(F)\).

We extend the above guarantees to this setting by explicitly characterizing the additional sketch-induced error arising from the component of \(g^{\prime}\) orthogonal to \(\range(F)\). 

\subsection{Leakage of Projection}\label{subsec:projection-leakage}
To make the source of this additional term explicit, we decompose \(g^{\prime} = g^{\prime}_{\parallel} + g^{\prime}_{\perp}\), where \(g^{\prime}_{\parallel} \in \range(F)\) and \(g^{\prime}_{\perp} \in \ker(F)\), such that the decomposition is orthogonal in the Euclidean inner product. Using linearity of \(\tau_\lambda(\cdot,\cdot)\) and \(\widetilde{\tau}_\lambda(\cdot,\cdot)\) in their second argument, we have \(\tau_\lambda(g,g^{\prime}) = \tau_\lambda(g,g^{\prime}_{\parallel}) + \tau_\lambda(g,g^{\prime}_{\perp})\) and \(\widetilde{\tau}_\lambda(g,g^{\prime}) = \widetilde{\tau}_\lambda(g,g^{\prime}_{\parallel}) + \widetilde{\tau}_\lambda(g,g^{\prime}_{\perp})\). Consequently,
\[
    \lvert \widetilde{\tau}_\lambda(g,g^{\prime}) - \tau_\lambda(g,g^{\prime}) \rvert
    = \left\lvert \big(\widetilde{\tau}_\lambda(g,g^{\prime}_{\parallel}) - \tau_\lambda(g,g^{\prime}_{\parallel}) \big) + \widetilde{\tau}_\lambda(g,g^{\prime}_{\perp}) - \tau_\lambda(g,g^{\prime}_{\perp}) \right\rvert .
\]
Observe that the true (regularized) influence does not couple \(\range(F)\) and \(\ker(F)\), i.e., \(\tau_\lambda(g,g^{\prime}_{\perp}) = g^{\top} (F+\lambda I)^{-1} g^{\prime}_{\perp} = 0\) for all \(\lambda \geq 0\): indeed, \(F\) and \((F+\lambda I)^{-1}\) share the same eigenbasis, and since \(g \in \range(F)\) and \(g^{\prime}_{\perp} \in \ker(F)\), hence \((F+\lambda I)^{-1} g\) and \(g^{\prime}_{\perp}\) lie in orthogonal subspaces. Hence,
\[
    \lvert \widetilde{\tau}_\lambda(g,g^{\prime}) - \tau_\lambda(g,g^{\prime}) \rvert
    \leq \lvert \widetilde{\tau}_\lambda(g,g^{\prime}_{\parallel}) - \tau_\lambda(g,g^{\prime}_{\parallel}) \rvert + \lvert \widetilde{\tau}_\lambda(g,g^{\prime}_{\perp}) \rvert .
\]
The first term can be bounded via \Cref{thm:regularized-projection-upper-bound}; on the other hand, the remaining term is a purely sketch-induced artifact: the sketch can introduce a nonzero \emph{leakage term} \(\widetilde{\tau}_\lambda(g,g^{\prime}_{\perp})\) due to mixing between \(\range(F)\) and \(\ker(F)\) under \(P^{\top}P\). We now present a general bound on the leakage:

\begin{theorem}\label{thm:projection-leakage}
    Let \(\{g^{\prime}_{j}\}_{j=1}^k \subset \mathbb{R}^d\), and for each \(j\) let \(g^{\prime}_{j,\perp} \coloneqq \Pi_{\ker(F)} g^{\prime}_j\) denote the orthogonal projection of \(g_j^{\prime}\) onto \(\ker(F)\). Let \(k^{\prime} \coloneqq \dim\bigl(g^{\prime}_{j, \perp}\}_{j=1}^k)\bigr)\).
    For any \(\varepsilon, \delta \in (0,1)\), if
    \[
        m
        = \Omega \Bigg(\frac{r + \min\left\{\log(k/\delta), k^{\prime} + \log(1/\delta)\right\}}{\varepsilon^2}\Bigg),
    \]
    then with probability at least \(1-\delta\), the following holds for all \(j \in \{1, \dots, k\}\):
    \begin{itemize}[leftmargin=*]
        \item \textbf{Unregularized}: \(\lvert \widetilde{\tau}_0(g, g^{\prime}_{j,\perp}) \rvert \leq \varepsilon \lVert g \rVert _2\lVert g^{\prime}_{j,\perp}\rVert _2 / \lambda_{\min}^{+}(F)\), where \(\lambda_{\min}^{+}(F)\) denotes the smallest non-zero eigenvalue of \(F\).
        \item \textbf{Regularized}: \(\lvert \widetilde{\tau}_{\lambda}(g, g^{\prime}_{j,\perp}) \rvert \leq \varepsilon \lVert g \rVert _2\lVert g^{\prime}_{j,\perp}\rVert _2 \big(\frac{1}{\lambda} + \frac{2 \lVert F \rVert _2}{\lambda^2}\big)\) for any \(\lambda > 0\).
    \end{itemize}
\end{theorem}
\begin{proofsketch}
    The proof is organized around a deterministic reduction: \Cref{lma:two-conc-suffice} (in \Cref{adxsec:projection-leakage}) shows that both the unregularized and regularized leakage bounds follow as soon as two concentration conditions hold for the sketch \(P\):
    \begin{enumerate*}[label=(\roman*)]
        \item an operator-norm bound on \(\range(F)\), \(\lVert U^{\top}(P^{\top}P-I)U\rVert_2\leq\varepsilon\) for an orthonormal basis \(U\) of \(\range(F)\), and
        \item a cross-term bound between \(\range(F)\) and the kernel direction(s), \(\lVert U^{\top}(P^{\top}P-I)g^{\prime}\rVert_2\leq\varepsilon\lVert g^{\prime}\rVert_2\).
    \end{enumerate*}
    For a single test gradient \(g_{\perp}^{\prime}\), both conditions follow from applying \Cref{lma:covariance-concentration} to the \((r+1)\)-dimensional subspace \(\Span(\range(F)\cup\{g^{\prime}\})\), which yields the claimed \(m=\Omega((r+\log(1/\delta))/\varepsilon^2)\) scaling. 
    To obtain uniform control over multiple test gradients, we use two complementary arguments: a \emph{subspace argument}, which applies the same concentration bound to \(\Span(\range(F)\cup\{g^{\prime}_{j,\perp}\}_{j=1}^k)\) and yields the dependence on \(k^{\prime}=\dim\Span(\{g^{\prime}_{j,\perp}\})\) (\Cref{prop:multiple-k-prime}); or a \emph{union-bound argument}, which establishes a fixed-\(g^{\prime}\) tail bound and unions over \(k\), yielding the \(O(\log k)\) dependence (\Cref{prop:multiple-k}).
\end{proofsketch}

\subsection{Leakage of Factorized Influence}\label{subsec:factorized-influence-leakage}
\Cref{thm:projection-leakage} is stated for oblivious sketches with i.i.d.\ rows. We now extend and prove an analogous leakage guarantee for factorized sketches \(P=P_A\otimes P_E\) when \(F\) admits a Kronecker factorization.

\begin{theorem}\label{thm:factorized-influence-leakage}
    Let \(A, E \succeq 0\) and \(F \coloneqq A \otimes E\), with \(P = P_A \otimes P_E\) be the same setting as \Cref{thm:regularized-factorized-projection-upper-bound}, and let \(r_A\coloneqq\rank(A)\), \(r_E\coloneqq\rank(E)\), and \(r\coloneqq\rank(F)=r_A r_E\). Let \(\{g_j^{\prime}\}_{j=1}^k\) be test gradients of the form \(g_j^{\prime}=a_j^{\prime}\otimes e_j^{\prime}\), and write \(a_j^{\prime}=a^{\prime}_{j,\parallel}+a^{\prime}_{j,\perp}\) with \(a^{\prime}_{j,\parallel}\in\range(A)\) and \(a^{\prime}_{j,\perp}\perp\range(A)\), and similarly \(e_j^{\prime}=e^{\prime}_{j,\parallel}+e^{\prime}_{j,\perp}\). Define \(k_A \coloneqq \sum_{j=1}^k \mathbbm{1}(a^{\prime}_{j,\perp}\neq 0)\), \(k_E \coloneqq \sum_{j=1}^k \mathbbm{1}(e^{\prime}_{j,\perp}\neq 0)\), and \(k_A^{\prime} \coloneqq \dim\bigl(\Span(\{a^{\prime}_{j,\perp}\}_{j=1}^k)\bigr)\), \(k_E^{\prime} \coloneqq \dim\bigl(\Span(\{e^{\prime}_{j,\perp}\}_{j=1}^k)\bigr)\). For any \(\varepsilon,\delta\in(0,1)\), if
    \[
        m_A = \Omega\Bigg(\frac{r_A+\min\{\log(\frac{k_A}{\delta}), k_A^{\prime}+\log(\frac{1}{\delta})\}}{\varepsilon^2}\Bigg),
        m_E = \Omega\Bigg(\frac{r_E+\min\{\log(\frac{k_E}{\delta}), k_E^{\prime}+\log(\frac{1}{\delta})\}}{\varepsilon^2}\Bigg),
    \]
    then with probability at least \(1-\delta\), the following bounds hold simultaneously for all \(j\in\{1,\dots,k\}\):
    \begin{itemize}[leftmargin=*]
        \item \textbf{Unregularized:} \(\lvert \widetilde{\tau}_0(g, g^{\prime}_{j,\perp}) \rvert \leq \varepsilon  \lVert g \rVert _2\lVert g^{\prime}_{j,\perp}\rVert _2 / \lambda_{\min}^{+}(F)\).
        \item \textbf{Regularized:} \(\lvert \widetilde{\tau}_{\lambda}(g, g^{\prime}_{j,\perp}) \rvert \leq \varepsilon  \lVert g \rVert _2\lVert g^{\prime}_{j,\perp}\rVert _2 \bigl(\frac{1}{\lambda} + \frac{2\lVert F \rVert _2}{\lambda^2}\bigr)\) for any \(\lambda>0\),
    \end{itemize}
\end{theorem}
\begin{proofsketch}
    The factorized theorem is proved by following the same high-level template as \Cref{thm:projection-leakage}: we first reduce the leakage bound to the two concentration conditions in \Cref{lma:two-conc-suffice} (stability on \(\range(F)\) and a cross-term bound between \(\range(F)\) and \(\ker(F)\)). For a Kronecker sketch \(P=P_A\otimes P_E\), the stability condition on \(\range(F)=\range(A)\otimes\range(E)\) is obtained by controlling the factor-level subspace deviations \(\lVert U_A^{\top}(P_A^{\top}P_A-I)U_A\rVert_2\) and \(\lVert U_E^{\top}(P_E^{\top}P_E-I)U_E\rVert_2\) (with \(U_A,U_E\) bases of \(\range(A),\range(E)\)). For the cross-term condition, we expand \(P^{\top}P-I\) into factor deviations and use \Cref{lma:kfac-cross-term} (in \Cref{adxsec:factorized-influence-leakage}) to reduce \(\lVert U^{\top}(P^{\top}P-I)g^{\prime}_{\perp}\rVert_2\) to a small collection of factor-level ``primitive'' quantities such as \(\lVert U_A^{\top}(P_A^{\top}P_A-I)(\cdot)\rVert_2\) and \(\lVert U_E^{\top}(P_E^{\top}P_E-I)(\cdot)\rVert_2\). Finally, as in the proof of~\Cref{thm:projection-leakage}, these primitives are controlled via a \emph{union-bound argument} (yielding the \(O(\log k_\cdot)\) terms) or a \emph{subspace argument} (yielding the \(k_\cdot^{\prime}\) terms). Plugging these bounds into \Cref{lma:two-conc-suffice} yields the stated leakage guarantees; full details are in \Cref{adxsec:factorized-influence-leakage}.
\end{proofsketch}

\begin{remark}
Which argument is tighter depends on the geometry of the test gradients. When \(\{g_j^{\prime}\}\) are strongly correlated or effectively low-dimensional, one can have \(k^{\prime}\ll k\), in which case the subspace argument is preferable. In contrast, in high ambient dimension, moderately many generic test gradients are typically in general position, so \(k^{\prime}\) rapidly grows to \(\min\{k,d\}\) and in particular satisfies \(k^{\prime}\approx k\) once \(k\ll d\). In this common regime, the union-bound argument yields the more practical scaling in \(k\), requiring only an additional \(O(\log k)\) sketch size to ensure uniform control.
\end{remark}
\section{Experiment and Discussion}\label{sec:experiment-discussion}
We empirically illustrate several implications of our theory. Throughout, we consider \(F\) to be the empirical Fisher, and \(P\) to be the sparse JL transform~\citep{kane2014sparser}, and we always report the results across \(5\) independent runs with different sampled \(P\). Following the data attribution library \texttt{dattri}~\citep{deng2024dattri},\footnote{Our code is publicly available at \url{https://github.com/sleepymalc/Projection-IF}.} we consider three dataset--model pairs:
\begin{enumerate*}[label=\arabic*.)]
    \item MNIST-10 + LR,
    \item MNIST-10 + MLP, and
    \item CIFAR-2 + ResNet9.
\end{enumerate*}
Each setting uses \(5000\) training examples and \(500\) held-out test examples, so the empirical Fisher has rank at most \(r\leq 5000\).

\begin{wrapfigure}[10]{r}{0.7\linewidth}
    \centering
    \vspace{-1\intextsep}
    \includegraphics[width=\linewidth]{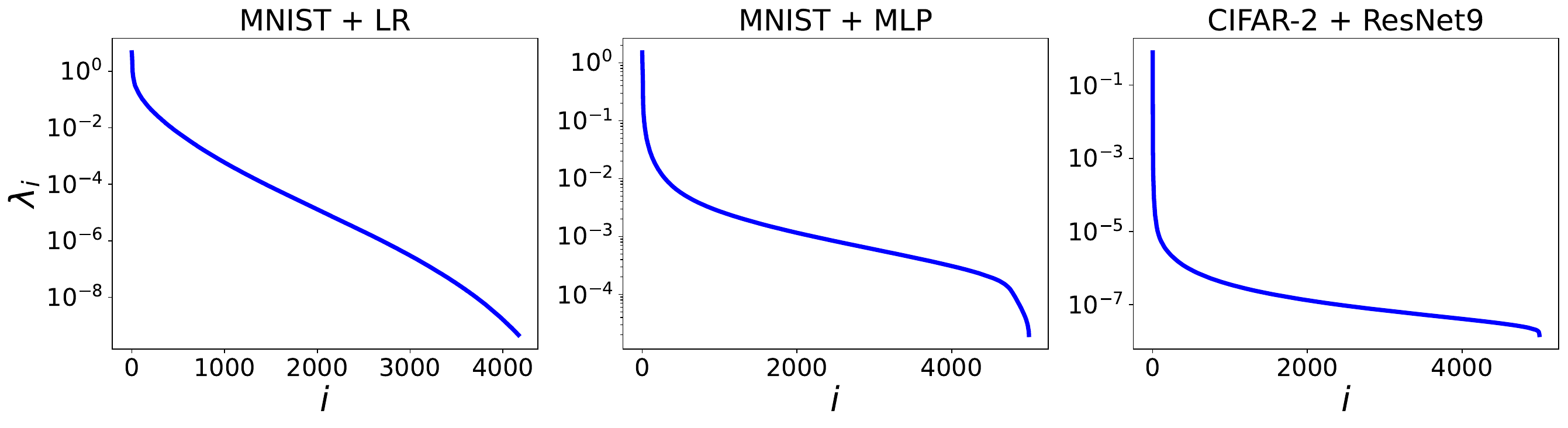}%
    \caption{Ordered spectrum \(\lambda_i\) of the empirical Fisher \(F\).}
    \label{fig:spectrum}
\end{wrapfigure}
Firstly, we show the effective dimension \(d_{\lambda}(F) = \sum_{i=1}^{r} \lambda_i /(\lambda_i + \lambda)\) can be much smaller than \(r = \rank(F)\). Specifically, \Cref{fig:spectrum} plots the ordered eigenvalues \(\{\lambda_i\}_{i=1}^r\) of \(F\). The spectrum decays quickly, hence for moderate \(\lambda\), the terms \(\lambda_i /(\lambda_i+\lambda)\) become small for large \(i\), and consequently \(d_{\lambda}(F)\) can be far smaller than \(r\).

We next test the predictions of \Cref{thm:regularized-projection-upper-bound,thm:projection-leakage} by directly measuring the approximation error. Given \(\lambda\geq 0\), we consider \(\varepsilon_{\lambda}(g, g^{\prime}) = \lvert \widetilde{\tau}_{\lambda}(g, g^{\prime}) - \tau_{\lambda}(g, g^{\prime}) \rvert / \sqrt{\tau_0(g, g)}\sqrt{\tau_0(g^{\prime}, g^{\prime})}\) for gradients \(g\) and \(g^{\prime}\), which is the normalized error considered in \Cref{thm:regularized-projection-upper-bound}. 

\begin{wrapfigure}[11]{l}{0.75\linewidth}
    \centering
    \vspace{-1\intextsep}
    \includegraphics[width=\linewidth]{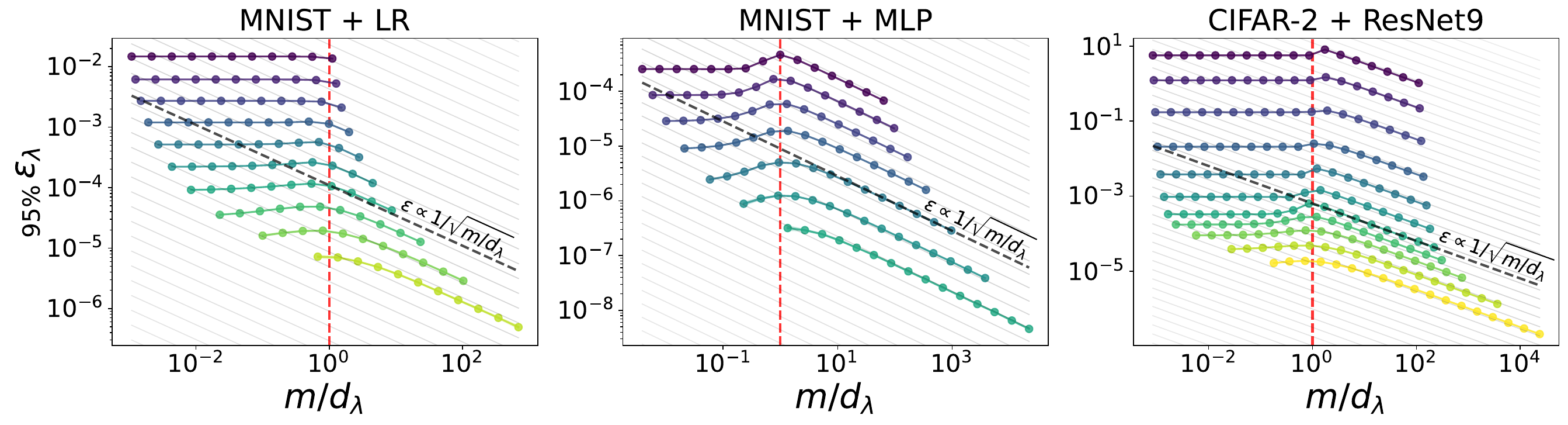}
    \includegraphics[width=0.4\linewidth]{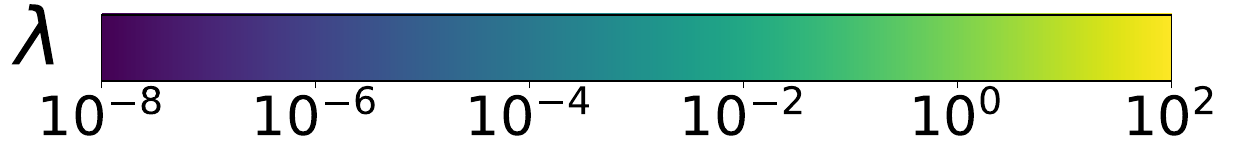}%
    \caption{Approximation error versus normalized sketch size.}
    \label{fig:spectrum-bounds-error}
\end{wrapfigure}

\Cref{fig:spectrum-bounds-error} supports the scaling predicted by our theory. Each curve plots the \(95^{\text{th}}\) percentile of \(\varepsilon_{\lambda}(g,g^{\prime})\) against the normalized sketch size \(m/d_{\lambda}(F)\). Once \(m\) is on the order of \(d_{\lambda}(F)\), the error begins to decay in the manner suggested by \Cref{thm:regularized-projection-upper-bound}. Empirically, this indicates that
\begin{enumerate*}[label=(\roman*)]
    \item the hidden constant in the sketch-size requirement is modest and
    \item the additional leakage effect from \Cref{thm:projection-leakage} decreases quickly as \(m\) grows.
\end{enumerate*}

\paragraph{Faithfulness--Utility Tradeoff.}
A small approximation error does not necessarily imply strong downstream performance. In particular, optimizing \(\varepsilon_{\lambda}\) to be very small typically favors larger \(\lambda\) and larger sketch size \(m\), because stronger regularization makes the influence computation less sensitive to sketching. As a result, the \(\lambda\) that minimizes \(\varepsilon_{\lambda}\) need not be the \(\lambda\) that maximizes downstream utility, especially when the curvature information in \(F\) is important for the task. We illustrate this using LDS~\citep{park2023trak}, a standard metric in data attribution. \Cref{fig:faithfulness-utility} reports LDS over a range of sketch sizes and regularization strengths, and as we expect, the best-performing \(\lambda^{\ast}\) is typically intermediate.

\begin{figure}[htpb]
    \centering
    \includegraphics[width=\linewidth]{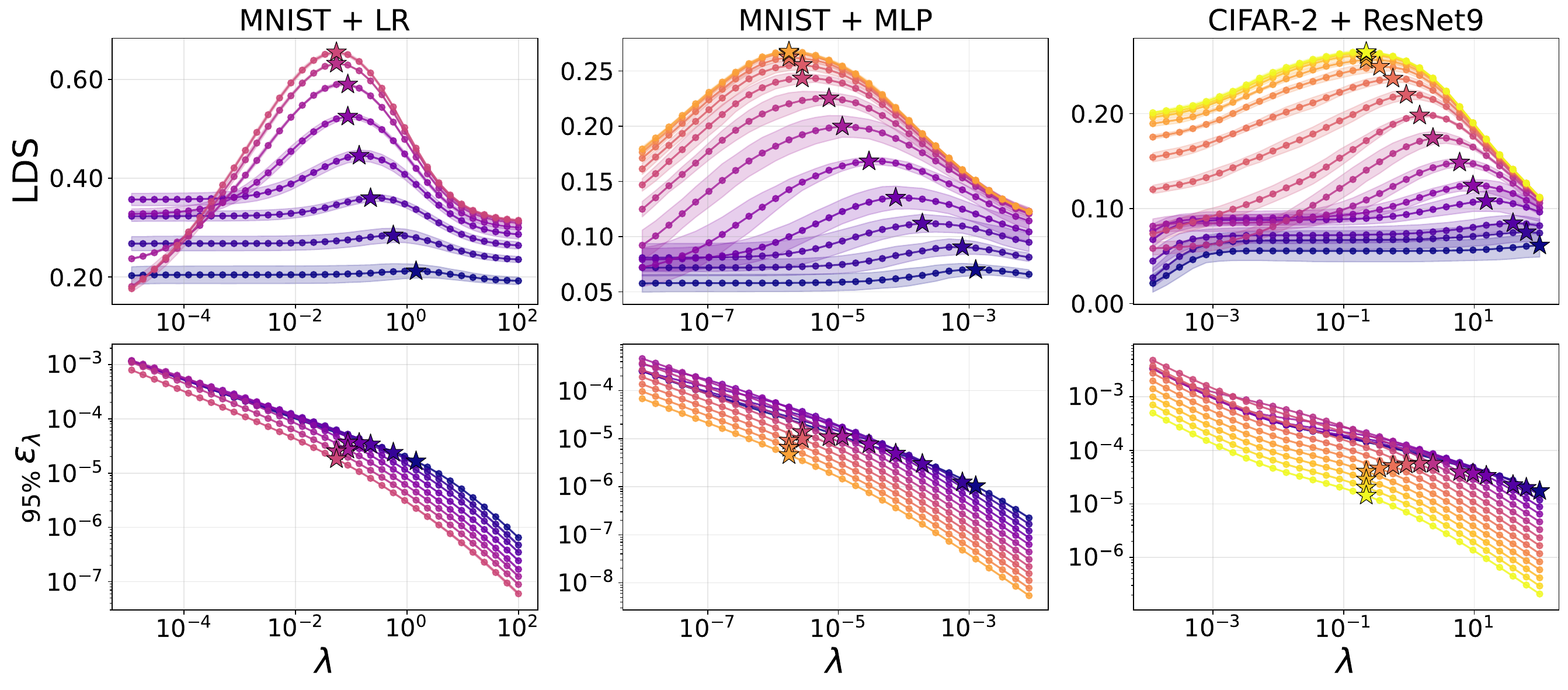}
    \includegraphics[width=0.23\linewidth]{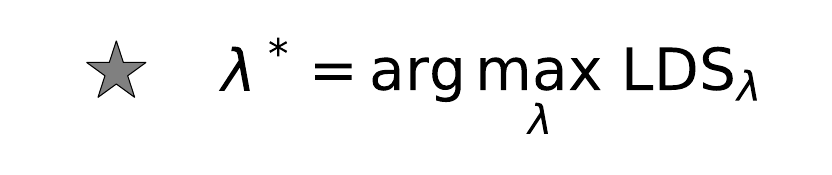}
    \includegraphics[width=0.3\linewidth]{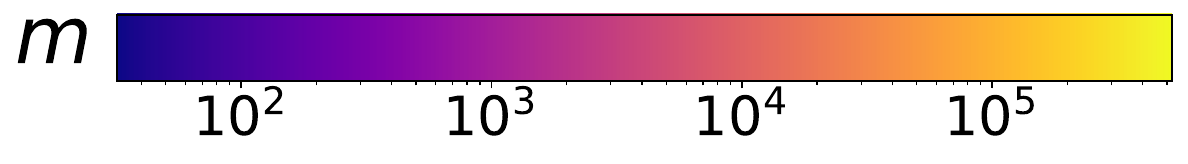}%
    \caption{Approximation error and LDS versus \(\lambda\).}
    \label{fig:faithfulness-utility}
\end{figure}

\begin{wrapfigure}[13]{l}{0.62\linewidth}
    \centering
    \includegraphics[width=\linewidth]{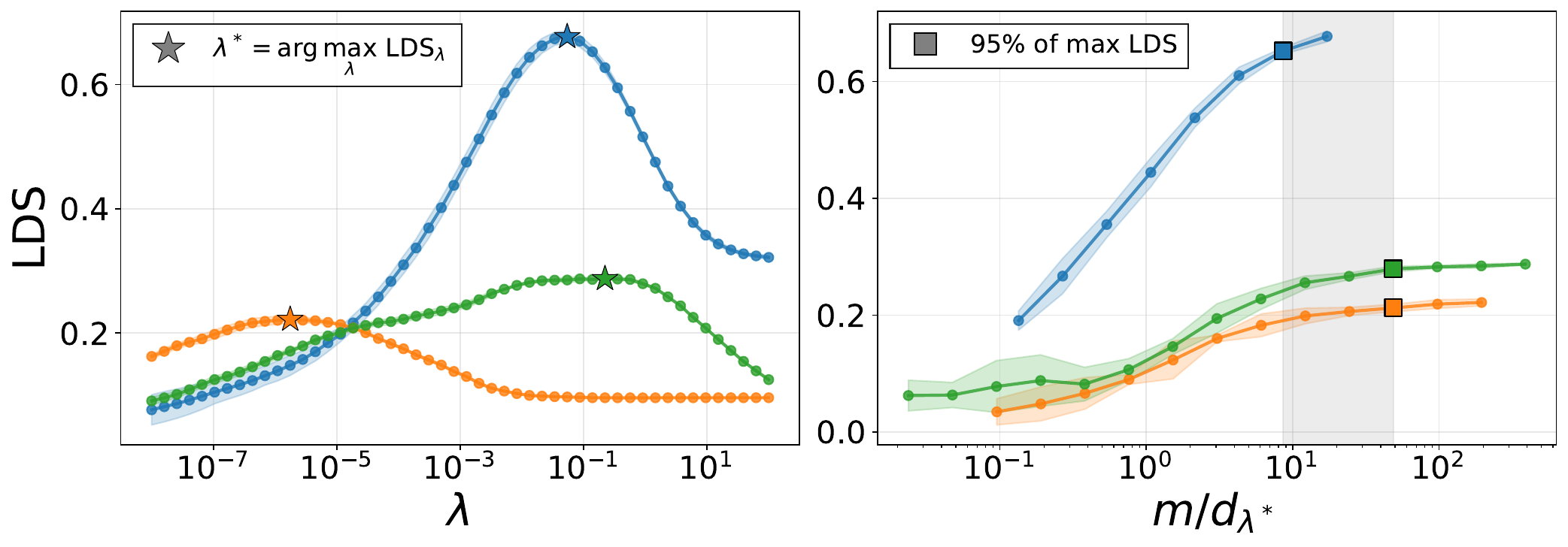}
    \includegraphics[width=\linewidth]{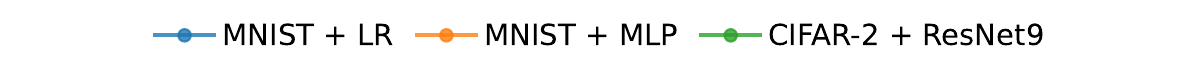}%
    \caption{Left: selecting \(\lambda^{\ast}\) on a validation set using large \(m\). Right: held-out test LDS versus \(m / d_{\lambda^{\ast}}(F)\).}
    \label{fig:hyperparameter-selection}
\end{wrapfigure}

These observations suggest a simple two-stage procedure. First, using a small validation set and a sufficiently large sketch size \(m\), sweep over \(\lambda\) and select \(\lambda^{\ast}\) that maximizes the downstream metric. Second, fix \(\lambda=\lambda^{\ast}\) and increase \(m\) until \(m \gtrsim C  d_{\lambda^{\ast}}(F)\), which ensures that the influence estimates are faithful. \Cref{fig:hyperparameter-selection} illustrates this strategy for LDS: the square markers in the right panel (\(95^{\text{th}}\) percentile LDS) indicate how large \(m\) must be to approach the best attainable LDS. In our experiments, a constant \(C\in(10,100)\) is sufficient, making the dependence on \(d_{\lambda^{\ast}}(F)\) operational.
\section{Conclusion}\label{sec:conclusion}
In this work, we show that projection-based influence is governed by the interaction between the sketch and the curvature operator, and that conventional Johnson--Lindenstrauss arguments, which only control Euclidean geometry, are misaligned with inverse-sensitive influence computations~\citep{park2023trak,schioppa2024efficient,hu2025grass}. By characterizing how projection interacts with common techniques such as ridge regularization and structured curvature approximations, our unified theory provides principled and actionable guidance for applying influence functions reliably at scale.

Our analysis also points to several important directions for future work. First, extending the theory to more sophisticated curvature approximations such as EK-FAC~\citep{george2018fast,grosse2023studying} remains highly nontrivial. Unlike standard K-FAC, EK-FAC introduces an additional eigenvalue correction step, which fundamentally alters the spectral structure of the curvature operator and breaks many of the techniques used in our analysis. Developing a projection theory that accounts for this additional structure is an interesting and challenging open problem.

Second, this work focuses on the approximation quality of projected influence relative to its unprojected counterpart. We do not study how projection, regularization, or curvature approximations affect the quality of influence functions as estimators of the underlying leave-one-out (LOO) quantity that influence functions are designed to approximate. In particular, correlation-based metrics such as LOO correlation or LDS~\citep{park2023trak} reflect not only \emph{approximation error} introduced by projection, but also \emph{modeling bias} arising from regularization and curvature approximations. From this perspective, our results show that in the regularized setting, low-curvature directions can be safely discarded without degrading approximation quality with respect to the regularized influence function. In contrast, recent work~\citep{wang2025better} suggests that these same low-curvature directions may play a crucial role in achieving high-quality influence estimates when evaluated against ground-truth LOO effects. Understanding how projection and other techniques jointly affect both approximation error and modeling bias remains an important open direction.

Overall, we view this work as a step toward a more principled understanding of scalable data attribution, and we hope it motivates further theoretical and empirical investigation into the interplay between projection, regularization, curvature, and evaluation criteria in influence functions.

\acks{
YH and HZ are partially supported by an NSF CAREER Award No. 2442290 and an NSF IIS Grant No. 2416897. Part of this work was conducted while YH and HZ were visiting the Simons Institute for the Theory of Computing. We thank Joseph Melkonian and David Woodruff for helpful discussions that motivated this study.
}

\addcontentsline{toc}{section}{References}
\bibliography{reference}

\crefalias{section}{appendix} 
\crefalias{subsection}{appendix} 

\newpage
\appendix

\setcounter{tocdepth}{2}
\tableofcontents


\section{Proofs for \texorpdfstring{\Cref{subsec:unregularized-projection}}{Section 2.1} (Unregularized Projection)}\label{adxsec:unregularized-projection}
In this section, we prove \Cref{thm:unregularized-projection}, which we first repeat the statement for convenience:

\begin{theorem*}
    The equality \(\tau_0(g, g^{\prime}) = \widetilde{\tau}_0(g, g^{\prime})\) holds for any \(g, g^{\prime} \in\range(F)\) \emph{iff} \(P\) is injective on \(\range(F)\), i.e. \(\rank(PU) = \rank(F) = r\) where \(F=U\Lambda U^{\top}\) is the compact eigendecomposition of \(F\) with \(U\in\mathbb{R}^{d\times r}\) orthonormal and \(\Lambda\in\mathbb{R}^{r\times r}\) positive definite. Subsequently, for any PSD \(F \in \mathbb{R}^{d \times d}\) and \textbf{any} matrix \(P \in \mathbb{R}^{m \times d}\), one cannot hope to obtain any multiplicative approximation of \(\tau_0(g, g^{\prime})\) via \(\widetilde{\tau}_0(g, g^{\prime})\) when \(\rank(PU) < r\).
\end{theorem*}
\begin{proof}
    For the ``if'' direction, suppose \(\rank(PU)=r\). Let \(A:=PU\Lambda^{1/2}\in\mathbb{R}^{m\times r}\) and it follows that \(A\) has full column rank. Then for any \(g\in\range(U) = \range(F)\), write \(g=Uz\) and \(g^{\prime} U z^{\prime}\) for some \(z, z^{\prime}\in\mathbb{R}^r\) and note \(Pg=PUz = A\Lambda^{-1/2}z\) and similarly, \(Pg^{\prime} = A\Lambda^{-1/2}z^{\prime}\), and \(PFP^{\top}=A A^{\top}\). For full-column-rank \(A\), \(A^{\top}(AA^{\top})^\dagger A = I_r\). Therefore
    \[
        (Pg)^{\top}(PFP^{\top})^{\dagger}(Pg^{\prime})
        = z^{\top} \Lambda^{-1/2} A^{\top}(AA^{\top})^{\dagger}A \Lambda^{-1/2}z^{\prime}
        = z^{\top} \Lambda^{-1} z^{\prime}
        = g^{\top} F^{\dagger}g^{\prime}.
    \]
    For the ``only if'' direction, suppose \(\rank(PU)<r\). Then there exists a nonzero \(z\in\mathbb{R}^r\) such that \(PUz=0\). Let \(g=Uz\in\range(F)\) be the corresponding vector. Then, as \(g^{\top} F^{\dagger} g = z^{\top} \Lambda^{-1} z > 0\), \((Pg)^{\top} (PFP^{\top})^{\dagger}(Pg) = 0 \neq g^{\top} F^{\dagger} g > 0\), proving the result.
\end{proof}
\section{Proofs for \texorpdfstring{\Cref{subsec:regularized-projection}}{Section 2.2} (Regularized Projection)}\label{adxsec:regulairzed-projection}
This section collects technical results used in \Cref{subsec:regularized-projection} that are omitted in the main text.

\subsection{Proof of Resolvent Perturbation Concentration for Regularized Projection}
We prove the key operator-norm perturbation step used in the proof of \Cref{thm:regularized-projection-upper-bound}.\footnote{This can be viewed as a special case of approximate matrix multiplication for sub-Gaussian sketches; see \citet[Theorem 1]{cohen2016optimal}. Here, we state and prove the special case for clarity.} The general idea is to use the concentration of the sample covariance (\Cref{lma:covariance-concentration}) to control the resolvent-type map \(A \mapsto A(A+\lambda I)^{-1}\) in operator norm, enabling the comparison of \(F(F+\lambda I)^{-1}\) and \(G(G+\lambda I)^{-1}\) in the proof of \Cref{thm:regularized-projection-upper-bound}.

To prove \Cref{lma:covariance-concentration}, the key input is a standard high-probability covariance estimation bound for sub-Gaussian vectors (\citet[Exercise 9.2.5]{vershynin2018high}), which we restate and prove as \Cref{prop:high-prob-cov-estimation}.

\begin{proposition}[High-Probability Covariance Estimation]\label{prop:high-prob-cov-estimation}
    Let \(\Sigma \succeq 0\) and let \(X, X_1,\dots,X_m \in \mathbb{R}^d\) be i.i.d.\ mean-zero sub-Gaussian random vectors with covariance \(\Sigma = \mathbb{E}[X X^{\top}]\). Define the sample covariance
    \[
        \Sigma_m
        \coloneqq \frac{1}{m}\sum_{i=1}^m X_{i} X_{i}^{\top} .
    \]
    Then for any \(u \geq 0\), with probability at least \(1-2e^{-u}\),
    \[
        \lVert \Sigma_m-\Sigma \rVert _2
        \leq C\left(\sqrt{\frac{r(\Sigma)+u}{m}}+\frac{r(\Sigma)+u}{m}\right) \lVert \Sigma \rVert _2,
    \]
    where \(r(\Sigma) \coloneqq \tr(\Sigma)/\lVert \Sigma \rVert _2\) is the stable rank of \(\Sigma^{1/2}\) and \(C>0\) is a universal constant.
\end{proposition}
\begin{proof}
    Write \(X = \Sigma^{1/2}Z\), where \(Z\) is an isotropic, mean-zero, sub-Gaussian random vector, and similarly \(X_{i}=\Sigma^{1/2}Z_{i}\) with i.i.d.\ copies \(Z_1,\dots,Z_m\). Let \(A\in\mathbb{R}^{m\times d}\) be the matrix whose \(i\)-th row is \(Z_{i}^{\top}\). As in the proof of \citet[Theorem 9.2.4]{vershynin2018high}, define \(T \coloneqq \Sigma^{1/2} S^{d-1}\) where \(S^{d-1}\) denotes the Euclidean unit sphere, then
    \[
        \lVert \Sigma_m-\Sigma \rVert _2
        = \frac{1}{m}\sup_{x\in T} \left\lvert \lVert Ax \rVert _2^2 - m\lVert x \rVert _2^2 \right\rvert .
    \]
    Consider the stochastic process
    \[
        Y_x
        \coloneqq \lVert Ax \rVert _2 - \sqrt{m} \lVert x \rVert _2, \qquad
        x\in T.
    \]
    By \citet[Theorem 9.1.3]{vershynin2018high}, \((Y_x)_{x\in T}\) has sub-Gaussian increments. Applying the high-probability Talagrand comparison inequality~\citep[Theorem 3.2]{dirksen2015tail}, we obtain that with probability at least \(1-2e^{-v^2}\),
    \[
        \sup_{x\in T} \lvert Y_x \rvert
        \leq  C\left(\gamma(T) + v \rad(T)\right),
    \]
    where \(\rad(T) \coloneqq \sup_{x\in T} \lVert x \rVert _2\) denotes the \emph{radius} of \(T\), and \(\gamma(T) \coloneqq \mathbb{E}\big[ \sup_{x\in T}\lvert \langle g, x \rangle \rvert \big] \) denotes the \emph{Gaussian complexity} of \(T\), for \(g \sim \mathcal{N}(0,I_d)\).

    Since \(T = \Sigma^{1/2} S^{d-1}\), we have \(\rad(T) = \lVert \Sigma \rVert _2^{1/2}\). Moreover,
    \[
        \gamma(T)
        = \mathbb{E}\big[ \lVert \Sigma^{1/2} g \rVert _2 \big]
        \leq \sqrt{\mathbb{E} [g^{\top}\Sigma g]}
        = \sqrt{\mathbb{E}[\tr(\Sigma g g^{\top})]}
        = \sqrt{\tr(\Sigma)},
    \]
    where the inequality follows from Jensen's inequality. Setting \(u = v^2\) and recalling that \(\tr(\Sigma) = r(\Sigma) \lVert \Sigma \rVert _2\), we conclude that, with probability at least \(1-2e^{-u}\),
    \[
        \sup_{x\in T} \lvert Y_x \rvert
        \leq C \lVert \Sigma \rVert _2^{1/2} (\sqrt{r(\Sigma)} + \sqrt{u}).
    \]

    Fix \(x\in T\) and write \(a\coloneqq \lVert Ax \rVert _2\) and \(b \coloneqq \sqrt{m} \lVert x \rVert _2\). Then \(b \leq \sqrt{m}\lVert \Sigma \rVert_2^{1/2}\) and
    \[
        \lvert a^2-b^2 \rvert
        \leq \lvert a-b \rvert (\lvert a-b \rvert + 2b).
    \]
    Using the bound above on \(\lvert a-b \rvert \) and the fact that \(b \geq 0\), we obtain
    \[
        \sup_{x\in T} \lvert a^2-b^2 \rvert
        \leq C\lVert \Sigma \rVert_2(\sqrt{r(\Sigma)}+\sqrt{u})
        \left( \sqrt{r(\Sigma)} + \sqrt{u} + \sqrt{m} \right).
    \]
    Dividing by \(m\) yields
    \[
        \lVert \Sigma_m-\Sigma \rVert _2
        \leq C\lVert \Sigma \rVert_2
        \left(\frac{r(\Sigma)+u}{m}+\sqrt{\frac{r(\Sigma)+u}{m}}\right),
    \]
    where we used \((\sqrt{r(\Sigma)}+\sqrt{u})^2\lesssim r(\Sigma)+u\). This completes the proof.
\end{proof}

We now prove the concentration of sample covariance formally.

\begin{lemma}\label{lma:covariance-concentration}
    Let \(P \in \mathbb{R}^{m \times d}\) be a sketching matrix whose rows are given by \(P_{i}^{\top} = \frac{1}{\sqrt{m}} W_{i}^{\top}\), where \(\{W_{i}\}_{i=1}^m\sim W\) are i.i.d.\ sub-Gaussian random vectors in \(\mathbb{R}^d\) satisfying \(\mathbb{E}[W]=0\) and \(\mathbb{E}[W W^{\top}]=I_d\). Let \(M \in \mathbb{R}^{d \times s}\) be a matrix and define \(\Sigma \coloneqq M^{\top} M\). For any \(\varepsilon, \delta \in (0,1)\), if
    \[
        m
        = \Omega \left(\frac{r(\Sigma) + \log(1/\delta)}{\varepsilon^2}\right),
    \]
    where \(r(\Sigma) = \tr(\Sigma) / \lVert \Sigma \rVert _2\) is the stable rank of \(\Sigma^{1/2}\), then with probability at least \(1-\delta\),
    \[
        \lVert M^{\top} (P^{\top} P - I_d) M \rVert _2
        \leq \varepsilon \lVert M \rVert _2^2 .
    \]
\end{lemma}
\begin{proof}
    The rows of \(P\) satisfy \(P_{i}^{\top} = \frac{1}{\sqrt{m}} X_{i}^{\top}\), where \(\{X_{i}\}_{i=1}^{m}\) are i.i.d.\ isotropic sub-Gaussian vectors. Observe that
    \[
        M^{\top} P^{\top} P M
        = M^{\top} \left(\frac{1}{m}\sum_{i=1}^{m} X_{i} X_{i}^{\top}\right) M
        = \frac{1}{m}\sum_{i=1}^m (M^{\top} X_{i})(M^{\top} X_{i})^{\top}
        \eqqcolon \Sigma_m .
    \]
    Define \(Y_{i} \coloneqq M^{\top} X_{i}\). Then \(\{Y_{i}\}_{i=1}^m\) are i.i.d.\ mean-zero sub-Gaussian vectors with covariance
    \[
        \mathbb{E}[Y Y^{\top}]
        = M^{\top} \mathbb{E}[X X^{\top}] M
        = M^{\top} M
        = \Sigma .
    \]
    Applying \citet[Exercise 9.2.5, \Cref{prop:high-prob-cov-estimation}]{vershynin2018high} yields that, with probability at least \(1-2e^{-u}\),
    \[
        \lVert \Sigma_m - \Sigma \rVert _2
        \leq C\left(\sqrt{\frac{r(\Sigma)+u}{m}} + \frac{r(\Sigma)+u}{m}\right)\lVert \Sigma \rVert_2,
    \]
    Choosing \(m \geq (r(\Sigma)+u)/\varepsilon^2\) ensures \(\sqrt{(r(\Sigma)+u)/m} \leq \varepsilon\) and \((r(\Sigma)+u)/m \leq \varepsilon^2 < \varepsilon\) for \(\varepsilon<1\). Since \(\lVert \Sigma \rVert_2 = \lVert M^{\top} M \rVert _2 = \lVert M \rVert _2^2\), we conclude that
    \[
        \lVert M^{\top} P^{\top} P M - M^{\top} M \rVert _2
        = \lVert M^{\top} (P^{\top} P - I_d) M \rVert _2
        \leq \varepsilon \lVert M \rVert _2^2 .
    \]
    Setting \(u = \Theta(\log(1/\delta))\) completes the proof.
\end{proof}

We can now state and prove the concentration of resolvent perturbation for regularized projection as follows:

\begin{lemma}\label{lma:concentration-covariance-operator-norm}
    Let \(F \succeq 0\) and \(\lambda > 0\), and define \(G = F^{1/2} P^{\top} P F^{1/2}\). Then for any \(\varepsilon, \delta \in (0, 1)\), if \(m = \Omega (\varepsilon^{-2}(d_{\lambda}(F) + \log(1/\delta) ))\), with probability at least \(1 - \delta\),
    \[
        \lVert F(F + \lambda I)^{-1} - G(G + \lambda I)^{-1} \rVert _2
        \leq \varepsilon.
    \]
\end{lemma}
\begin{proof}
    Applying \Cref{lma:covariance-concentration} with \(M = B = F^{1/2} (F + \lambda I)^{-1/2}\), for any \(\delta, \epsilon > 0\), if \(m = \Omega (\epsilon^{-2}(r(B^{\top} B) + \log(1/\delta)))\) then with probability at least \(1-\delta\),
    \[
        \lVert B^{\top} (P^{\top} P - I_d) B \rVert _2
        \leq \epsilon \lVert B \rVert_2^2 .
    \]
    We first note that if \(\lVert B \rVert _2^2 = 0\), then the bound is trivial. Assuming \(\lVert B \rVert _2 > 0\). Then we see that \(\lVert B \rVert _2^2 = \lVert B^{\top} B \rVert _2 = \lVert F(F + \lambda I)^{-1} \rVert _2 \leq 1\) since the eigenvalues of \(F(F + \lambda I)^{-1}\) equal \(\lambda_i (F) / (\lambda_i (F) + \lambda)\). Now, pick \(\epsilon \coloneqq \min(1, \varepsilon / 2\lVert B \rVert _2)\), and note that \(\lVert B \rVert_F^2 = \tr(F(F+\lambda I)^{-1}) = d_\lambda(F)\), we have
    \[
        r(B^{\top} B)
        = \frac{\tr(B^{\top} B)}{\lVert B^{\top} B \rVert _2}
        = \frac{\lVert B \rVert _F^2}{\lVert B \rVert _2^2}
        = \frac{d_{\lambda}(F)}{\lVert B \rVert _2^2}.
    \]
    After substitution, with \(\lVert B \rVert _2^2 \leq 1\), we conclude that if
    \[
        m
        = \Omega \left(\epsilon^{-2} \left(\frac{d_{\lambda}(F)}{\lVert B \rVert _2^2} + \log (1/\delta)\right)\right)
        = \Omega \left(\varepsilon^{-2} \left(d_{\lambda}(F) + \log (1/\delta) \right)\right),
    \]
    we have \(\lVert B^{\top} (P^{\top} P - I_d) B \rVert _2 \leq \epsilon \lVert B \rVert_2^2 \leq \varepsilon / 2\). This implies
    \[
        - \frac{\varepsilon}{2} I \preceq B^{\top} (P^{\top} P - I) B \preceq \frac{\varepsilon}{2} I
        \implies B^{\top} B - \frac{\varepsilon}{2} I \preceq B^{\top} P^{\top} P B \preceq B^{\top} B + \frac{\varepsilon}{2} I .
    \]
    With
    \[
        \begin{split}
            B^{\top} B            & = (F + \lambda I)^{-1/2} F (F + \lambda I)^{-1/2},                                             \\
            B^{\top} P^{\top} P B & = (F + \lambda I)^{-1 / 2} \underbrace{F^{1/2} P^{\top} P F^{1/2}}_{G} (F + \lambda I)^{-1/2},
        \end{split}
    \]
    we can conjugate by \((F+\lambda I)^{1/2}\), which yields
    \[
        \left(1-\frac{\varepsilon}{2}\right)F - \frac{\varepsilon}{2} \lambda I
        \preceq G
        \preceq \left(1+\frac{\varepsilon}{2}\right)F + \frac{\varepsilon}{2} \lambda I .
    \]
    Adding \(\lambda I\) gives
    \[
        \left(1-\frac{\varepsilon}{2}\right) (F + \lambda I)
        \preceq G + \lambda I
        \preceq \left(1+\frac{\varepsilon}{2}\right) (F + \lambda I).
    \]
    Define \(S \coloneqq (F + \lambda I)^{-1/2} (G + \lambda I) (F + \lambda I)^{-1/2}\). Conjugating the above by \((F + \lambda I)^{-1/2}\) yields
    \[
        \left(1 - \frac{\varepsilon}{2}\right) I
        \preceq S
        \preceq \left(1 + \frac{\varepsilon}{2}\right) I.
    \]
    Hence, \(S \succ 0\) and \(\lVert S - I \rVert_2 \leq \varepsilon / 2\) and \(\lVert S^{-1} \rVert _2 \leq \frac{1}{1 - \varepsilon / 2}\). From the definition of \(S\),
    \[
        (G + \lambda I)^{-1}
        = (F + \lambda I)^{-1/2} S^{-1} (F + \lambda I)^{-1/2},
    \]
    hence
    \[
        (G + \lambda I)^{-1} - (F + \lambda I)^{-1}
        = (F + \lambda I)^{-1/2} (S^{-1} - I)(F + \lambda I)^{-1/2},
    \]
    giving
    \[
        \lVert (G + \lambda I)^{-1} - (F + \lambda I)^{-1} \rVert _2
        \leq \lVert (F + \lambda I)^{-1} \rVert _2 \lVert S^{-1} - I \rVert _2.
    \]
    From the identity \(S^{-1} - I = S^{-1} (I - S)\), we have
    \[
        \lVert S^{-1} - I \rVert _2
        \leq \lVert S^{-1} \rVert _2 \lVert S - I \rVert _2
        \leq \frac{\varepsilon / 2}{1 - \varepsilon / 2}.
    \]
    With \(\lVert (F + \lambda I)^{-1} \rVert _2 \leq 1/\lambda\), we have
    \[
        \lVert (G + \lambda I)^{-1} - (F + \lambda I)^{-1} \rVert _2
        \leq \frac{1}{\lambda} \frac{\varepsilon / 2}{1 - \varepsilon / 2}.
    \]
    From the identity \(A(A + \lambda I)^{-1} = I - \lambda (A + \lambda I)^{-1}\) for any PSD \(A\), we have
    \[
        F(F + \lambda I)^{-1} - G(G + \lambda I)^{-1}
        = \lambda \left((G + \lambda I)^{-1} - (F + \lambda I)^{-1}\right),
    \]
    and hence
    \[
        \lVert F(F + \lambda I)^{-1} - G(G + \lambda I)^{-1} \rVert _2
        \leq \lambda \frac{1}{\lambda} \frac{\varepsilon / 2}{1 - \varepsilon / 2}
        = \frac{\varepsilon / 2}{1 - \varepsilon / 2}.
    \]
    Finally, note that \(\frac{\varepsilon / 2}{1 - \varepsilon / 2} \leq \varepsilon\) for any \(\varepsilon \in (0, 1)\), this proves the result.
\end{proof}

\subsection{OSE-Based Alternative Analysis}\label{adxsubsec:OSE}
We record a self-contained proof of the OSE-based alternative analysis sketched in discussion following \Cref{thm:regularized-projection-upper-bound}. Let \(A \in \mathbb{R}^{d \times r}\) be a fixed matrix. A random matrix \(P \in \mathbb{R}^{m \times d}\) is an \emph{oblivious subspace embedding (OSE)} for \(\range(A)\) with distortion \(\varepsilon\in(0,1)\) if, with high probability,
\[
    (1-\varepsilon)\lVert Ax\rVert_2^2
    \leq \lVert PAx\rVert_2^2
    \leq (1+\varepsilon)\lVert Ax\rVert_2^2,
    \qquad \forall x \in \mathbb{R}^r.
\]
Equivalently,
\[
    -\varepsilon A^{\top}A
    \preceq A^{\top}(P^{\top}P-I_d)A
    \preceq \varepsilon A^{\top}A.
\]
It is well known that standard oblivious sketches (Gaussian, Rademacher, SJLT) satisfy this property provided \(m = \Omega\big(\varepsilon^{-2}\rank(A)\big)\)~\citep[Theorems~2.3 and~6.10]{woodruff2014sketching}. In our case, we apply the OSE framework with $A = F^{1/2}$. 
It is straightforward to see that \(\rank(A)=\rank(F) = r\), so achieving an \(\varepsilon\)-OSE for \(\range(A)\) requires \(m = \Omega(r/\varepsilon^{2})\).

Define \(G \coloneqq F^{1/2}P^{\top}PF^{1/2}\). The OSE condition gives
\[
    (1-\varepsilon) F
    \preceq G
    \preceq (1+\varepsilon) F.
\]
Consider \(f(t) \coloneqq \frac{t}{t+\lambda}\) for \(t\geq 0\). Since \(t\mapsto (t+\lambda)^{-1}\) is operator monotone decreasing on \([0,\infty)\), it follows that \(f(t)=1-\lambda(t+\lambda)^{-1}\) is operator monotone increasing. 

Applying \(f\) to the sandwich gives
\[
    f\big((1-\varepsilon)F\big)
    \preceq f(G)
    \preceq f\big((1+\varepsilon)F\big),
\]
or
\[
    (1-\varepsilon)F\big((1-\varepsilon)F+\lambda I\big)^{-1}
    \preceq G(G+\lambda I)^{-1}
    \preceq (1+\varepsilon)F\big((1+\varepsilon)F+\lambda I\big)^{-1}.
\]
Since \(F\) commutes with any function of itself, the resulting operator-norm deviation reduces to a scalar supremum. For example,
\[
    \big\lVert f\big((1+\varepsilon)F\big) - f(F) \big\rVert_2
    =
    \sup_{t\geq 0}
    \left\lvert
    \frac{(1+\varepsilon)t}{(1+\varepsilon)t+\lambda}
    -
    \frac{t}{t+\lambda}
    \right\rvert
    =
    \sup_{t\geq 0}
    \frac{\varepsilon\lambda t}{\big((1+\varepsilon)t+\lambda\big)(t+\lambda)}.
\]
The same bound holds with \((1+\varepsilon)\) replaced by \((1-\varepsilon)\). A short calculus argument shows the supremum is at most \(\varepsilon\); hence
\[
    \big\lVert f(G) - f(F) \big\rVert_2
    = \big\lVert G(G+\lambda I)^{-1} - F(F+\lambda I)^{-1} \big\rVert_2
    \leq O(\varepsilon).
\]
Combining the above operator control with the argument in the proof of \Cref{thm:regularized-projection-upper-bound} yields the same bilinear and quadratic influence error bounds. The key difference is the sample complexity: the OSE route fundamentally scales with \(r\), whereas our main analysis scales with the effective dimension \(d_{\lambda}(F)\).

\subsection{Proof of Anti-Concentration of Gaussian Sample Covariance}
Next, we prove the worst-case lower bound (\Cref{thm:regularized-projection-lower-bound}). The proof consists of two main components:
\begin{enumerate}
    \item An anti-concentration result for the sample covariance of Gaussian matrices, which shows that deviations of order \(\sqrt{k/m}\) occur with constant probability (\Cref{lma:anti-concentration}).
    \item A carefully constructed hard instance \(F\) for which such deviations translate directly into a large error in the regularized quadratic form.
\end{enumerate}

We note that since the proof of \Cref{lma:anti-concentration} contains many technical computation, we defer them for a cleaner presentation after the main proof.

\begin{lemma}\label{lma:anti-concentration}
    Let \(W\in\mathbb{R}^{m\times k}\) have rows \(w_1,\dots,w_m \sim \mathcal{N}(0,I_k)\) i.i.d., and define \(S \coloneqq \frac{1}{m}W^{\top} W\). Then for all \(m,k\geq 1\),
    \[
        \Pr\left( \lVert S-I_k\rVert _2 \geq \frac{1}{2}\sqrt{\frac{k}{m}} \right)
        \geq \frac{3}{80}.
    \]
\end{lemma}
\begin{proof}
    Define
    \[
        A \coloneqq S-I_k
        = \frac{1}{m}\sum_{i=1}^m X_i,
        \qquad
        X_i \coloneqq w_i w_i^{\top} - I_k.
    \]
    Then \(\mathbb{E}[X_i]=0\) and \(X_1,\dots,X_m\) are independent. Let \(g \coloneqq \lVert A\rVert_F^2 \geq 0\). Expanding, we have
    \[
        g
        = \left\lVert \frac{1}{m}\sum_{i=1}^m X_i \right\rVert_F^2
        = \frac{1}{m^2}\sum_{i,j=1}^m \langle X_i,X_j\rangle.
    \]
    Since \(\mathbb{E}[\langle X_i,X_j\rangle] = 0\) for \(i\neq j\) from independence and \(\mathbb{E}[X_i]=0\),
    \[
        \mathbb{E}[g]
        = \frac{1}{m^2}\sum_{i=1}^m \mathbb{E}[\lVert X_i\rVert_F^2]
        = \frac{1}{m} \mathbb{E}[\lVert X_1\rVert _F^2].
    \]
    A direct computation gives \(\mathbb{E}[\lVert X_1\rVert_F^2] = k(k+1)\), hence
    \[
        \mathbb{E}[g]
        = \frac{k(k+1)}{m}.
    \]
    On the other hand, as \(g = \frac{1}{m^2}\sum_{i,j} Y_{ij}\) where \(Y_{ij}\coloneqq \langle X_i,X_j\rangle\), we have
    \[
        \mathbb{E}[g^2]
        = \frac{1}{m^4}\sum_{i,j,p,q}\mathbb{E}[Y_{ij}Y_{pq}].
    \]
    By independence and centering, only overlapping index patterns contribute, and one obtains
    \[
        \mathbb{E}[g^2]
        = \frac{1}{m^4}\left( m a + m(m-1)\mu^2 + 2m(m-1)b \right),
    \]
    where \(\mu \coloneqq \mathbb{E}[\lVert X_1\rVert_F^2]\), \(a \coloneqq \mathbb{E}[\lVert X_1\rVert _F^4]\), and \(b \coloneqq \mathbb{E}[\langle X,Y\rangle^2]\), and \(X\coloneqq w w^{\top} - I_k\), \(Y\coloneqq u u^{\top}-I_k\) with \(w \perp u\) i.i.d.\ \(\mathcal{N}(0,I_k)\). Moreover, a moment calculations yield
    \[
        \mu = k(k+1),\qquad
        a = k^4 + 10k^3 + 25k^2 + 24k,\qquad
        b = 2k^2 + 2k.
    \]
    Substituting these expressions into the above formula for \(\mathbb{E}[g^2]\) and simplifying gives the explicit comparison
    \[
        \frac{(\mathbb{E}[g])^2}{\mathbb{E}[g^2]}
        \geq \frac{1}{15},
    \]
    uniformly for all \(m,k\geq 1\). Equivalently, \(\mathbb{E}[g^2] \leq 15(\mathbb{E}[g])^2\). Then, by Paley--Zygmund~\citep{paley1932some}, for any \(\theta\in(0,1)\),
    \[
        \Pr(g \geq \theta \mathbb{E}[g])
        \geq (1-\theta)^2 \frac{(\mathbb{E}[g])^2}{\mathbb{E}[g^2]}
        \geq \frac{(1-\theta)^2}{15}.
    \]
    Taking \(\theta=1 / 4\) yields
    \[
        \Pr\left(g \geq \frac{1}{4}\mathbb{E}[g]\right)
        \geq \frac{(3/4)^2}{15}
        = \frac{3}{80}.
    \]
    On this event,
    \[
        g
        = \lVert A\rVert_F^2
        \geq \frac{1}{4}\cdot \frac{k(k+1)}{m}.
    \]
    Using \(\lVert A\rVert _F^2 \leq k \lVert A\rVert _2^2\), we obtain
    \[
        \lVert A\rVert _2^2
        \geq \frac{1}{k}\lVert A\rVert _F^2
        \geq \frac{1}{k}\cdot \frac{1}{4}\cdot \frac{k(k+1)}{m}
        = \frac{k+1}{4m}
        \geq \frac{k}{4m}.
    \]
    Hence, with probability at least \(3/80\),
    \[
        \lVert A\rVert_2
        = \lVert S-I_k \rVert _2
        \geq \frac{1}{2}\sqrt{\frac{k}{m}}.
    \]
\end{proof}

We now provide the routine calculations used in the proof of \Cref{lma:anti-concentration}. In particular, we compute moments of Gaussian rank-one matrices and enumerate the index patterns in \(\mathbb{E}[\lVert A \rVert_F^4]\).

\paragraph{Chi-square moments.}
Let \(r\sim\chi_k^2\). For any \(n\in\mathbb{N}_+\),
\[
    \mathbb{E}[r^n] = \prod_{i=0}^{n-1} (k+2i).
\]
In particular,
\[
    \mathbb{E}[r] = k,\quad
    \mathbb{E}[r^2] = k(k+2),\quad
    \mathbb{E}[r^3] = k(k+2)(k+4),\quad
    \mathbb{E}[r^4] = k(k+2)(k+4)(k+6).
\]

\paragraph{Moments of \(X = ww^{\top} - I_k\).}
Let \(w\sim\mathcal{N}(0,I_k)\) and define \(X \coloneqq ww^{\top} - I_k\). Write \(r\coloneqq \lVert w \rVert_2^2\sim\chi_k^2\). We compute \(\mu\coloneqq\mathbb{E}[\lVert X \rVert_F^2]\) and \(a\coloneqq\mathbb{E}[\lVert X \rVert_F^4]\). First,
\begin{align*}
    \lVert X \rVert_F^2
     & = \tr(X^{\top} X) = \tr(X^2)
    = \tr\big((ww^{\top} - I_k)^2\big)                           \\
     & = \tr(ww^{\top} ww^{\top}) - 2 \tr(ww^{\top}) + \tr(I_k).
\end{align*}
By trace cyclicity, \(\tr(ww^{\top} ww^{\top}) = \tr(w(w^{\top} w)w^{\top}) = (w^{\top} w) \tr(ww^{\top}) = r^2\), while \(\tr(ww^{\top})=r\) and \(\tr(I_k)=k\). Hence
\[
    \lVert X \rVert_F^2 = r^2 - 2r + k.
\]
Taking expectation and using the moments above gives
\[
    \mu = \mathbb{E}[r^2 - 2r + k] = k(k+2) - 2k + k = k(k+1).
\]
Moreover,
\[
    a
    = \mathbb{E}(r^2 - 2r + k)^2
    = \mathbb{E}\big[r^4 - 4r^3 + (4+2k)r^2 - 4kr + k^2\big].
\]
Substituting \(\mathbb{E}[r],\dots,\mathbb{E}[r^4]\) yields
\[
    a = k^4 + 10k^3 + 25k^2 + 24k.
\]

\paragraph{The mixed term \(b = \mathbb{E}[\langle X,Y\rangle^2]\).}
Let \(w,u\sim\mathcal{N}(0,I_k)\) be independent, and define \(X \coloneqq ww^{\top} - I_k\) and \(Y \coloneqq uu^{\top} - I_k\). Set \(r\coloneqq\lVert w \rVert_2^2\), \(s\coloneqq\lVert u \rVert_2^2\), and \(t\coloneqq w^{\top} u\). A direct expansion gives
\[
    \langle X,Y\rangle
    = \tr\big((ww^{\top} - I_k)(uu^{\top} - I_k)\big)
    = t^2 - r - s + k,
\]
since \(\tr(ww^{\top} uu^{\top}) = \tr(w(w^{\top} u)u^{\top}) = (w^{\top} u)^2 = t^2\).
Therefore,
\[
    b = \mathbb{E}[(t^2-r-s+k)^2]
    = \mathbb{E}[t^4] + \mathbb{E}[(r+s-k)^2] - 2 \mathbb{E}\big[t^2(r+s-k)\big].
\]
To evaluate these terms, write \(t=\sum_{\ell=1}^k Z_{\ell}\) with \(Z_{\ell}\coloneqq w_{\ell} u_{\ell}\). Then \(\mathbb{E}[Z_{\ell}]=0\), \(\mathbb{E}[Z_{\ell}^2]=1\), and \(\mathbb{E}[Z_{\ell}^4]=9\), and hence
\[
    \mathbb{E}[t^4]
    = \sum_{\ell=1}^k \mathbb{E}[Z_{\ell}^4] + 6\sum_{1\leq i<j\leq k}\mathbb{E}[Z_i^2] \mathbb{E}[Z_j^2]
    = 9k + 6\binom{k}{2}
    = 3k^2 + 6k.
\]
Next, since \(r,s\sim\chi_k^2\) are independent, we have \(\mathbb{E}[r]=\mathbb{E}[s]=k\) and \(\Var[r]=\Var[s]=2k\), so
\[
    \mathbb{E}[(r+s-k)^2]
    = \Var[r+s-k] + (\mathbb{E}[r+s-k])^2
    = 4k + k^2.
\]
Finally, conditioning on \(w\) gives \(t \mid w \sim \mathcal{N}(0,\lVert w \rVert_2^2)=\mathcal{N}(0,r)\), so \(\mathbb{E}[t^2\mid w]=r\) and hence \(\mathbb{E}[t^2]=\mathbb{E}[r]=k\). Moreover,
\[
    \mathbb{E}[t^2 r]
    = \mathbb{E}\big[r \mathbb{E}[t^2\mid w]\big]
    = \mathbb{E}[r^2]
    = k(k+2).
\]
By symmetry, \(\mathbb{E}[t^2(r+s-k)] = 2\mathbb{E}[t^2 r] - k\mathbb{E}[t^2] = k^2 + 4k\), so altogether
\[
    b = (3k^2+6k) + (k^2+4k) - 2(k^2+4k) = 2k^2 + 2k.
\]

\paragraph{Enumerating index patterns in \(\mathbb{E}\lVert A \rVert_F^4\).}
Let \(A=\frac{1}{m}\sum_{i=1}^m X_i\) with \(X_i= w_i w_i^{\top} - I_k\) i.i.d.\ and mean-zero, and set \(Z=\lVert A \rVert_F^2\). With \(Y_{ij}\coloneqq \langle X_i,X_j\rangle\), we have
\[
    Z = \frac{1}{m^2}\sum_{i,j=1}^m Y_{ij},
    \qquad
    Z^2 = \frac{1}{m^4}\sum_{i,j,p,q=1}^m Y_{ij}Y_{pq},
    \qquad
    \mathbb{E}[Z^2] = \frac{1}{m^4}\sum_{i,j,p,q}\mathbb{E}[Y_{ij}Y_{pq}].
\]
The expectation \(\mathbb{E}[Y_{ij}Y_{pq}]\) is zero unless \(\{i,j\}\cap\{p,q\}\neq \varnothing\). Indeed, if \(\{i,j\}\cap\{p,q\}=\varnothing\), then the two factors depend on disjoint sets of independent random variables. Moreover, for \(i\neq j\), \(\mathbb{E}[Y_{ij}]=\mathbb{E}[\langle X_i,X_j\rangle]=\langle \mathbb{E}[X_i],\mathbb{E}[X_j]\rangle=0\), so such disjoint products vanish. The only contributing configurations are:
\begin{enumerate}[label=(T\arabic*),ref=(T\arabic*)]
    \item\label{pat:T1} \((i,j)=(p,q)\), contributing \(\mathbb{E}[Y_{ij}^2]\);
    \item\label{pat:T2} \((i,j)=(q,p)\), contributing \(\mathbb{E}[Y_{ij}Y_{ji}]=\mathbb{E}[Y_{ij}^2]\) since \(Y_{ij}=Y_{ji}\);
    \item\label{pat:T3} \(i=j\) and \(p=q\) with \(i\neq p\), contributing \(\mathbb{E}[Y_{ii}]\mathbb{E}[Y_{pp}]=\mu^2\).
\end{enumerate}
Counting multiplicities, type \labelcref{pat:T1} gives \(\sum_{i,j}\mathbb{E}[Y_{ij}^2] = m a + m(m-1)b\), where \(a=\mathbb{E}[\lVert X_1 \rVert_F^4]\) (since \(Y_{11}=\langle X_1,X_1\rangle=\lVert X_1 \rVert_F^2\)) and \(b=\mathbb{E}[\langle X,Y\rangle^2]\) for independent copies \(X,Y\). Type \labelcref{pat:T2} contributes another \(m(m-1)b\), and type \labelcref{pat:T3} contributes \(m(m-1)\mu^2\). Hence
\[
    \mathbb{E}[Z^2]
    = \frac{1}{m^4}\Big( m a + m(m-1)\mu^2 + 2m(m-1)b \Big).
\]
Using \(\mu=k(k+1)\), \(a=k^4+10k^3+25k^2+24k\), and \(b=2k^2+2k\), one checks that for all \(m,k\geq 1\),
\[
    \mathbb{E}[Z^2] \leq \frac{15 k^2(k+1)^2}{m^2}.
\]

\subsection{Proof of Worst-Case Lower Bound}
We restate \Cref{thm:regularized-projection-lower-bound} below for convenience:

\begin{theorem*}
    Let \( P \in \mathbb{R}^{m \times d} \) be a Gaussian oblivious sketch with rows i.i.d.\ \(\mathcal{N}(0,I_d)\). There exists a family of matrices \( F \in \mathbb{R}^{d \times d} \) such that if \( m = o(d_\lambda(F)/\varepsilon^2) \), then with constant probability, there exists \(g \in \range(F)\) such that
    \[
        \lvert \widetilde{\tau}_\lambda(g,g) - \tau_\lambda(g,g) \rvert
        = \Omega(\varepsilon) \tau_0(g,g).
    \]
\end{theorem*}
\begin{proof}
    Fix integers \(k \leq r = \rank(F) \leq d\) and define
    \[
        F = \diag(\underbrace{\lambda,\dots,\lambda}_{k},
        \underbrace{\eta\lambda,\dots,\eta\lambda}_{r-k},
        \underbrace{0,\dots,0}_{d-r}),
    \]
    where \(\eta>0\) will be chosen sufficiently small (as a function of \(\varepsilon\) and fixed constants only). Then
    \[
        d_\lambda(F)
        = \sum_{i=1}^d \frac{\lambda_i(F)}{\lambda_i(F)+\lambda}
        = \frac{k\lambda}{\lambda+\lambda} + \frac{(r-k)\eta\lambda}{\eta\lambda+\lambda}
        = \frac{k}{2} + \frac{\eta}{1+\eta}(r-k)
        = \Theta(k)
        \quad \text{for } \eta \ll 1.
    \]

    Let \(P = \frac{1}{\sqrt{m}}W\) where \(W\in\mathbb{R}^{m\times d}\) has i.i.d.\ \(\mathcal{N}(0,1)\) entries, and partition
    \[
        P=(P_L,P_S,P_Z)
    \]
    according to the blocks of \(F\), i.e.\ \(P_L\in\mathbb{R}^{m\times k}\), \(P_S\in\mathbb{R}^{m\times (r-k)}\). Choose
    \[
        g = F^{1/2}y,
        \qquad
        y=\begin{pmatrix}
            y_L \\
            0   \\
            0   \\
        \end{pmatrix},
        \quad \lVert y_L \rVert_2=1,
    \]
    for some \(y\). We see that \(g_L = \sqrt{\lambda} y_L\). Now, we see that
    \[
        \tau_\lambda(g,g)
        = g^{\top}(F+\lambda I)^{-1}g
        = y^{\top} F^{1/2}(F+\lambda I)^{-1}F^{1/2}y
        = y_L^{\top} \frac{\lambda}{\lambda+\lambda} I_k  y_L
        = \frac{1}{2},
    \]
    and
    \[
        \tau_0(g,g)
        = g^{\top} F^\dagger g
        = y^{\top} y
        = \lVert y_L \rVert_2^2
        = 1.
    \]
    On the other hand, the sketched quantity equals
    \[
        \widetilde{\tau}_\lambda(g,g)
        = g^{\top} P^{\top}(PFP^{\top} + \lambda I)^{-1}Pg.
    \]
    Since \(g=F^{1/2}y\) and \(F=\lambda\diag(I_k,\eta I_{r-k},0)\), we have \(Pg = \sqrt{\lambda} P_L y_L\) and
    \[
        PFP^{\top} + \lambda I
        = \lambda\big(P_LP_L^{\top} + \eta P_SP_S^{\top} + I\big).
    \]
    Therefore
    \[
        \widetilde{\tau}_\lambda(g,g)
        = y_L^{\top} P_L^{\top}\big(P_LP_L^{\top} + \eta P_SP_S^{\top} + I\big)^{-1}P_L y_L.
    \]

    Decomposing the error, write
    \[
        \big\lvert \widetilde{\tau}_\lambda(g,g)-\tau_\lambda(g,g)\big\rvert \geq T_1 - T_2,
    \]
    where
    \[
        \begin{split}
            T_1 & \coloneqq \left\lvert y_L^{\top} P_L^{\top} (P_LP_L^{\top}+I)^{-1} P_L y_L - \frac{1}{2} \right\rvert,                                            \\
            T_2 & \coloneqq \left\lvert y_L^{\top} P_L^{\top} \left[(P_LP_L^{\top}+I)^{-1} - (P_LP_L^{\top}+\eta P_SP_S^{\top}+I)^{-1}\right] P_L y_L \right\rvert.
        \end{split}
    \]

    \paragraph{Lower-Bounding \(T_1\).}
    Let \(M\coloneqq P_L^{\top} P_L\in\mathbb{R}^{k\times k}\). Using the push-through identity
    \[
        P_L^{\top}(P_LP_L^{\top} + I)^{-1}P_L
        = M(M+I)^{-1},
    \]
    we have
    \[
        \left\lvert y_L^{\top} P_L^{\top} (P_LP_L^{\top}+I)^{-1} P_L y_L - \frac{1}{2} \right\rvert
        = \left\lvert y_L^{\top} M(M+I)^{-1}y_L - \frac{1}{2} \right\rvert.
    \]
    Let \(\lambda_1,\dots,\lambda_k\) be the eigenvalues of \(M\) and choose \(y_L\) to be a unit eigenvector corresponding to an eigenvalue \(\lambda_\star\). Then
    \[
        y_L^{\top} M(M+I)^{-1}y_L
        = \frac{\lambda_\star}{\lambda_\star+1},
    \]
    and hence
    \[
        \left\lvert y_L^{\top} M(M+I)^{-1}y_L - \frac{1}{2} \right\rvert
        = \left\lvert \frac{\lambda_\star}{\lambda_\star+1}-\frac{1}{2} \right\rvert
        = \left\lvert \frac{\lambda_\star-1}{2(\lambda_\star + 1)}\right\rvert
        = \frac{\lvert \lambda_\star-1 \rvert}{2(\lambda_\star+1)}.
    \]
    Using \(\lambda_\star+1 \leq \lvert \lambda_\star-1\rvert + 2\), we obtain
    \[
        \left\lvert y_L^{\top} M(M+I)^{-1}y_L - \frac{1}{2} \right\rvert
        \geq \frac{\lvert \lambda_\star-1 \rvert}{2 \lvert \lambda_\star-1 \rvert + 4}
        \geq \min \left\{\frac{\lvert \lambda_\star-1 \rvert }{8}, \frac14\right\}.
    \]
    Now observe that \(\lVert M-I \rVert_2 = \max_i \lvert \lambda_i-1 \rvert\), so if \(\lVert M-I \rVert_2\geq t\), then there exists \(\lambda_\star\) with \(\lvert \lambda_\star-1 \rvert \geq t\) and the above choice of \(y_L\) yields
    \[
        \left\lvert y_L^{\top} M(M+I)^{-1}y_L - \frac{1}{2} \right\rvert
        \geq \min \left\{\frac{t}{8}, \frac{1}{4}\right\}.
    \]
    Since \(P_L=\frac{1}{\sqrt{m}}W_L\) with \(W_L\in\mathbb{R}^{m\times k}\) i.i.d.\ Gaussian rows, we have
    \[
        M
        = P_L^{\top} P_L
        = \frac{1}{m}W_L^{\top} W_L.
    \]
    Applying \Cref{lma:anti-concentration} to \(W_L\) gives
    \[
        \Pr\left(\lVert M-I_k \rVert_2 \geq \frac{1}{2} \sqrt{\frac{k}{m}}\right)
        \geq \frac{3}{80}.
    \]
    On this event we may take \(t=\frac{1}{2}\sqrt{k/m}\) above, giving
    \[
        \left\lvert y_L^{\top} M(M+I)^{-1}y_L - \frac{1}{2} \right\rvert
        \geq \min \left\{\frac{1}{16}\sqrt{\frac{k}{m}}, \frac{1}{4}\right\},
    \]
    with probability at least \(3 / 80\).

    \paragraph{Upper-Bounding \(T_2\).}
    Let \(A \coloneqq P_LP_L^{\top}+I \succ 0\) and \(B \coloneqq A+\eta P_SP_S^{\top}\succ 0\). By Woodbury matrix identity,
    \[
        B^{-1}
        = (A+\eta P_SP_S^{\top})^{-1}
        = A^{-1}-A^{-1}P_S (\eta^{-1}I+P_S^{\top}A^{-1}P_S )^{-1}P_S^{\top}A^{-1}.
    \]
    Hence
    \[
        A^{-1}-B^{-1}
        = A^{-1}P_S (\eta^{-1}I+P_S^{\top}A^{-1}P_S )^{-1}P_S^{\top}A^{-1}.
    \]
    Using \(\lvert y_L^{\top} (\cdot) y_L \rvert \leq \lVert \cdot\rVert _2\), we obtain
    \[
        T_2
        \leq \left\lVert P_L^{\top}(A^{-1}-B^{-1})P_L\right\rVert _2
        = \left\lVert P_L^{\top}A^{-1}P_S\left(\eta^{-1}I+P_S^{\top}A^{-1}P_S\right)^{-1}P_S^{\top}A^{-1}P_L\right\rVert _2.
    \]
    Define
    \[
        X \coloneqq A^{-1/2}P_L,\qquad
        C \coloneqq A^{-1/2}P_S.
    \]
    Then \(P_L^{\top}A^{-1}P_S = X^{\top}C\) and \(P_S^{\top}A^{-1}P_S = C^{\top}C\), so
    \[
        T_2
        \leq \left\lVert X^{\top}C\left(\eta^{-1}I+C^{\top}C \right)^{-1}C^{\top}X\right\rVert _2
        \leq \lVert X\rVert _2^2 \cdot \left\lVert C\left(\eta^{-1}I+C^{\top}C\right)^{-1}C^{\top}\right\rVert _2.
    \]
    We claim \(\lVert X\rVert _2\leq 1\). Indeed,
    \[
        X^{\top}X
        = P_L^{\top}A^{-1}P_L
        = P_L^{\top}(P_LP_L^{\top}+I)^{-1}P_L
        = M(M+I)^{-1}
        \preceq I,
    \]
    where \(M=P_L^{\top}P_L\succeq 0\), and the last inequality holds since the eigenvalues of \(M(M+I)^{-1}\) are \(\lambda/(\lambda+1)\in[0,1)\). Therefore \(\lVert X\rVert _2^2\leq 1\), and hence
    \[
        T_2
        \leq \left\lVert C\left(\eta^{-1}I+C^{\top}C\right)^{-1}C^{\top}\right\rVert _2.
    \]

    Next, diagonalize \(C^{\top}C\) and let \(\sigma_{\max}^2 = \lVert C\rVert _2^2\) be its largest eigenvalue. The nonzero eigenvalues of \(C(\eta^{-1}I+C^{\top}C)^{-1}C^{\top}\) are
    \[
        \frac{\sigma_i^2}{\eta^{-1}+\sigma_i^2}
        = \frac{\eta\sigma_i^2}{1+\eta\sigma_i^2},
    \]
    so
    \[
        \left\lVert C\left(\eta^{-1}I+C^{\top}C\right)^{-1}C^{\top}\right\rVert _2
        = \frac{\eta\lVert C\rVert _2^2}{1+\eta\lVert C\rVert _2^2}
        \leq \eta\lVert C\rVert _2^2.
    \]
    Finally, since \(A\succeq I\), we have \(\|C\|_2=\|A^{-1/2}P_S\|_2\leq \|P_S\|_2\), and thus
    \[
        T_2 \leq \frac{\eta\|P_S\|_2^2}{1+\eta\|P_S\|_2^2}
        \leq \eta\|P_S\|_2^2.
    \]

    It remains to control \(\|P_S\|_2\). Since \(P_S=\frac{1}{\sqrt{m}}W_S\) is Gaussian, standard spectral norm bounds imply that for any \(\delta\in(0,1)\), with probability at least \(1-\delta\),
    \[
        \|P_S\|_2 \leq 1+\sqrt{\frac{r-k}{m}}+\sqrt{\frac{\log(2/\delta)}{m}}.
    \]
    In particular, if \(r-k\leq m\) and \(m\geq \log(2/\delta)\), then on this event \(\lVert P_S\rVert _2\leq 3\) and hence
    \[
        T_2 \leq 9\eta.
    \]

    \paragraph{Choosing Parameters.}
    Fix \(\delta\coloneqq \frac{1}{160}\) and assume \(r-k\leq m\) and \(m\geq \log(2/\delta)\). This is possible by choosing \(r\) appropriately, e.g., \(r=k+m\) or \(r=2k\) when \(m\leq k\). Then the above bound on \(T_2\) holds with probability at least \(1-\delta\). By \Cref{lma:anti-concentration}, the lower bound on \(T_1\) holds with probability at least \(3/80\). By the union bound, both events hold simultaneously with probability at least \(3/80 - 1/160 = 1/32\). On this intersection event, using the bound from the \(T_1\) part,
    \[
        \lvert \widetilde{\tau}_\lambda(g,g)-\tau_\lambda(g,g)\rvert
        \geq T_1-T_2
        \geq \min\left\{\frac{1}{16}\sqrt{\frac{k}{m}},\frac{1}{4}\right\}-9\eta.
    \]
    We work in the nontrivial regime \(\sqrt{k/m}\leq 4\), so the minimum equals \(\frac{1}{16}\sqrt{k/m}\).
    Now choose
    \[
        \eta \coloneqq \frac{\varepsilon}{288}.
    \]
    If \(m \leq k / \varepsilon^2\), then \(\sqrt{k/m}\geq \varepsilon\), and hence
    \[
        \min\left\{\frac{1}{16}\sqrt{\frac{k}{m}},\frac{1}{4}\right\}-9\eta
        \geq \frac{1}{16}\varepsilon-\frac{9}{288}\varepsilon
        = \frac{1}{32}\varepsilon.
    \]
    Therefore, with probability at least \(1/32\),
    \[
        \lvert \widetilde{\tau}_\lambda(g,g)-\tau_\lambda(g,g)\rvert
        \geq \frac{1}{32}\varepsilon.
    \]
    Recalling that \(\tau_0(g,g)=1\) and that \(d_\lambda(F)=\Theta(k)\) for \(\eta\ll 1\), this shows that whenever \(m = o(d_\lambda(F)/\varepsilon^2)\), with constant probability there exists \(g\in\range(F)\) such that
    \[
        \lvert \widetilde{\tau}_\lambda(g,g)-\tau_\lambda(g,g)\rvert
        = \Omega(\varepsilon) \tau_0(g,g),
    \]
    as claimed.
\end{proof}
\section{Proofs for \texorpdfstring{\Cref{subsec:factorized-influence}}{Section 2.3} (Factorized Influence)}\label{adxsec:factorized-influence}

\subsection{Proof of the Barrier of Unregularized Factorized Influence}\label{adxsubsec:unregularized-factorized-projection}
We first record the factorized counterpart of the sharp barrier for exact preservation (\Cref{thm:unregularized-projection}) discussion in \Cref{subsec:factorized-influence}, i.e., \Cref{thm:unregularized-factorized-projection}. While \Cref{thm:unregularized-projection} characterizes exact invariance for general sketches, the factorized sketch \(P=P_A\otimes P_E\) admits a more explicit, factor-level injectivity condition. We restate \Cref{thm:unregularized-factorized-projection} and prove it below:

\begin{theorem*}
    Let \(A\succeq 0\in\mathbb{R}^{d_A\times d_A}\), \(E\succeq 0\in\mathbb{R}^{d_E\times d_E}\), and \(F\coloneqq A\otimes E\succeq 0\in\mathbb{R}^{(d_A d_E)\times(d_A d_E)}\). Let \(r_A\coloneqq\rank(A)\), \(r_E\coloneqq\rank(E)\), and \(r\coloneqq\rank(F)=r_A r_E\). Fix \(P_A\in\mathbb{R}^{m_A\times d_A}\), \(P_E\in\mathbb{R}^{m_E\times d_E}\), and define \(P\coloneqq P_A\otimes P_E\in\mathbb{R}^{(m_A m_E)\times(d_A d_E)}\). Then the following are equivalent:
    \begin{enumerate}[label=(\roman*)]
        \item\label{thm:unregularized-factorized-projection-i} For all \(g,g^{\prime}\in\range(F)\), we have \(\widetilde{\tau}_0(g,g^{\prime})=\tau_0(g,g^{\prime})\).
        \item\label{thm:unregularized-factorized-projection-ii} \(P\) is injective on \(\range(F)\), i.e., \(\rank(PU)=r\) for \emph{any} orthonormal basis \(U\in\mathbb{R}^{(d_A d_E)\times r}\) of \(\range(F)\).
        \item\label{thm:unregularized-factorized-projection-iii} \(P_A\) is injective on \(\range(A)\) and \(P_E\) is injective on \(\range(E)\). Equivalently, for orthonormal bases \(U_A\in\mathbb{R}^{d_A\times r_A}\) of \(\range(A)\) and \(U_E\in\mathbb{R}^{d_E\times r_E}\) of \(\range(E)\), we have \(\rank(P_A U_A)=r_A\) and \(\rank(P_E U_E)=r_E\).
    \end{enumerate}
    In particular, \(m_A\geq r_A\) and \(m_E\geq r_E\) are necessary, hence \(m=m_A m_E\geq r_A r_E=r\).
\end{theorem*}
\begin{proof}
    The equivalence between \labelcref{thm:unregularized-factorized-projection-i} and \labelcref{thm:unregularized-factorized-projection-ii} is exactly \Cref{thm:unregularized-projection}. It remains to relate \labelcref{thm:unregularized-factorized-projection-ii} and \labelcref{thm:unregularized-factorized-projection-iii} in the factorized setting. Let \(U_A\) and \(U_E\) be orthonormal bases of \(\range(A)\) and \(\range(E)\), respectively. Then \(U\coloneqq U_A\otimes U_E\) is an orthonormal basis of \(\range(F)\). Using the mixed-product identity,
    \[
        PU
        = (P_A\otimes P_E)(U_A\otimes U_E)
        = (P_A U_A)\otimes (P_E U_E).
    \]
    Moreover, \(\rank(X\otimes Y)=\rank(X)\rank(Y)\) for any matrices \(X,Y\). Therefore,
    \[
        \rank(PU)
        = \rank(P_A U_A) \rank(P_E U_E).
    \]
    Since \(\rank(P_A U_A)\leq r_A\) and \(\rank(P_E U_E)\leq r_E\), we have \(\rank(PU)=r_A r_E\) if and only if \(\rank(P_A U_A)=r_A\) and \(\rank(P_E U_E)=r_E\), which is equivalent to injectivity of \(P_A\) on \(\range(A)\) and \(P_E\) on \(\range(E)\).

    The dimensional necessity \(m_A\geq r_A\), \(m_E\geq r_E\) follows immediately from \(\rank(P_A U_A)\leq \min\{m_A,r_A\}\) and \(\rank(P_E U_E)\leq \min\{m_E,r_E\}\).
\end{proof}

\subsection{Proof of Factorized Resolvent Perturbation Concentration for Regularized Projection}
This section proves the key technical lemma used in the factorized influence analysis in the main text (\Cref{thm:regularized-factorized-projection-upper-bound}). The main technical challenges relative to the i.i.d.\ sketching setting are that, for a Kronecker sketch \(P=P_A\otimes P_E\), the matrix \(P^{\top}P\) decomposes into a sum of Kronecker-structured error terms rather than a single sample covariance, and \(P\) no longer satisfies the i.i.d.\ assumptions.

In the following, we prove the factorized version of \Cref{lma:covariance-concentration}:

\begin{theorem}[Factorized covariance deviation for K-FAC]\label{thm:factorized-covariance}
    Let \(F=A\otimes E \succeq 0\) and \(P=P_A\otimes P_E\) be as above, and fix \(\varepsilon,\delta\in(0,1)\). Assuming \(\lambda \leq \lVert A \rVert_2\lVert E \rVert_2\), and define the rescaled regularization levels \(\lambda_E \coloneqq \lambda / \lVert E \rVert _2\) and \(\lambda _A \coloneqq \lambda / \lVert A \rVert _2\). If
    \[
        m_A
        = \Omega \left(\frac{d_{\lambda_E}(A)+\log(1/\delta)}{\varepsilon^2}\right), \qquad
        m_E
        = \Omega \left(\frac{d_{\lambda_A}(E)+\log(1/\delta)}{\varepsilon^2}\right),
    \]
    then with probability at least \(1-2\delta\),
    \[
        \lVert B^{\top}(P^{\top} P-I)B\rVert _2
        \leq 2\varepsilon + 3\varepsilon^2.
    \]
\end{theorem}
\begin{proof}
    Write
    \[
        \Delta_A \coloneqq P_A^{\top} P_A - I_{d_A},
        \qquad
        \Delta_E \coloneqq P_E^{\top} P_E - I_{d_E}.
    \]
    Using \((X\otimes Y)^{\top}(X\otimes Y)=(X^{\top} X)\otimes(Y^{\top} Y)\), we have
    \[
        P^{\top} P - I_{d_A d_E}
        = (P_A^{\top} P_A)\otimes(P_E^{\top} P_E) - I_{d_A}\otimes I_{d_E}
        = \Delta_A\otimes I_{d_E} + I_{d_A}\otimes \Delta_E + \Delta_A\otimes \Delta_E.
    \]
    Therefore, by the triangle inequality,
    \begin{equation}\label{eq:T123-def}
        \lVert B^{\top}(P^{\top} P-I)B\rVert _2
        \leq T_1 + T_2 + T_3,
    \end{equation}
    where
    \[
        T_1 \coloneqq \bigl\|B^{\top}(\Delta_A\otimes I_{d_E})B\bigr\|_2,
        \quad
        T_2 \coloneqq \bigl\|B^{\top}(I_{d_A}\otimes \Delta_E)B\bigr\|_2,
        \quad
        T_3 \coloneqq \bigl\|B^{\top}(\Delta_A\otimes \Delta_E)B\bigr\|_2.
    \]

    \paragraph{Bounding \(T_1\).}
    Let \(A=U_A \Lambda_A U_A^{\top}\) and \(E=U_E \Lambda_E U_E^{\top}\) be eigendecompositions with \(\Lambda_A=\diag(\{\alpha_i\}_{i=1}^{d_A})\), \(\Lambda_E=\diag(\{\gamma_j\}_{j=1}^{d_E})\), and \(U_A,U_E\) orthonormal. Then \(F=A\otimes E\) is diagonalized by \(U\coloneqq U_A\otimes U_E\), and
    \[
        B = F^{1/2}(F+\lambda I)^{-1/2}
        = U D U^{\top},
    \]
    where \(D\) is diagonal with entries \(\beta_{ij}\) such that
    \[
        \beta_{ij}
        \coloneqq \sqrt{\frac{\alpha_i\gamma_j}{\alpha_i\gamma_j+\lambda}}, \qquad
        (i,j)\in[d_A]\times[d_E].
    \]
    Define \(\widetilde{\Delta}_A \coloneqq U_A^{\top} \Delta_A U_A\). Then using the basic identity \((X\otimes Y) (Z \otimes W) = (XZ) \otimes (YW)\),
    \[
        \begin{split}
            T_1
             & = \lVert U D U^{\top} (\Delta_A \otimes I_{d_E}) U D U^{\top} \rVert _2                            \\
             & = \lVert D (U_A^{\top} \otimes U_E^{\top}) (\Delta_A \otimes I_{d_E})(U_A \otimes U_E) D \rVert _2 \\
             & = \lVert D (U_A^{\top} \Delta_A U_A \otimes I_{d_E}) D  \rVert _2
            = \lVert D ( \widetilde{\Delta}_A\otimes I_{d_E}) D\rVert _2.
        \end{split}
    \]
    The matrix \(D(\widetilde{\Delta}_A\otimes I)D\) is not itself a Kronecker product, but it becomes
    block diagonal after a permutation of coordinates. Let \(\Pi\in\{0,1\}^{(d_A d_E)\times(d_A d_E)}\)
    be the canonical commutation matrix satisfying
    \[
        \Pi(X\otimes Y)\Pi^{\top} = Y\otimes X
        \qquad
        \text{for all conformable }X,Y.
    \]
    Since \(\Pi\) is orthogonal, \(\lVert M\rVert _2 = \lVert \Pi M\Pi^{\top}\rVert _2\) for any \(M\). Thus,
    \[
        T_1
        = \lVert \Pi D(\widetilde{\Delta}_A\otimes I_{d_E})D\Pi^{\top}\rVert _2
        = \lVert D_\Pi (I_{d_E}\otimes\widetilde{\Delta}_A) D_\Pi\rVert _2,
    \]
    where \(D_\Pi\coloneqq \Pi D\Pi^{\top}\) remains diagonal. The matrix \(D_\Pi(I_{d_E}\otimes\widetilde{\Delta}_A)D_\Pi\) is block diagonal with \(d_E\) blocks; the \(j\)-th block (corresponding to the \(j\)-th eigenvalue \(\gamma_j\)) equals
    \[
        D^{(j)} \widetilde{\Delta}_A D^{(j)}, \qquad
        D^{(j)} \coloneqq \diag(\{\beta_{ij}\}_{i=1}^{d_A})
        = \diag\left(\left\{\sqrt{\frac{\alpha_i\gamma_j}{\alpha_i\gamma_j+\lambda}}\right\}_{i=1}^{d_A}\right).
    \]
    Hence,
    \begin{equation}\label{eq:T1-block}
        T_1
        = \max_{j\in[d_E]} \lVert D^{(j)} \widetilde{\Delta}_A D^{(j)}\rVert _2.
    \end{equation}
    We now compare each \(D^{(j)}\) to a single dominating diagonal depending only on \(A\). Since \(\gamma_j\leq \lVert E \rVert_2\) and \(\alpha_i\geq 0\),\footnote{Note that the inequality holds trivially when \(\gamma_j = 0\).}
    \[
        \frac{\alpha_i\gamma_j}{\alpha_i\gamma_j+\lambda}
        \leq \frac{\alpha_i}{\alpha_i+\lambda/\gamma_j}
        \leq \frac{\alpha_i}{\alpha_i+\lambda/\lVert E \rVert_2}
        = \frac{\alpha_i}{\alpha_i+\lambda_E}.
    \]
    Define
    \[
        D_A^{\max}
        \coloneqq \diag\left(\left\{\sqrt{\frac{\alpha_i}{\alpha_i+\lambda_E}}\right\}_{i=1}^{d_A}\right).
    \]
    Then for each \(j\) there exists a diagonal contraction \(S^{(j)}\) such that
    \[
        D^{(j)} = S^{(j)} D_A^{\max} = D_A^{\max}S^{(j)},
        \qquad \lVert S^{(j)}\rVert _2\leq 1,
    \]
    and therefore,
    \[
        \lVert D^{(j)} \widetilde{\Delta}_A D^{(j)}\rVert _2 = \lVert S^{(j)} D_A^{\max} \widetilde{\Delta}_A D_A^{\max}S^{(j)}\rVert _2
        \leq \lVert D_A^{\max} \widetilde{\Delta}_A D_A^{\max}\rVert _2.
    \]
    Combining with \Cref{eq:T1-block} yields
    \[
        T_1
        \leq \lVert D_A^{\max} \widetilde{\Delta}_A D_A^{\max}\rVert _2.
    \]
    Finally, note that
    \[
        D_A^{\max} \widetilde{\Delta}_A D_A^{\max}
        = (U_A D_A^{\max})^{\top} (P_A^{\top} P_A - I_{d_A}) (U_A D_A^{\max}).
    \]
    Let \(M_A\coloneqq U_A D_A^{\max}\). Applying \Cref{lma:covariance-concentration} to \(M_A\) and sketching \(P_A\) (with failure probability \(\delta\)) gives that when
    \[
        m_A
        = \Omega \left(\frac{r(M_A^{\top} M_A)+\log(1/\delta)}{\varepsilon^2}\right),
    \]
    we have \(T_1\leq \varepsilon\) with probability at least \(1-\delta\). It remains to identify \(r(M_A^{\top} M_A)\). Since \(M_A^{\top} M_A=(D_A^{\max})^2\) is diagonal with spectral norm at most \(1\),
    \[
        r(M_A^{\top} M_A)
        = \frac{\tr((D_A^{\max})^2)}{\lVert (D_A^{\max})^2\rVert _2}
        = \tr((D_A^{\max})^2)
        = \sum_{i=1}^{d_A}\frac{\alpha_i}{\alpha_i+\lambda_E}
        = d_{\lambda_E}(A).
    \]
    Thus, under the stated condition on \(m_A\), with probability at least \(1-\delta\),
    \begin{equation}\label{eq:T1-eps}
        T_1 \leq \varepsilon.
    \end{equation}

    \paragraph{Bounding \(T_2\).}
    The bound for \(T_2\) is identical by symmetry (and is in fact simpler because \(I_{d_A}\otimes\widetilde{\Delta}_E\) is already block diagonal in the \(A\)-first ordering). Specifically, define \(\widetilde{\Delta}_E\coloneqq U_E^{\top} \Delta_E U_E\) and
    \[
        D_E^{\max}
        \coloneqq \diag \left(\left\{\sqrt{\frac{\gamma_j}{\gamma_j+\lambda_A}}\right\}_{j=1}^{d_E}\right),
        \qquad M_E\coloneqq U_E D_E^{\max}.
    \]
    Applying \Cref{lma:covariance-concentration} to \(M_E\) and \(P_E\) yields that, when
    \[
        m_E
        = \Omega \left(\frac{d_{\lambda_A}(E)+\log(1/\delta)}{\varepsilon^2}\right),
    \]
    we have with probability at least \(1-\delta\),
    \begin{equation}\label{eq:T2-eps}
        T_2 \leq \varepsilon.
    \end{equation}

    \paragraph{Bounding \(T_3\).}
    We show that the diagonal \(D\) is dominated by a Kronecker product of the dominating diagonals \(D_A^{\max}\) and \(D_E^{\max}\), up to a universal constant, provided \(\lambda\leq \lVert A \rVert_2\lVert E \rVert_2\). For each \((i,j)\),
    \[
        \beta_{ij}^2
        = \frac{\alpha_i\gamma_j}{\alpha_i\gamma_j+\lambda}.
    \]
    We claim that
    \begin{equation}\label{eq:key-ineq}
        \frac{\alpha_i\gamma_j}{\alpha_i\gamma_j+\lambda}
        \leq
        3\cdot \frac{\alpha_i}{\alpha_i+\lambda_E}\cdot \frac{\gamma_j}{\gamma_j+\lambda_A}.
    \end{equation}
    Indeed, \Cref{eq:key-ineq} is equivalent (after taking reciprocals of positive quantities) to
    \[
        (\alpha_i+\lambda_E)(\gamma_j+\lambda_A)
        \leq 3(\alpha_i\gamma_j+\lambda).
    \]
    Expanding the left-hand side gives
    \[
        (\alpha_i+\lambda_E)(\gamma_j+\lambda_A)
        = \alpha_i\gamma_j + \alpha_i\lambda_A + \gamma_j\lambda_E + \lambda_A\lambda_E.
    \]
    Using \(\alpha_i\leq \lVert A \rVert_2\), \(\gamma_j\leq \lVert E \rVert_2\), and the definitions \(\lambda_A=\lambda/\lVert A \rVert_2\), \(\lambda_E=\lambda/\lVert E \rVert_2\), we obtain
    \[
        \alpha_i\lambda_A \leq \lambda,\qquad
        \gamma_j\lambda_E \leq \lambda,\qquad
        \lambda_A\lambda_E = \frac{\lambda^2}{\lVert A \rVert_2\lVert E \rVert_2} \leq \lambda,
    \]
    where the last inequality uses the assumption \(\lambda\leq \lVert A \rVert_2\lVert E \rVert_2\). Therefore,
    \[
        (\alpha_i+\lambda_E)(\gamma_j+\lambda_A)
        \leq \alpha_i\gamma_j + 3\lambda
        \leq 3(\alpha_i\gamma_j+\lambda),
    \]
    which proves \Cref{eq:key-ineq}. Taking square-roots yields
    \[
        \beta_{ij}
        \leq \sqrt{3} \sqrt{\frac{\alpha_i}{\alpha_i+\lambda_E}} \sqrt{\frac{\gamma_j}{\gamma_j+\lambda_A}}.
    \]
    Therefore, there exists a diagonal contraction \(S\) such that
    \begin{equation*}\label{eq:D-domination}
        D = \sqrt{3}S(D_A^{\max}\otimes D_E^{\max}) = \sqrt{3}(D_A^{\max}\otimes D_E^{\max})S,
        \qquad \lVert S\rVert _2\leq 1,
    \end{equation*}
    and therefore
    \begin{align*}
        T_3
         & = 3\lVert S(D_A^{\max}\otimes D_E^{\max}) (\widetilde{\Delta}_A\otimes\widetilde{\Delta}_E) (D_A^{\max}\otimes D_E^{\max})S\rVert _2 \\
         & \leq 3 \lVert (D_A^{\max}\widetilde{\Delta}_A D_A^{\max})\otimes (D_E^{\max}\widetilde{\Delta}_E D_E^{\max})\rVert _2.
    \end{align*}
    Since \(\lVert X\otimes Y\rVert _2 = \lVert X\rVert _2\lVert Y\rVert _2\), this becomes
    \[
        T_3 \leq 3
        \lVert D_A^{\max}\widetilde{\Delta}_A D_A^{\max}\rVert _2
        \lVert D_E^{\max}\widetilde{\Delta}_E D_E^{\max}\rVert _2.
    \]
    On the event where both \Cref{eq:T1-eps} and \Cref{eq:T2-eps} hold, we obtain
    \begin{equation}\label{eq:T3-eps2}
        T_3 \leq 3\varepsilon^2.
    \end{equation}

    \paragraph{Putting together.}
    By \Cref{eq:T123-def,eq:T1-eps,eq:T2-eps,eq:T3-eps2}, on the intersection of the two concentration events (one for \(P_A\), one for \(P_E\)),
    \[
        \lVert B^{\top}(P^{\top} P-I)B\rVert _2
        \leq \varepsilon + \varepsilon + 3\varepsilon^2
        = 2\varepsilon + 3\varepsilon^2.
    \]
    The two concentration events each fail with probability at most \(\delta\), so by a union bound, the intersection holds with probability at least \(1-2\delta\). This completes the proof.
\end{proof}

\subsection{Note on Proof of \texorpdfstring{\Cref{thm:regularized-factorized-projection-upper-bound}}{Theorem 4}}
We note that while \Cref{thm:factorized-covariance} is stated with failure probability \(2\delta\) and deviation level \(2\varepsilon+3\varepsilon^2\), in the proof of \Cref{thm:regularized-factorized-projection-upper-bound}, we require it to be with failure probability \(\delta\) and deviation level \(\varepsilon\). This is only for notational convenience: given target parameters \((\varepsilon,\delta)\), one may apply the theorem with \(\epsilon \coloneqq \varepsilon/10\) and \(\eta \coloneqq \delta/2\), which yields probability at least \(1-2\eta = 1-\delta\) and deviation at most \(2\epsilon +3\epsilon ^2 \leq \varepsilon/2 \leq \varepsilon\) for \(\varepsilon\in(0,1)\).
\section{Proofs for \texorpdfstring{\Cref{subsec:projection-leakage}}{Section 3.1} (Leakage of Projection)}\label{adxsec:projection-leakage}
In this section, we prove \Cref{thm:projection-leakage}, which we first repeat the statement for convenience:

\begin{theorem*}
	Let \(\{g^{\prime}_{j}\}_{j=1}^k \subset \mathbb{R}^d\), and for each \(j\) let \(g^{\prime}_{j,\perp}\) denote the orthogonal projection of \(g_j^{\prime}\) onto \(\ker(F)\). Let \(k^{\prime} = \dim(\Span(\{g^{\prime}_{j,\perp}\}_{j=1}^k))\). For any \(\varepsilon, \delta \in (0,1)\), if
	\[
		m
		= \Omega \Bigg(\frac{r + \min\left(\log(k/\delta), k^{\prime} + \log(1/\delta)\right)}{\varepsilon^2}\Bigg),
	\]
	then with probability at least \(1-\delta\), the following holds for all \(j \in \{1, \dots, k\}\):
	\begin{itemize}
		\item \textbf{Unregularized}: For \(T_j \coloneqq (Pg)^{\top} (PFP^{\top})^{\dagger} (Pg^{\prime}_{j,\perp})\), we have
		      \[
			      \lvert T_j \rvert
			      \leq \varepsilon \frac{\lVert g \rVert _2\lVert g^{\prime}_{j,\perp}\rVert _2}{\lambda_{\min}^{+}(F)},
		      \]
		      where \(\lambda_{\min}^{+}(F)\) denotes the smallest non-zero eigenvalue of \(F\).
		\item \textbf{Regularized}: For \(T_{\lambda,j} \coloneqq (Pg)^{\top} (PFP^{\top} + \lambda I)^{-1} (Pg^{\prime}_{j,\perp})\), we have
		      \[
			      \lvert T_{\lambda,j} \rvert
			      \leq \varepsilon \lVert g \rVert _2\lVert g^{\prime}_{j,\perp}\rVert _2 \left(\frac{1}{\lambda} + \frac{2 \lVert F \rVert _2}{\lambda^2}\right)
		      \]
	\end{itemize}
\end{theorem*}

\subsection{Proof Plan for \texorpdfstring{\Cref{thm:projection-leakage}}{Theorem 6}}\label{adxsubsec:proof-plan-for-thm:projection-leakage}
The main organizing step is a deterministic reduction: \Cref{lma:two-conc-suffice} shows that \emph{both} the regularized and unregularized leakage bounds follow once the sketch \(P\) satisfies two concentration conditions with respect to an orthonormal basis \(U\) of \(\range(F)\):
\begin{enumerate*}[label=(\roman*)]
	\item subspace stability on \(\range(F)\), \(\lVert U^{\top}(P^{\top}P-I_d)U\rVert_2\leq\varepsilon\), and
	\item cross-term control between \(\range(F)\) and the kernel direction(s), \(\lVert U^{\top}(P^{\top}P-I_d)g^{\prime}_{j,\perp}\rVert_2\leq\varepsilon\lVert g^{\prime}_{j,\perp}\rVert_2\) for each \(j\).
\end{enumerate*}
Indeed, this is shown formally in \Cref{lma:two-conc-suffice}.

\begin{lemma}\label{lma:two-conc-suffice}
	Let \(\{g^{\prime}_{j}\}_{j=1}^k \subset \mathbb{R}^d\), and for each \(j\) let \(g^{\prime}_{j,\perp}\) denote the orthogonal projection of \(g^{\prime}_{j}\) onto \(\ker(F)\). Fix a realization of \(P\), and let \(U\in\mathbb{R}^{d\times r}\) be an orthonormal basis for \(\range(F)\). Assume that for some \(\varepsilon\in(0,1)\), the following two inequalities hold:
	\begin{enumerate}[label=(\roman*)]
		\item\label{lma:two-conc-suffice-i} \(\lVert U^{\top}(P^{\top}P-I_d)U\rVert_2 \leq \varepsilon\),
		\item\label{lma:two-conc-suffice-ii} \(\lVert U^{\top}(P^{\top}P-I_d)g^{\prime}_{j,\perp}\rVert_2 \leq \varepsilon \lVert g^{\prime}_{j,\perp}\rVert_2\) for every \(j\in\{1,\dots,k\}\).
	\end{enumerate}
	Then for any fixed \(g\in\range(F)\), the following bounds hold \emph{simultaneously for all} \(j\in\{1,\dots,k\}\):
	\[
		\bigl\lvert (Pg)^{\top}(PFP^{\top})^{\dagger}(Pg^{\prime}_{j,\perp})\bigr\rvert
		\leq \varepsilon \frac{1+\varepsilon}{(1-\varepsilon)^2}\cdot \frac{\lVert g\rVert_2 \lVert g^{\prime}_{j,\perp}\rVert_2}{\lambda_{\min}^{+}(F)}.
	\]
	and
	\[
		\bigl\lvert (Pg)^{\top}(PFP^{\top}+\lambda I)^{-1}(Pg^{\prime}_{j,\perp})\bigr\rvert
		\leq \varepsilon \lVert g\rVert_2 \lVert g^{\prime}_{j,\perp}\rVert_2\left(\frac{1}{\lambda}+\frac{2\lVert F\rVert_2}{\lambda^2}\right).
	\]
\end{lemma}
\begin{proof}
	We prove the unregulairzed case first.
	\paragraph{Unregularized Case.}
	Fix \(j\). Using \(\range(PFP^{\top})=\range(PU)\), let
	\[
		\Pi_{PU}\coloneqq PU (U^{\top}P^{\top}PU)^{-1}(PU)^{\top}
	\]
	denote the orthogonal projector onto \(\range(PU)\). Then
	\[
		(Pg)^{\top}(PFP^{\top})^{\dagger}(Pg^{\prime}_{j,\perp})
		= g^{\top}P^{\top}(PFP^{\top})^{\dagger}\Pi_{PU}Pg^{\prime}_{j,\perp},
	\]
	and hence
	\[
		\bigl\lvert (Pg)^{\top}(PFP^{\top})^{\dagger}(Pg^{\prime}_{j,\perp})\bigr\rvert
		\leq \lVert (PFP^{\top})^{\dagger}\rVert_2 \lVert Pg\rVert_2 \lVert \Pi_{PU}Pg^{\prime}_{j,\perp}\rVert_2.
	\]
	We bound the three terms on the right-hand side.

	First, we bound \(\lVert (PFP^{\top})^{\dagger}\rVert_2\). Write the compact eigendecomposition \(F=U\Sigma U^{\top}\), where \(\Sigma\succ 0\) is diagonal and \(\lVert \Sigma^{-1}\rVert_2 = 1/\lambda_{\min}^{+}(F)\). Since \(PFP^{\top}=(PU)\Sigma(PU)^{\top}\), we have
	\[
		\lVert (PFP^{\top})^{\dagger}\rVert_2
		=\lVert (PU)^{\dagger}\rVert_2^2 \lVert \Sigma^{-1}\rVert_2
		=\frac{1}{\sigma_{\min}(PU)^2}\cdot\frac{1}{\lambda_{\min}^{+}(F)}.
	\]
	Moreover, assumption \labelcref{lma:two-conc-suffice-i} implies that all eigenvalues of \(U^{\top}P^{\top}PU=(PU)^{\top}(PU)\) lie in \([1-\varepsilon,1+\varepsilon]\), hence \(\sigma_{\min}(PU)^2\geq 1-\varepsilon\) and
	\[
		\lVert (PFP^{\top})^{\dagger}\rVert_2
		\leq \frac{1}{(1-\varepsilon)\lambda_{\min}^{+}(F)}.
	\]
	To bound \(\lVert Pg\rVert_2\), write \(g=Uh\), we have \(\lVert Pg\rVert_2\leq \sqrt{1+\varepsilon} \lVert g\rVert_2\) since
	\[
		\lVert Pg\rVert_2^2
		=h^{\top}(U^{\top}P^{\top}PU)h
		\leq (1+\varepsilon)\lVert h\rVert_2^2=(1+\varepsilon)\lVert g\rVert_2^2.
	\]
	Finally, to bound \(\lVert \Pi_{PU}Pg^{\prime}_{j,\perp}\rVert_2\), with \(U^{\top}g^{\prime}_{j,\perp}=0\), we have \(U^{\top}P^{\top}Pg^{\prime}_{j,\perp}=U^{\top}(P^{\top}P-I_d)g^{\prime}_{j,\perp}\), hence
	\[
		\begin{split}
			\lVert \Pi_{PU}Pg^{\prime}_{j,\perp}\rVert_2
			 & =\bigl\lVert PU(U^{\top}P^{\top}PU)^{-1}U^{\top}P^{\top}Pg^{\prime}_{j,\perp}\bigr\rVert_2                                  \\
			 & \leq \lVert PU\rVert_2 \lVert (U^{\top}P^{\top}PU)^{-1}\rVert_2 \lVert U^{\top}(P^{\top}P-I_d)g^{\prime}_{j,\perp}\rVert_2.
		\end{split}
	\]
	By assumption \labelcref{lma:two-conc-suffice-i}, \(\lVert PU\rVert_2=\sigma_{\max}(PU)\leq \sqrt{1+\varepsilon}\) and \(\lVert (U^{\top}P^{\top}PU)^{-1}\rVert_2\leq 1/(1-\varepsilon)\). By assumption \labelcref{lma:two-conc-suffice-ii}, \(\lVert U^{\top}(P^{\top}P-I_d)g^{\prime}_{j,\perp}\rVert_2\leq \varepsilon\lVert g^{\prime}_{j,\perp}\rVert_2\). Therefore,
	\[
		\lVert \Pi_{PU}Pg^{\prime}_{j,\perp}\rVert_2
		\leq \varepsilon \frac{\sqrt{1+\varepsilon}}{1-\varepsilon} \lVert g^{\prime}_{j,\perp}\rVert_2.
	\]
	Combining the bounds yields
	\[
		\lvert (Pg)^{\top}(PFP^{\top})^{\dagger}(Pg^{\prime}_{j,\perp}) \rvert
		\leq \varepsilon \frac{1+\varepsilon}{(1-\varepsilon)^2}\cdot \frac{\lVert g\rVert_2\lVert g^{\prime}_{j,\perp}\rVert_2}{\lambda_{\min}^{+}(F)}.
	\]

	\paragraph{Regularized Case.}
	Fix \(g\in\range(F)\) and \(j\). Write \(g=Uh\). Set \(V\coloneqq PF^{1/2}\), so that \(PFP^{\top}=VV^{\top}\). By the Woodbury identity,
	\[
		(VV^{\top}+\lambda I)^{-1}
		=\frac{1}{\lambda}I-\frac{1}{\lambda^2}V\Bigl(I+\frac{1}{\lambda}V^{\top}V\Bigr)^{-1}V^{\top}.
	\]
	Therefore, writing \(T_{\lambda,j}=(Pg)^{\top}(VV^{\top}+\lambda I)^{-1}(Pg^{\prime}_{j,\perp})\), we have the decomposition \(T_{\lambda,j}=T_{\lambda,j}^{(1)}-T_{\lambda,j}^{(2)}\) where
	\[
		T_{\lambda,j}^{(1)}=\frac{1}{\lambda}g^{\top}P^{\top}Pg^{\prime}_{j,\perp},
		\qquad
		T_{\lambda,j}^{(2)}=\frac{1}{\lambda^2}g^{\top}P^{\top}PF^{1/2}\Bigl(I+\frac{1}{\lambda}V^{\top}V\Bigr)^{-1}F^{1/2}P^{\top}Pg^{\prime}_{j,\perp}.
	\]
	Since \(g^{\top}g^{\prime}_{j,\perp}=0\),
	\[
		\lvert T_{\lambda,j}^{(1)}\rvert
		=\frac{1}{\lambda}\bigl\lvert g^{\top}(P^{\top}P-I_d)g^{\prime}_{j,\perp}\bigr\rvert
		=\frac{1}{\lambda}\bigl\lvert h^{\top}U^{\top}(P^{\top}P-I_d)g^{\prime}_{j,\perp}\bigr\rvert
		\leq \frac{\varepsilon}{\lambda} \lVert g\rVert_2 \lVert g^{\prime}_{j,\perp}\rVert_2,
	\]
	using Cauchy--Schwarz and assumption \labelcref{lma:two-conc-suffice-ii}.

	Next, we bound \(T_{\lambda,j}^{(2)}\). Note that
	\(V^{\top}V = F^{1/2}P^{\top}PF^{1/2}\succeq 0\), so
	\(I+\frac{1}{\lambda}V^{\top}V\succeq I\), which implies
	\(\bigl\lVert (I+\frac{1}{\lambda}V^{\top}V)^{-1}\bigr\rVert_2\leq 1\).
	Moreover, since \(\range(F^{1/2})=\range(F)\), we have
	\(F^{1/2}=\Pi_F F^{1/2}=F^{1/2}\Pi_F\), and hence we may insert \(\Pi_F\) on both sides of each \(F^{1/2}\) factor. Using sub-multiplicativity and \(\lVert F^{1/2}\rVert_2^2=\lVert F\rVert_2\), we obtain
	\[
		\lvert T_{\lambda,j}^{(2)}\rvert
		\leq \frac{1}{\lambda^2}
		\lVert \Pi_F P^{\top}Pg\rVert_2
		\lVert F\rVert_2
		\lVert \Pi_F P^{\top}Pg^{\prime}_{j,\perp}\rVert_2.
	\]
	Since \(\Pi_F=UU^{\top}\) and \(g=Uh\),
	\[
		\lVert \Pi_F P^{\top}Pg\rVert_2
		=\lVert U(U^{\top}P^{\top}PU)h\rVert_2
		\leq \lVert U^{\top}P^{\top}PU\rVert_2 \lVert g\rVert_2
		\leq (1+\varepsilon) \lVert g\rVert_2,
	\]
	where we used \(U^{\top}P^{\top}PU = I_r + U^{\top}(P^{\top}P-I_d)U\) and assumption \labelcref{lma:two-conc-suffice-i}. Moreover, since \(U^{\top}g^{\prime}_{j,\perp}=0\),
	\[
		\lVert \Pi_F P^{\top}Pg^{\prime}_{j,\perp}\rVert_2
		=\lVert U^{\top}P^{\top}Pg^{\prime}_{j,\perp}\rVert_2
		=\lVert U^{\top}(P^{\top}P-I_d)g^{\prime}_{j,\perp}\rVert_2
		\leq \varepsilon \lVert g^{\prime}_{j,\perp}\rVert_2,
	\]
	by assumption \labelcref{lma:two-conc-suffice-ii}. Hence
	\[
		\lvert T_{\lambda,j}^{(2)}\rvert
		\leq \frac{\lVert F\rVert_2}{\lambda^2}(1+\varepsilon)\varepsilon \lVert g\rVert_2 \lVert g^{\prime}_{j,\perp}\rVert_2.
	\]
	Combining the two pieces and using \(\varepsilon\leq 1\) gives
	\[
		\lvert T_{\lambda,j}\rvert
		\leq \varepsilon \lVert g\rVert_2 \lVert g^{\prime}_{j,\perp}\rVert_2\left(\frac{1}{\lambda}+\frac{2\lVert F\rVert_2}{\lambda^2}\right).
	\]
\end{proof}

Thus, the remaining work in the proof is to verify these two conditions in the single-gradient and multi-gradient regimes. We break the proof into the following cases:
\begin{enumerate}
	\item For a single kernel component \(g^{\prime}_{\perp} \in \ker(F)\):
	      \begin{itemize}
		      \item \Cref{prop:leakage}: prove the bound for unregularized and regularized case.
	      \end{itemize}
	\item Extend both cases to \(\{g^{\prime}_{j,\perp}\}_{j=1}^{k} \subseteq \ker (F)\) with \(k^{\prime} = \dim(\Span(\{g^{\prime}_{j,\perp}\}_{j=1}^k))\):
	      \begin{itemize}
		      \item \Cref{prop:multiple-k-prime}: subspace argument with \(m = \Omega \left( \frac{r + k^{\prime} + \log(1/\delta)}{\varepsilon^2} \right)\).
		      \item \Cref{prop:multiple-k}: union-bound argument with \(m = \Omega \left( \frac{r + \log(k/\delta)}{\varepsilon^2} \right)\).
	      \end{itemize}
\end{enumerate}
We now start the proof.

\subsection{Proof of Single Test Gradient Leakage}\label{adxsubsec:leakage-single-test}
\begin{proposition}\label{prop:leakage}
	Assume \(g \in \range(F)\) and let \(g^{\prime}\in\mathbb{R}^d\). For any \(\varepsilon, \delta \in (0,1)\), if \(m = \Omega (\varepsilon^{-2} (r + \log(1/\delta)))\), then with probability at least \(1-\delta\),
	\begin{enumerate}
		\item \textbf{Unregularized}: Let \(T \coloneqq (Pg)^{\top} (PFP^{\top})^{\dagger} (Pg^{\prime}_{\perp})\), then
		      \[
			      \lvert T \rvert \leq \varepsilon \frac{\lVert g \rVert _2\lVert g^{\prime}_{\perp}\rVert _2}{\lambda_{\min}^{+}(F)},
		      \]
		      where \(\lambda_{\min}^{+}(F)\) denotes the smallest non-zero eigenvalue of \(F\).
		\item \textbf{Regularized}: Let \(T_\lambda \coloneqq (Pg)^{\top} (PFP^{\top} + \lambda I)^{-1} (Pg^{\prime}_{\perp})\), then
		      \[
			      \lvert T_\lambda \rvert
			      \leq \varepsilon \lVert g \rVert _2\lVert g^{\prime}_{\perp}\rVert _2 \left(\frac{1}{\lambda} + \frac{2\lVert F \rVert _2}{\lambda^2}\right).
		      \]
	\end{enumerate}
\end{proposition}
\begin{proof}
	From \Cref{lma:two-conc-suffice}, it suffices to verify conditions
	\labelcref{lma:two-conc-suffice-i,lma:two-conc-suffice-ii}.

	Let \(S \coloneqq \Span(\range(F)\cup\{g^{\prime}_{\perp}\})\), so \(\dim(S)=r+1\), and let \(W\in\mathbb{R}^{d\times(r+1)}\) be an orthonormal basis for \(S\). Fix any \(\eta\in(0,1)\) and define the event
	\[
		\mathcal{E}(\eta)
		\coloneqq\left\{\bigl\lVert W^{\top}(P^{\top}P-I_d)W\bigr\rVert_2\leq \eta\right\}.
	\]
	On \(\mathcal{E}(\eta)\), for any orthonormal basis \(U\in\mathbb{R}^{d\times r}\) of \(\range(F)\subseteq S\) there exists \(R\in\mathbb{R}^{(r+1)\times r}\) with \(R^{\top}R=I_r\) such that \(U=WR\). Thus,
	\[
		\lVert U^{\top}(P^{\top}P-I_d)U\rVert_2
		=\lVert R^{\top}W^{\top}(P^{\top}P-I_d)WR\rVert_2
		\leq \eta.
	\]
	Moreover, since \(g^{\prime}_{\perp}\in S\), we have \(g^{\prime}_{\perp}=WW^{\top}g^{\prime}_{\perp}\) and \(\lVert W^{\top}g^{\prime}_{\perp}\rVert_2=\lVert g^{\prime}_{\perp}\rVert_2\), and hence
	\[
		\lVert U^{\top}(P^{\top}P-I_d)g^{\prime}_{\perp}\rVert_2
		=\lVert R^{\top}W^{\top}(P^{\top}P-I_d)W W^{\top}g^{\prime}_{\perp}\rVert_2
		\leq \eta\lVert g^{\prime}_{\perp}\rVert_2.
	\]
	Therefore, on \(\mathcal{E}(\eta)\) the assumptions of \Cref{lma:two-conc-suffice} hold with parameter \(\eta\).

	\paragraph{Unregularized.}
	By \Cref{lma:covariance-concentration} applied to \(S\), if
	\(m=\Omega\bigl(((r+1)+\log(1/\delta))/\eta^2\bigr)\), then \(\mathbb{P}(\mathcal{E}(\eta))\geq 1-\delta\).
	Taking \(\eta=\varepsilon/4\) and applying \Cref{lma:two-conc-suffice} yields
	\[
		\lvert T \rvert
		\leq \frac{\varepsilon}{4}\cdot\frac{1+\varepsilon/4}{(1-\varepsilon/4)^2}\cdot\frac{\lVert g\rVert_2\lVert g^{\prime}_{\perp}\rVert_2}{\lambda_{\min}^{+}(F)}.
	\]
	As in the previous argument, \(\frac{1+\varepsilon/4}{(1-\varepsilon/4)^2}\leq \frac{20}{9}\), hence the prefactor is \(\leq\varepsilon\).

	\paragraph{Regularized.}
	Taking \(\eta=\varepsilon\) and applying \Cref{lma:two-conc-suffice} gives
	\[
		\lvert T_\lambda \rvert
		\leq \varepsilon \lVert g \rVert _2\lVert g^{\prime}_{\perp}\rVert _2 \left(\frac{1}{\lambda} + \frac{2\lVert F \rVert _2}{\lambda^2}\right).
	\]
\end{proof}

This result shows that, in the unregularized case, the kernel leakage term decays at rate \(O(m^{-1/2})\), with constants that depend on the smallest non-zero eigenvalue of \(F\). On the other hand, in the regularized case, the kernel leakage term also decays at rate \(O(m^{-1/2})\), but with constants depending on \(\lVert F \rVert _2\) and the regularization parameter \(\lambda\). This dependence reflects the sensitivity of the pseudoinverse to near-degeneracies in the spectrum of \(F\).

\subsection{Proof of Multiple Test Gradients Leakage}\label{adxsubsec:leakage-multiple-test}
Having established the deterministic reduction in \Cref{lma:two-conc-suffice}, we now show how to enforce its two assumptions uniformly over multiple test gradients. First, we observe that the previous analysis naturally generalizes by considering the subspace spanned by all test gradients.

\begin{proposition}\label{prop:multiple-k-prime}
	Let \(\{g^{\prime}_{j}\}_{j=1}^k \subset \mathbb{R}^d\), and for each \(j\) let \(g^{\prime}_{j,\perp}\) denote the orthogonal projection of \(g^{\prime}_{j}\) onto \(\ker(F)\). Let
	\(k^{\prime} = \dim\bigl(\Span(\{g^{\prime}_{j,\perp}\}_{j=1}^k)\bigr)\).
	For any \(\varepsilon, \delta \in (0,1)\), if
	\[
		m
		= \Omega \left( \frac{r + k^{\prime} + \log(1/\delta)}{\varepsilon^2} \right),
	\]
	then with probability at least \(1-\delta\), the leakage bounds in \Cref{prop:leakage} hold simultaneously for all \(j \in \{1, \dots, k\}\).
\end{proposition}
\begin{proof}
	Let \(S \coloneqq \Span(\range(F) \cup \{g^{\prime}_{j,\perp}\}_{j=1}^k)\), so that \(\dim(S)=r+k^{\prime}\), and let \(W \in \mathbb{R}^{d\times (r+k^{\prime})}\) be an orthonormal basis for \(S\). By \Cref{lma:covariance-concentration}, with probability at least \(1-\delta\),
	\[
		\bigl\lVert W^{\top}(P^{\top}P-I_d)W\bigr\rVert_2 \leq \varepsilon,
	\]
	provided that \(m = \Omega\bigl(\varepsilon^{-2}(r+k^{\prime}+\log(1/\delta))\bigr)\). On this event, for all \(x,y\in S\),
	\[
		\bigl\lvert x^{\top}(P^{\top}P-I_d)y\bigr\rvert
		= \bigl\lvert (W^{\top}x)^{\top}  W^{\top}(P^{\top}P-I_d)W (W^{\top}y)\bigr\rvert
		\leq \varepsilon \lVert x\rVert_2 \lVert y\rVert_2.
	\]

	Now let \(U \in \mathbb{R}^{d\times r}\) be an orthonormal basis for \(\range(F)\). Since \(\range(F)\subseteq S\), the columns of \(U\) are contained in \(S\), and hence
	\[
		\bigl\lVert U^{\top}(P^{\top}P-I_d)U\bigr\rVert_2
		\leq \bigl\lVert W^{\top}(P^{\top}P-I_d)W\bigr\rVert_2
		\leq \varepsilon.
	\]
	Moreover, for each \(j\), using that \(U\) has orthonormal columns and \(\range(F)\subseteq S\), we have
	\[
		\begin{split}
			\bigl\lVert U^{\top}(P^{\top}P-I_d)g_{j,\perp}^{\prime}\bigr\rVert_2
			 & = \sup_{\substack{a\in\mathbb{R}^r \\ \lVert a\rVert_2=1}} \bigl\lvert a^{\top}U^{\top}(P^{\top}P-I_d)g_{j,\perp}^{\prime}\bigr\rvert \\
			 & = \sup_{\substack{x\in\range(F)    \\ \lVert x\rVert_2=1}} \bigl\lvert x^{\top}(P^{\top}P-I_d)g_{j,\perp}^{\prime}\bigr\rvert
			\leq \varepsilon \lVert g_{j,\perp}^{\prime}\rVert_2,
		\end{split}
	\]
	where the last inequality applies the bilinear bound above with \(x\in\range(F)\subseteq S\) and \(y=g_{j,\perp}^{\prime}\in S\).
	Thus, the assumptions of \Cref{lma:two-conc-suffice} hold simultaneously for all \(j\), and the corollary follows by applying \Cref{lma:two-conc-suffice}.
\end{proof}

While \Cref{prop:multiple-k-prime} is effective when the test gradients are low-dimensional, as \(g_j^{\prime} \in \mathbb{R}^d\) lies in high dimension, it is almost certain that \(k^{\prime}\) will be large, and most likely \(k^{\prime} \approx k\). In this case, by directly controlling the concentration of the bilinear form, we can obtain a bound that scales only logarithmically with the \emph{number} of test gradients.

\begin{proposition}\label{prop:multiple-k}
	Let \(\{g^{\prime}_{j}\}_{j=1}^k \subset \mathbb{R}^d\), and for each \(j\) let \(g^{\prime}_{j,\perp}\) denote the orthogonal projection of \(g^{\prime}_{j}\) onto \(\ker(F)\).
	For any \(\varepsilon, \delta \in (0,1)\), if
	\[
		m
		= \Omega \left(\frac{r + \log(k/\delta)}{\varepsilon^{2}}\right),
	\]
	then with probability at least \(1-\delta\), the leakage bounds in \Cref{prop:leakage} hold for all \(j \in \{1, \dots, k\}\).
\end{proposition}
\begin{proof}
	Let \(U \in \mathbb{R}^{d\times r}\) be an orthonormal basis for \(\range(F)\). We will verify the two assumptions of \Cref{lma:two-conc-suffice} uniformly over \(\{g_{j,\perp}^{\prime}\}_{j=1}^k\). Since \Cref{lma:two-conc-suffice} incurs a benign factor \(\frac{1+\varepsilon}{(1-\varepsilon)^2}\) in the unregularized case, we will run the concentration argument below with accuracy parameter \(\varepsilon/4\); the resulting constant-factor strengthening is absorbed by the \(\Omega(\cdot)\) sample complexity.

	\paragraph{Controlling \(\lVert U^{\top}(P^{\top}P-I_d)U\rVert_2\).}
	By \Cref{lma:covariance-concentration} applied to the \(r\)-dimensional subspace \(\range(F)\), with probability at least \(1-\delta/2\),
	\[
		\lVert U^{\top}(P^{\top}P-I_d)U\rVert_2 \leq \varepsilon,
	\]
	provided that \(m = \Omega\bigl(\varepsilon^{-2}(r+\log(2/\delta))\bigr)\).

	\paragraph{Controlling \(\lVert U^{\top}(P^{\top}P-I_d)g^{\prime}_{j,\perp}\rVert_2\) for all \(j\).}
	Fix \(g^{\prime}_{\perp} \in \ker(F)\) with \(\lVert g^{\prime}_{\perp} \rVert_2 = 1\). Note that
	\[
		\lVert U^{\top} (P^{\top} P - I_d) g^{\prime}_{\perp} \rVert_2
		= \sup_{a \in S^{r-1}} \lvert (Ua)^{\top} (P^{\top} P - I_d) g^{\prime}_{\perp} \rvert
		= \sup_{x \in US^{r-1}} \lvert x^{\top} (P^{\top}P - I_d) g^{\prime}_{\perp} \rvert,
	\]
	where \(US^{r-1} = \{Ua \colon a \in \mathbb{R}^r, \lVert a \rVert_2 = 1\}\) is the unit sphere in \(\range(F)\).

	By the polarization identity \(x^{\top} y = \frac{1}{4}(\lVert x + y \rVert_2^2 - \lVert x - y \rVert_2^2)\), the bilinear form can be written as:
	\[
		x^{\top} (P^{\top} P - I_d) g^{\prime}_{\perp}
		= \frac{1}{4} \left( \lVert P(x+g^{\prime}_{\perp}) \rVert_2^2 - \lVert x+g^{\prime}_{\perp} \rVert_2^2 \right) - \frac{1}{4} \left( \lVert P(x-g^{\prime}_{\perp}) \rVert_2^2 - \lVert x-g^{\prime}_{\perp} \rVert_2^2 \right).
	\]
	To bound this uniformly over \(x \in US^{r-1}\), define the set \(T = T_{+} \cup T_{-}\), where \(T_{\pm} = \{x \pm g^{\prime}_{\perp} \colon x \in US^{r-1}\}\). It follows that
	\[
		\sup_{x \in US^{r-1}} \lvert x^{\top} (P^{\top} P - I_d) g^{\prime}_{\perp} \rvert
		\leq \frac{1}{2} \sup_{z \in T} \lvert \lVert Pz \rVert _2^2 - \lVert z \rVert _2^2 \rvert.
	\]

	Define the sub-Gaussian stochastic process \(Y_z = \lVert Pz \rVert _2 - \lVert z \rVert _2\) for \(z \in T\), similar to the proof of \Cref{prop:high-prob-cov-estimation}. Applying the Talagrand comparison inequality~\citep[Theorem 3.2]{dirksen2015tail}, with probability at least \(1 - 2e^{-u}\),
	\[
		\sup_{z \in T} \lvert \lVert Pz \rVert_2 - \lVert z \rVert _2 \rvert
		\leq \frac{C}{\sqrt{m}} (\gamma(T) + \sqrt{u} \cdot \rad(T)).
	\]
	We analyze the radius and Gaussian complexity of \(T\):
	\begin{itemize}
		\item \(\rad(T)\): For any \(z \in T\), \(z = x \pm g^{\prime}_{\perp}\). Since \(x \perp g^{\prime}_{\perp}\) as \(x \in \range(F)\) and \(g^{\prime}_{\perp} \in \ker(F)\), the Pythagorean theorem gives \(\lVert z \rVert_2^2 = \lVert x \rVert_2^2 + \lVert g^{\prime}_{\perp} \rVert_2^2 = 1 + 1 = 2\), giving \(\rad(T) = \sqrt{2} = O(1)\).
		\item \(\gamma(T)\): By definition, \(\gamma(T) = \mathbb{E}[\sup_{z \in T} \lvert \langle h, z \rangle \rvert]\) for \(h \sim \mathcal{N}(0, I_d)\). For \(z = x \pm g^{\prime}_{\perp}\), we have \(\langle h, x\pm g^{\prime}_{\perp} \rangle = \langle h, x \rangle \pm \langle h, g^{\prime}_{\perp} \rangle\). Thus,
		      \[
			      \gamma(T)
			      \leq \mathbb{E}\left[ \sup_{x \in US^{r-1}} \lvert \langle h, x \rangle \rvert \right] + \mathbb{E}[\lvert \langle h, g^{\prime}_{\perp} \rangle \rvert].
		      \]
		      The first term is the Gaussian complexity of the unit sphere in an \(r\)-dimensional subspace, which is bounded by \(\sqrt{r}\). The second term is \(\mathbb{E}[Z]\) for \(Z \sim \mathcal{N}(0, 1)\), which is \(\sqrt{2 / \pi}\). Overall, \(\gamma(T) \leq \sqrt{r} + \sqrt{2 / \pi} \lesssim \sqrt{r}\).
	\end{itemize}
	With again \(\lvert a^2 - b^2 \rvert \leq \lvert a - b \rvert (\lvert a - b \rvert + 2b)\) with \(a = \lVert Pz \rVert _2\) and \(b = \lVert z \rVert _2 = \sqrt{2}\), we have
	\[
		\sup_{z \in T} \lvert \lVert Pz \rVert _2^2 - \lVert z \rVert _2^2 \rvert
		\leq \frac{C}{\sqrt{m}} (\gamma(T) + \sqrt{u} \rad(T)) \left( \frac{C}{\sqrt{m}} (\gamma(T) + \sqrt{u} \rad(T)) + 2 \rad(T)\right).
	\]
	Distributing the terms and substituting \(\rad(T) = \sqrt{2}\) and \(\gamma(T) \leq \sqrt{r} + 1\), we have
	\[
		\sup_{z \in T} \lvert \lVert Pz \rVert _2^2 - \lVert z \rVert _2^2 \rvert
		\leq C \left( \frac{r + u}{m} + \sqrt{\frac{r + u}{m}}\right).
	\]

	Setting \(u = \log(4k/\delta)\) ensures that \(2e^{-u} = \delta/(2k)\). Hence, for a fixed \(g^{\prime}_{\perp}\), we have \(\lVert U^{\top} (P^{\top} P - I_d) g^{\prime}_{\perp} \rVert _2 \leq \varepsilon\) with probability at least \(1 - \delta/(2k)\), provided that \(m = \Omega((r + \log(k / \delta)) / \varepsilon^2)\). By a union bound over \(j \in \{1, \ldots, k\}\), the bound holds simultaneously for all \(k\) test gradients with probability at least \(1-\delta/2\).

	Finally, taking a union bound over the two failure events (the subspace event and the \(k\) bilinear events), the same argument (with \(\varepsilon\) replaced by \(\varepsilon/4\)) yields that with probability at least \(1-\delta\), \(\bigl\lVert U^{\top}(P^{\top}P-I_d)U\bigr\rVert_2 \leq \varepsilon/4\) and
	\[
		\bigl\lVert U^{\top}(P^{\top}P-I_d)g^{\prime}_{j,\perp}\bigr\rVert_2
		\leq (\varepsilon/4) \lVert g^{\prime}_{j,\perp}\rVert_2
		\quad \text{for all }j\in\{1,\dots,k\}.
	\]
	On this event, we apply \Cref{lma:two-conc-suffice} with parameter \(\varepsilon/4\). The regularized leakage bound then holds with prefactor \(\varepsilon/4\leq\varepsilon\). For the unregularized leakage bound, we obtain
	\[
		\bigl\lvert (Pg)^{\top}(PFP^{\top})^{\dagger}(Pg^{\prime}_{j,\perp})\bigr\rvert
		\leq \frac{\varepsilon}{4}\cdot\frac{1+\varepsilon/4}{(1-\varepsilon/4)^2}\cdot\frac{\lVert g\rVert_2\lVert g^{\prime}_{j,\perp}\rVert_2}{\lambda_{\min}^{+}(F)}
		\leq \varepsilon\cdot\frac{\lVert g\rVert_2\lVert g^{\prime}_{j,\perp}\rVert_2}{\lambda_{\min}^{+}(F)},
	\]
	using \(\frac{1+\varepsilon/4}{(1-\varepsilon/4)^2}\leq \frac{20}{9}\) as in \Cref{prop:leakage}. This completes the proof.
\end{proof}
\section{Proofs for \texorpdfstring{\Cref{subsec:factorized-influence-leakage}}{Section 3.2} (Leakage of Factorized Influence)}\label{adxsec:factorized-influence-leakage}
In this subsection, we extend the leakage analysis in \Cref{adxsec:projection-leakage} to the factorized influence setting. We consider curvature matrices of the form
\[
    F = A \otimes E \in \mathbb{R}^{(d_A d_E)\times(d_A d_E)},
\]
where \(A\succeq 0\) and \(E\succeq 0\). We analyze a factorized sketch
\[
    P = P_A \otimes P_E,
    \qquad
    P_A\in\mathbb{R}^{m_A\times d_A},
    P_E\in\mathbb{R}^{m_E\times d_E},
\]
so that \(P\in\mathbb{R}^{(m_A m_E)\times(d_A d_E)} = \mathbb{R}^{m \times d}\) with \(m = m_A m_E\) and \(d = d_A d_E\). Throughout, we assume \(P_A\) and \(P_E\) are both oblivious sketches as defined in \Cref{thm:regularized-projection-upper-bound}. We will show that the only new work needed is to bound the cross-term quantity
\(\|U^\top(P^\top P-I)g^{\prime}_{\perp}\|_2\) (for kernel components \(g^{\prime}_{\perp}\in\ker(F)\)) appearing in \Cref{lma:two-conc-suffice} via factor-level primitive bounds.

\begin{theorem*}
    Let \(A, E \succeq 0\) and \(F \coloneqq A \otimes E\), and let \(P = P_A \otimes P_E\) with \(P_A\in\mathbb{R}^{m_A\times d_A}\) and \(P_E\in\mathbb{R}^{m_E\times d_E}\). Let \(r_A\coloneqq\rank(A)\), \(r_E\coloneqq\rank(E)\), and \(r\coloneqq\rank(F)=r_A r_E\).

    Let \(\{g^{\prime}_j\}_{j=1}^k \subset \mathbb{R}^{d_A d_E}\) be test gradients of the form \(g^{\prime}_j=a^{\prime}_j\otimes e^{\prime}_j\). For each \(j\), define the kernel component \(g^{\prime}_{j,\perp}\coloneqq \Pi_{\ker(F)}g^{\prime}_j\). Write
    \(a^{\prime}_j=a^{\prime}_{j,\parallel}+a^{\prime}_{j,\perp}\) with \(a^{\prime}_{j,\parallel}\in\range(A)\) and \(a^{\prime}_{j,\perp}\perp\range(A)\), and similarly \(e^{\prime}_j=e^{\prime}_{j,\parallel}+e^{\prime}_{j,\perp}\).
    Define \(k_A \coloneqq \sum_{j=1}^k \mathbbm{1}(a^{\prime}_{j,\perp}\neq 0)\), \(k_E \coloneqq \sum_{j=1}^k \mathbbm{1}(e^{\prime}_{j,\perp}\neq 0)\), and
    \(k_A^{\prime} \coloneqq \dim\bigl(\Span(\{a^{\prime}_{j,\perp}\}_{j=1}^k)\bigr)\), \(k_E^{\prime} \coloneqq \dim\bigl(\Span(\{e^{\prime}_{j,\perp}\}_{j=1}^k)\bigr)\).
    For any \(\varepsilon,\delta\in(0,1)\), if
    \[
        m_A = \Omega\Bigg(\frac{r_A+\min\{\log(\frac{k_A}{\delta}), k_A^{\prime}+\log(\frac{1}{\delta})\}}{\varepsilon^2}\Bigg),
        m_E = \Omega\Bigg(\frac{r_E+\min\{\log(\frac{k_E}{\delta}), k_E^{\prime}+\log(\frac{1}{\delta})\}}{\varepsilon^2}\Bigg),
    \]
    then with probability at least \(1-\delta\), the following bounds hold simultaneously for all \(j\in\{1,\dots,k\}\):
    \begin{itemize}[leftmargin=*]
        \item \textbf{Unregularized:} \(\lvert \widetilde{\tau}_0(g, g^{\prime}_{j,\perp}) \rvert \leq \varepsilon  \lVert g \rVert _2\lVert g^{\prime}_{j,\perp}\rVert _2 / \lambda_{\min}^{+}(F)\).
        \item \textbf{Regularized:} \(\lvert \widetilde{\tau}_{\lambda}(g, g^{\prime}_{j,\perp}) \rvert \leq \varepsilon  \lVert g \rVert _2\lVert g^{\prime}_{j,\perp}\rVert _2 \bigl(\frac{1}{\lambda} + \frac{2\lVert F \rVert _2}{\lambda^2}\bigr)\) for any \(\lambda>0\),
    \end{itemize}
\end{theorem*}

\paragraph{Setup and Notation.}
We fix orthonormal bases \(U_A\in\mathbb{R}^{d_A\times r_A}\) and \(U_E\in\mathbb{R}^{d_E\times r_E}\) for \(\range(A)\) and \(\range(E)\), respectively, and write \(U\coloneqq U_A\otimes U_E\) for the induced orthonormal basis of \(\range(F)=\range(A)\otimes\range(E)\) (so \(r=\rank(F)=r_A r_E\)).

For factorized test gradients \(g^{\prime}=a^{\prime}\otimes e^{\prime}\), we decompose \(a^{\prime} = a^{\prime}_{\parallel} + a^{\prime}_{\perp}\) with \(a^{\prime}_{\parallel}\in\range(A)\) and \(a^{\prime}_{\perp}\perp\range(A)\), and similarly \(e^{\prime} = e^{\prime}_{\parallel} + e^{\prime}_{\perp}\). The orthogonal projection of \(g^{\prime}\) onto \(\ker(F)\) is
\begin{equation}\label{eq:kfac-gperp}
    g^{\prime}_{\perp}
    =
    a^{\prime}_{\parallel}\otimes e^{\prime}_{\perp}
    +
    a^{\prime}_{\perp}\otimes e^{\prime}_{\parallel}
    +
    a^{\prime}_{\perp}\otimes e^{\prime}_{\perp}.
\end{equation}
In particular, \(g^{\prime}_{\perp}\in\ker(F)\), so (as in the i.i.d.\ case) it suffices to analyze leakage terms with kernel components \(g^{\prime}_{\perp}\in\ker(F)\).

\subsection{Proof Plan for \texorpdfstring{\Cref{thm:factorized-influence-leakage}}{Theorem 6}}\label{adxsubsec:proof-plan-for-thm:factorized-influence-leakage}
The factorized proof follows the same structure as the i.i.d.\ sketch case in \Cref{adxsec:projection-leakage}:
\begin{enumerate}
    \item \textbf{Deterministic reduction to two concentration conditions}. By \Cref{lma:two-conc-suffice}, it is enough to verify
          \begin{enumerate*}[label=(\roman*)]
              \item subspace concentration on \(\range(F)\), i.e., \(\lVert U^{\top}(P^{\top}P-I)U\rVert_2\leq\varepsilon\), and
              \item cross-term concentration \(\lVert U^{\top}(P^{\top}P-I)g^{\prime}_{j,\perp}\rVert_2\leq\varepsilon\lVert g^{\prime}_{j,\perp}\rVert_2\)
                    for the relevant kernel components \(\{g^{\prime}_{j,\perp}\}_{j=1}^k\).
          \end{enumerate*}
    \item \textbf{Stability on \(\range(F)=\range(A)\otimes\range(E)\)}. We control \(\lVert U^{\top}(P^{\top}P-I)U\rVert_2\) by bounding the corresponding factor-level deviations on \(\range(A)\) and \(\range(E)\).
    \item \textbf{Cross-term via Kronecker reduction with primitive bounds}. We expand \(P^{\top}P-I\) into factor sketch deviations and use \Cref{lma:kfac-cross-term} to reduce the cross-term \(\lVert U^{\top}(P^{\top}P-I)g^{\prime}_{j,\perp}\rVert_2\) to a collection of factor-level primitive quantities. These primitives are then controlled uniformly over the \(k\) test gradients using either a union bound over the nonzero out-of-range factor components (yielding the logarithmic dependence on \(k_A,k_E\)) or a subspace argument on their spans (yielding the \(k_A^{\prime},k_E^{\prime}\) dependence); see \Cref{prop:kfac-primitive-multi}.
    \item \textbf{Conclusion}. Plugging the primitive bounds into \Cref{lma:kfac-cross-term} and then into \Cref{lma:two-conc-suffice} yields \Cref{thm:factorized-influence-leakage}.
\end{enumerate}

The single-gradient proofs in \Cref{prop:leakage} (and the uniform extensions in \Cref{prop:multiple-k-prime}) depend on \(P\) only through two inequalities in \Cref{lma:two-conc-suffice}. In the factorized influence setting, the only additional step is to control the cross-term \(\lVert U^\top(P^\top P-I)g^{\prime}_{\perp}\rVert _2\) for \(g^{\prime}_{\perp}\in\ker(F)\) from factor-level primitive quantities. Define the factor sketch deviations
\[
    \Delta_A \coloneqq P_A^\top P_A - I_{d_A}, \qquad
    \Delta_E \coloneqq P_E^\top P_E - I_{d_E}.
\]
A direct expansion shows
\begin{equation}\label{eq:kfac-PtP-expand}
    P^\top P - I_{d_A d_E}
    = \Delta_A\otimes I_{d_E} + I_{d_A}\otimes \Delta_E + \Delta_A\otimes \Delta_E.
\end{equation}

The same expansion also makes the stability condition in \Cref{lma:two-conc-suffice} explicit.

\begin{lemma}\label{lma:kfac-stability}
    Assume \(\lVert U_A^{\top}\Delta_A U_A\rVert_2\leq \varepsilon\) and \(\lVert U_E^{\top}\Delta_E U_E\rVert_2\leq \varepsilon\) for some \(\varepsilon\in(0,1)\). Then with \(U=U_A\otimes U_E\),
    \[
        \lVert U^{\top}(P^{\top}P-I)U\rVert_2
        \leq \lVert U_A^{\top}\Delta_A U_A\rVert_2 + \lVert U_E^{\top}\Delta_E U_E\rVert_2 + \lVert U_A^{\top}\Delta_A U_A\rVert_2 \lVert U_E^{\top}\Delta_E U_E\rVert_2
        \leq 3\varepsilon.
    \]
\end{lemma}
\begin{proof}
    Using \Cref{eq:kfac-PtP-expand} and \(U^{\top}=(U_A\otimes U_E)^{\top}=U_A^{\top}\otimes U_E^{\top}\), we have
    \[
        U^{\top}(P^{\top}P-I)U
        =(U_A^{\top}\Delta_A U_A)\otimes I_{r_E} + I_{r_A}\otimes (U_E^{\top}\Delta_E U_E) + (U_A^{\top}\Delta_A U_A)\otimes (U_E^{\top}\Delta_E U_E).
    \]
    Taking operator norms and using \(\lVert X\otimes Y\rVert_2=\lVert X\rVert_2\lVert Y\rVert_2\) gives the claim.
\end{proof}

\begin{lemma}\label{lma:kfac-cross-term}
    Fix \(P_A,P_E\) (hence \(P\)), and let \(U_A,U_E\) be orthonormal bases for \(\range(A)\) and \(\range(E)\), and \(U\coloneqq U_A\otimes U_E\). Let \(g^{\prime} = a^{\prime}\otimes e^{\prime}\), decompose \(a^{\prime}=a^{\prime}_{\parallel}+a^{\prime}_{\perp}\) and \(e^{\prime}=e^{\prime}_{\parallel}+e^{\prime}_{\perp}\), and let \(g^{\prime}_{\perp}\) be the orthogonal projection of \(g^{\prime}\) onto \(\ker(F)\) given by \Cref{eq:kfac-gperp}. Define \(\Delta_A\coloneqq P_A^{\top}P_A-I_{d_A}\) and \(\Delta_E\coloneqq P_E^{\top}P_E-I_{d_E}\).

    Then
    \begin{equation}\label{eq:kfac-cross-term-9}
        \begin{split}
            \lVert U^{\top}(P^{\top}P-I)g^{\prime}_{\perp}\rVert_2
             & \leq
            \lVert U_A^{\top}\Delta_A a^{\prime}_{\perp}\rVert_2 \lVert e^{\prime}_{\parallel}\rVert_2
            +
            \lVert a^{\prime}_{\parallel}\rVert_2 \lVert U_E^{\top}\Delta_E e^{\prime}_{\perp}\rVert_2                    \\
             & \quad+
            \lVert U_A^{\top}\Delta_A a^{\prime}_{\parallel}\rVert_2 \lVert U_E^{\top}\Delta_E e^{\prime}_{\perp}\rVert_2
            +
            \lVert U_A^{\top}\Delta_A a^{\prime}_{\perp}\rVert_2 \lVert U_E^{\top}\Delta_E e^{\prime}_{\parallel}\rVert_2 \\
             & \quad+
            \lVert U_A^{\top}\Delta_A a^{\prime}_{\perp}\rVert_2 \lVert U_E^{\top}\Delta_E e^{\prime}_{\perp}\rVert_2.
        \end{split}
    \end{equation}
    In particular, if for some \(\varepsilon\in(0,1)\),
    \begin{equation}\label{eq:kfac-primitive-eps}
        \lVert U_A^{\top}\Delta_A x\rVert_2 \leq \varepsilon\lVert x\rVert_2
        \ \text{ for }x\in\{a^{\prime}_{\parallel},a^{\prime}_{\perp}\},
        \qquad
        \lVert U_E^{\top}\Delta_E y\rVert_2 \leq \varepsilon\lVert y\rVert_2
        \ \text{ for }y\in\{e^{\prime}_{\parallel},e^{\prime}_{\perp}\},
    \end{equation}
    then
    \begin{equation}\label{eq:kfac-cross-term-final}
        \lVert U^{\top}(P^{\top}P-I)g^{\prime}_{\perp}\rVert_2
        \leq (2\varepsilon+3\varepsilon^2)\bigl(\lVert a^{\prime}_{\parallel}\rVert_2\lVert e^{\prime}_{\perp}\rVert_2+\lVert a^{\prime}_{\perp}\rVert_2\lVert e^{\prime}_{\parallel}\rVert_2+\lVert a^{\prime}_{\perp}\rVert_2\lVert e^{\prime}_{\perp}\rVert_2\bigr)
        \leq 5\sqrt{3} \varepsilon\lVert g^{\prime}_{\perp}\rVert_2.
    \end{equation}
\end{lemma}
\begin{proof}
    Start from the decompositions \Cref{eq:kfac-PtP-expand,eq:kfac-gperp}:
    \[
        (P^{\top}P-I)g^{\prime}_{\perp}
        =(\Delta_A\otimes I+I\otimes\Delta_E+\Delta_A\otimes\Delta_E) \bigl(a^{\prime}_{\parallel}\otimes e^{\prime}_{\perp}+a^{\prime}_{\perp}\otimes e^{\prime}_{\parallel}+a^{\prime}_{\perp}\otimes e^{\prime}_{\perp}\bigr).
    \]

    Expanding gives nine Kronecker products. Applying \(U^{\top}=U_A^{\top}\otimes U_E^{\top}\) yields the explicit expansion
    \begin{align}\label{eq:kfac-cross-term-9terms}
         & U^{\top}(P^{\top}P-I)g^{\prime}_{\perp} \notag                                                                                                                                  \\
         & =
        \underbrace{\bigl(U_A^{\top}\Delta_A a^{\prime}_{\parallel}\bigr)\otimes\bigl(U_E^{\top} e^{\prime}_{\perp}\bigr)}_{(\Delta_A\otimes I)(a^{\prime}_{\parallel}\otimes e^{\prime}_{\perp})}
        +
        \underbrace{\bigl(U_A^{\top}\Delta_A a^{\prime}_{\perp}\bigr)\otimes\bigl(U_E^{\top} e^{\prime}_{\parallel}\bigr)}_{(\Delta_A\otimes I)(a^{\prime}_{\perp}\otimes e^{\prime}_{\parallel})}
        +
        \underbrace{\bigl(U_A^{\top}\Delta_A a^{\prime}_{\perp}\bigr)\otimes\bigl(U_E^{\top} e^{\prime}_{\perp}\bigr)}_{(\Delta_A\otimes I)(a^{\prime}_{\perp}\otimes e^{\prime}_{\perp})} \\
         & \quad+
        \underbrace{\bigl(U_A^{\top} a^{\prime}_{\parallel}\bigr)\otimes\bigl(U_E^{\top}\Delta_E e^{\prime}_{\perp}\bigr)}_{(I\otimes\Delta_E)(a^{\prime}_{\parallel}\otimes e^{\prime}_{\perp})}
        +
        \underbrace{\bigl(U_A^{\top} a^{\prime}_{\perp}\bigr)\otimes\bigl(U_E^{\top}\Delta_E e^{\prime}_{\parallel}\bigr)}_{(I\otimes\Delta_E)(a^{\prime}_{\perp}\otimes e^{\prime}_{\parallel})}
        +
        \underbrace{\bigl(U_A^{\top} a^{\prime}_{\perp}\bigr)\otimes\bigl(U_E^{\top}\Delta_E e^{\prime}_{\perp}\bigr)}_{(I\otimes\Delta_E)(a^{\prime}_{\perp}\otimes e^{\prime}_{\perp})}  \\
         & \quad+
        \underbrace{\bigl(U_A^{\top}\Delta_A a^{\prime}_{\parallel}\bigr)\otimes\bigl(U_E^{\top}\Delta_E e^{\prime}_{\perp}\bigr)}_{(\Delta_A\otimes\Delta_E)(a^{\prime}_{\parallel}\otimes e^{\prime}_{\perp})}
        +
        \underbrace{\bigl(U_A^{\top}\Delta_A a^{\prime}_{\perp}\bigr)\otimes\bigl(U_E^{\top}\Delta_E e^{\prime}_{\parallel}\bigr)}_{(\Delta_A\otimes\Delta_E)(a^{\prime}_{\perp}\otimes e^{\prime}_{\parallel})}
        +
        \underbrace{\bigl(U_A^{\top}\Delta_A a^{\prime}_{\perp}\bigr)\otimes\bigl(U_E^{\top}\Delta_E e^{\prime}_{\perp}\bigr)}_{(\Delta_A\otimes\Delta_E)(a^{\prime}_{\perp}\otimes e^{\prime}_{\perp})}.
    \end{align}
    Since \(U_A^{\top}a^{\prime}_{\perp}=0\) and \(U_E^{\top}e^{\prime}_{\perp}=0\), four terms vanish, leaving the five nonzero contributions
    \begin{align}\label{eq:kfac-cross-term-explicit}
         & U^{\top}(P^{\top}P-I)g^{\prime}_{\perp}                                                                    \\
         & \qquad=
        \bigl(U_A^{\top}\Delta_A a^{\prime}_{\perp}\bigr)\otimes\bigl(U_E^{\top} e^{\prime}_{\parallel}\bigr)
        +
        \bigl(U_A^{\top} a^{\prime}_{\parallel}\bigr)\otimes\bigl(U_E^{\top}\Delta_E e^{\prime}_{\perp}\bigr)
        +
        \bigl(U_A^{\top}\Delta_A a^{\prime}_{\parallel}\bigr)\otimes\bigl(U_E^{\top}\Delta_E e^{\prime}_{\perp}\bigr) \\
         & \qquad\quad+
        \bigl(U_A^{\top}\Delta_A a^{\prime}_{\perp}\bigr)\otimes\bigl(U_E^{\top}\Delta_E e^{\prime}_{\parallel}\bigr)
        +
        \bigl(U_A^{\top}\Delta_A a^{\prime}_{\perp}\bigr)\otimes\bigl(U_E^{\top}\Delta_E e^{\prime}_{\perp}\bigr).
    \end{align}

    Taking Euclidean norms and using \(\lVert u\otimes v\rVert_2=\lVert u\rVert_2\lVert v\rVert_2\), together with \(\lVert U_A^{\top}a^{\prime}_{\parallel}\rVert_2=\lVert a^{\prime}_{\parallel}\rVert_2\) and \(\lVert U_E^{\top}e^{\prime}_{\parallel}\rVert_2=\lVert e^{\prime}_{\parallel}\rVert_2\), yields \Cref{eq:kfac-cross-term-9}.

    Under \Cref{eq:kfac-primitive-eps}, the first two (single-factor) terms in \Cref{eq:kfac-cross-term-9} are bounded by \(\varepsilon\lVert a^{\prime}_{\perp}\rVert_2\lVert e^{\prime}_{\parallel}\rVert_2\) and \(\varepsilon\lVert a^{\prime}_{\parallel}\rVert_2\lVert e^{\prime}_{\perp}\rVert_2\), respectively.
    The last three (product) terms are each bounded by \(\varepsilon^2\lVert a^{\prime}_{\cdot}\rVert_2\lVert e^{\prime}_{\cdot}\rVert_2\), where \(a^{\prime}_{\cdot}\) can be either \(a^{\prime}_{\parallel}\) or \(a^{\prime}_{\perp}\), same for \(e^{\prime}_{\cdot}\). Summing and regrouping gives the first inequality in \Cref{eq:kfac-cross-term-final}.

    For the second inequality, note that the three summands in \Cref{eq:kfac-gperp} are pairwise orthogonal (since \(a^{\prime}_{\parallel}\perp a^{\prime}_{\perp}\) and \(e^{\prime}_{\parallel}\perp e^{\prime}_{\perp}\)), so
    \[
        \lVert g^{\prime}_{\perp}\rVert_2^2
        = \lVert a^{\prime}_{\parallel}\rVert_2^2\lVert e^{\prime}_{\perp}\rVert_2^2
        + \lVert a^{\prime}_{\perp}\rVert_2^2\lVert e^{\prime}_{\parallel}\rVert_2^2
        + \lVert a^{\prime}_{\perp}\rVert_2^2\lVert e^{\prime}_{\perp}\rVert_2^2.
    \]
    By Cauchy--Schwarz,
    \[
        \lVert a^{\prime}_{\parallel}\rVert_2\lVert e^{\prime}_{\perp}\rVert_2
        + \lVert a^{\prime}_{\perp}\rVert_2\lVert e^{\prime}_{\parallel}\rVert_2
        + \lVert a^{\prime}_{\perp}\rVert_2\lVert e^{\prime}_{\perp}\rVert_2
        \le \sqrt{3} \lVert g^{\prime}_{\perp}\rVert_2.
    \]
    Since \(\varepsilon\le 1\), we have \(2\varepsilon+3\varepsilon^2\le 5\varepsilon\), yielding the second inequality in \Cref{eq:kfac-cross-term-final}.
\end{proof}

\Cref{lma:kfac-cross-term} shows that to apply \Cref{lma:two-conc-suffice} in the factorized influence setting, it suffices to control factor-level deviations \(\lVert U_A^{\top}\Delta_A(\cdot)\rVert_2\) and \(\lVert U_E^{\top}\Delta_E(\cdot)\rVert_2\) on the relevant vectors. Once these are controlled with parameter \(\varepsilon\), the cross-term condition \(\lVert U^{\top}(P^{\top} P-I)g^{\prime}_{\perp}\rVert_2\leq \widetilde\varepsilon \lVert g^{\prime}_{\perp}\rVert_2\) holds with \(\widetilde\varepsilon=O(\varepsilon)\).

\subsection{Proof of Concentration of Factor-Level Primitives}\label{adxsubsec:kfac-probabilistic}
We now show how to obtain the factor-level bounds \Cref{eq:kfac-primitive-eps} with high probability from the same concentration tools used in \Cref{adxsubsec:leakage-multiple-test}. The key point is that the K-FAC structure allows us to control the relevant quantities by augmenting and controlling \(U_A\) and \(U_E\) separately, rather than working in dimension \(d_A d_E\) directly.

\begin{proposition}\label{prop:kfac-primitive-multi}
    Let \(\{g^{\prime}_j\}_{j=1}^k\) with \(g^{\prime}_j=a^{\prime}_j\otimes e^{\prime}_j\), and let \(g^{\prime}_{j,\perp}\) be the projection onto \(\ker(F)\). Define \(U_A,U_E\) as above and write \(a^{\prime}_j=a^{\prime}_{j,\parallel}+a^{\prime}_{j,\perp}\) and \(e^{\prime}_j=e^{\prime}_{j,\parallel}+e^{\prime}_{j,\perp}\). Denote
    \[
        k_A \coloneqq \sum_{j=1}^k \mathbbm{1}(a^{\prime}_{j,\perp}\neq 0), \qquad
        k_E \coloneqq \sum_{j=1}^k \mathbbm{1}(e^{\prime}_{j,\perp}\neq 0),
    \]
    and also,
    \[
        k_A^{\prime} \coloneqq \dim\bigl(\Span(\{a^{\prime}_{j,\perp}\}_{j=1}^k)\bigr),\qquad
        k_E^{\prime} \coloneqq \dim\bigl(\Span(\{e^{\prime}_{j,\perp}\}_{j=1}^k)\bigr).
    \]
    Assume \(P_A\) and \(P_E\) satisfy the same sketch assumptions as in \Cref{lma:covariance-concentration} (independently across factors). Then, for any \(\varepsilon,\delta\in(0,1)\), if
    \[
        m_A = \Omega\left(\frac{r_A+\min\{\log(k_A/\delta), k_A^{\prime}+\log(1/\delta)\}}{\varepsilon^2}\right)
    \]
    and
    \[
        m_E = \Omega\left(\frac{r_E+\min\{\log(k_E/\delta), k_E^{\prime}+\log(1/\delta)\}}{\varepsilon^2}\right),
    \]
    then with probability at least \(1-\delta\), the following bounds hold simultaneously for all \(j\in\{1,\dots,k\}\):
    \[
        \lVert U_A^{\top} (P_A^{\top} P_A-I)a^{\prime}_{j,\parallel}\rVert_2 \leq \varepsilon \lVert a^{\prime}_{j,\parallel}\rVert_2,
        \qquad
        \lVert U_A^{\top} (P_A^{\top} P_A-I)a^{\prime}_{j,\perp}\rVert_2 \leq \varepsilon \lVert a^{\prime}_{j,\perp}\rVert_2,
    \]
    and
    \[
        \lVert U_E^{\top} (P_E^{\top} P_E-I)e^{\prime}_{j,\parallel}\rVert_2 \leq \varepsilon \lVert e^{\prime}_{j,\parallel}\rVert_2,
        \qquad
        \lVert U_E^{\top} (P_E^{\top} P_E-I)e^{\prime}_{j,\perp}\rVert_2 \leq \varepsilon \lVert e^{\prime}_{j,\perp}\rVert_2.
    \]
    Consequently, the cross-term condition
    \[
        \lVert U^{\top}(P^{\top} P-I)g^{\prime}_{j,\perp}\rVert_2
        \leq 5\sqrt{3} \varepsilon \lVert g^{\prime}_{j,\perp}\rVert_2
    \]
    holds for all \(j\) simultaneously.\footnote{Equivalently, one can run the primitive bounds \Cref{eq:kfac-primitive-eps} with accuracy \(\varepsilon/(5\sqrt{3})\) to obtain a cross-term tolerance of \(\varepsilon\); this only changes \(m_A,m_E\) by constant factors in the \(\Omega(\cdot)\) conditions.}
\end{proposition}
\begin{proof}
    We prove the \(A\)-factor bounds; the \(E\)-factor bounds are identical. Firstly, by \Cref{lma:covariance-concentration} applied to the \(r_A\)-dimensional subspace \(\range(A)\), with probability at least \(1-\delta/4\),
    \[
        \lVert U_A^{\top}(P_A^{\top} P_A-I)U_A\rVert_2 \leq \varepsilon,
    \]
    provided \(m_A=\Omega(\varepsilon^{-2}(r_A+\log(4/\delta)))\). Next, to control \(\lVert U_A^{\top}(P_A^{\top} P_A-I)a^{\prime}_{j,\perp}\rVert_2\) uniformly over \(j\), we use either:
    \begin{enumerate}[label=(\roman*)]
        \item\label{prop:kfac-primitive-multi-i} a union bound over the \(k_A\) nonzero vectors \(\{a^{\prime}_{j,\perp}\}\), giving a \(\log k_A\) dependence, or
        \item\label{prop:kfac-primitive-multi-ii} a subspace argument on \(\Span(\range(A)\cup\{a^{\prime}_{j,\perp}\}_{j=1}^k)\), giving a dependence on \(k_A^{\prime}\).
              These two routes yield the stated \(\min\{\log k_A, k_A^{\prime}\}\) dependence.
    \end{enumerate}
    Concretely, route \labelcref{prop:kfac-primitive-multi-i} follows exactly as in \Cref{prop:multiple-k}: for a fixed unit vector \(v\perp\range(A)\), \(\lVert U_A^{\top}(P_A^{\top} P_A-I)v\rVert_2 \leq \varepsilon\) holds with probability at least \(1-\delta/(4\max\{k_A,1\})\) provided
    \[
        m_A
        =\Omega\left(\frac{r_A+\log(4\max\{k_A,1\}/\delta)}{\varepsilon^{2}}\right)
    \]
    and a union bound over the nonzero \(a^{\prime}_{j,\perp}\) gives the desired uniform control.

    On the other hand, route \labelcref{prop:kfac-primitive-multi-ii} is obtained by applying \Cref{lma:covariance-concentration} to the \((r_A+k_A^{\prime})\)-dimensional subspace \(\Span(\range(A)\cup\{a^{\prime}_{j,\perp}\}_{j=1}^k)\), which yields the same uniform bound with
    \[
        m_A
        =\Omega\left(\frac{r_A+k_A^{\prime}+\log(4/\delta)}{\varepsilon^{2}}\right).
    \]
    For \(a^{\prime}_{j,\parallel}\in\range(A)\), the desired inequality follows deterministically from the operator-norm event:
    \[
        \begin{split}
            \lVert U_A^{\top}(P_A^{\top} P_A-I)a^{\prime}_{j,\parallel}\rVert_2
             & = \lVert U_A^{\top}(P_A^{\top} P_A-I)U_A(U_A^{\top} a^{\prime}_{j,\parallel})\rVert_2            \\
             & \leq \lVert U_A^{\top}(P_A^{\top} P_A-I)U_A\rVert_2\cdot \lVert a^{\prime}_{j,\parallel}\rVert_2
            \leq \varepsilon\lVert a^{\prime}_{j,\parallel}\rVert_2.
        \end{split}
    \]
    Repeating the above argument for the \(E\)-factor and union bounding the \(A\) and \(E\) events gives the four primitive inequalities simultaneously for all \(j\). The claimed cross-term bound then follows by \Cref{lma:kfac-cross-term}.
\end{proof}

On the event in \Cref{prop:kfac-primitive-multi}, \Cref{lma:kfac-cross-term} gives the cross-term condition required by \Cref{lma:two-conc-suffice}. The stability condition on \(\range(F)\) follows from the factor operator-norm events via \Cref{lma:kfac-stability}. Thus \Cref{lma:two-conc-suffice} applies and yields the stated unregularized and regularized leakage bounds for \(g^{\prime}_{j,\perp}\), uniformly over \(j\).

\end{document}